\documentclass[opre,nonblindrev]{informs3a}

\OneAndAHalfSpacedXI

\usepackage{natbib}
\bibpunct[, ]{(}{)}{,}{a}{}{,}
\def\bibfont{\small}

\usepackage{amssymb}
\usepackage{pifont}
\usepackage{makecell}
\usepackage{multirow}

\renewcommand{\bibfont}{\footnotesize\linespread{1}\selectfont}

\renewcommand{\bar}{\widebar}

\renewcommand{\P}{ \mathbb{P}}

\newcommand{\E}{ \mathbb{E} }

\usepackage{parskip} 

\usepackage[normalem]{ulem}

\usepackage{graphicx}
\usepackage{subcaption}

\usepackage{array}
\usepackage{cases}
\usepackage{pifont}
\usepackage{braket}
\usepackage{algorithm}
\usepackage{algorithmicx}
\usepackage[noend]{algpseudocode}
\def\1{(\mathrm{\uppercase\expandafter{\romannumeral1}})}
\def\2{(\mathrm{\uppercase\expandafter{\romannumeral2}})}

\RequirePackage[colorlinks,citecolor=blue,linkcolor=blue,urlcolor=blue]{hyperref}
\usepackage{enumitem}
\DeclareMathAlphabet{\mathscr}{OT1}{pzc}{m}{it}

\usepackage{enumitem}
\usepackage{tablefootnote}

\usepackage{bm}
\usepackage{hyperref}

\usepackage{xurl}

\TheoremsNumberedThrough
\ECRepeatTheorems

\EquationsNumberedThrough

\usepackage{xcolor}

\definecolor{DSgray}{cmyk}{0,1,0,0}

\usepackage{mathabx}
\renewcommand{\hat}{\widehat}
\renewcommand{\tilde}{\widetilde}

\begin{document}

\RUNTITLE{Learning in Position-Aware Multinomial Logit Bandits}

 \TITLE{Learning in Position-Aware Multinomial Logit Bandits: From Multiplicative to General Position Effects}
 \RUNAUTHOR{Chen, Dai, Lyu, and Zhou}

\ARTICLEAUTHORS{
\AUTHOR{Xi Chen\footnotemark[1]}
\AFF{Leonard N.~Stern School of Business, New York University, New York, NY 10012, USA, \EMAIL{xc13@stern.nyu.edu}}
 \AUTHOR{Shibo Dai\footnotemark[1]}
 \AFF{Qiuzhen College, Tsinghua University, Beijing 100084, China, \EMAIL{dsb24@mails.tsinghua.edu.cn}}
 \AUTHOR{Jiameng Lyu\footnotemark[1]}
 \AFF{Department of Management Science, School of Management, Fudan  University, Shanghai, 200433, China,  \EMAIL{jiamenglyu@fudan.edu.cn}}
 \AUTHOR{Yuan Zhou\footnotemark[1]}
 \AFF{Yau Mathematical Sciences Center \& Department of Mathematical Sciences, Tsinghua University, Beijing 100084, China, \EMAIL{yuan-zhou@tsinghua.edu.cn}}
 } 
 
\renewcommand{\thefootnote}{\fnsymbol{footnote}}
\footnotetext[1]{Author names listed in alphabetical order.}
\renewcommand{\thefootnote}{\arabic{footnote}}

\ABSTRACT{
    We study the dynamic joint assortment selection and positioning problem, where the attraction of each product depends on both its intrinsic appeal and its display position under a Multinomial Logit (MNL) choice framework. Our study ranges from the multiplicative position effects model, in which each product's attraction is scaled by a position-specific factor, to a general position effects model assigning independent attraction parameters to every product--position pair to capture heterogeneous synergies.

    For both models, we design round-based learning algorithms that update decisions after every single feedback, and establish the first regret-optimal characterization. Besides, our round-based algorithms provide the prompt operations needed by modern platforms. For the multiplicative model, we develop a cross-position pairwise maximum likelihood estimator with a clipping mechanism, and prove that our algorithm P2MLE-UCB attains a regret of $\tilde{O}(\sqrt{NT})$, matching the lower bound and closing the $\sqrt{K}$ gap left by prior epoch-based analyses. For the general model, we establish a minimax lower bound and propose GP2-UCB with a matching upper bound. Moreover, we design an efficient subroutine for the per-round joint assortment and positioning optimization based on Dinkelbach's method and maximum-weight bipartite matching. Numerical experiments on synthetic data and the Expedia dataset show that our algorithms consistently outperform state-of-the-art benchmarks.
}

\maketitle

\section{Introduction}

The Multinomial Logit (MNL) model is one of the most widely used frameworks in discrete choice modeling and revenue management. In the MNL model, each product is associated with a parameter that captures its intrinsic attraction, and purchase probabilities are determined by a simple closed-form expression. Owing to its analytical tractability, interpretability, and computational convenience, the MNL model has become a standard tool for assortment optimization in retail, online marketplaces, and recommendation systems.

The rapid growth of digital platforms has introduced new challenges that extend beyond the classical assortment framework. On e-commerce websites, search engines, food delivery apps, and streaming platforms, products are typically displayed in ranked lists or visual grids, making display position a critical determinant of consumer choice.
Empirical evaluations on large-scale search and online advertising platforms also demonstrate that click-through rates (CTR) are profoundly impacted by display positioning, typically exhibiting a sharp decline as one moves down a ranked list \citep{craswell2008experimental, chen2023bias}.
Recent recommender-system evidence \citep{joachims2017unbiased, wang2018position} shows that ignoring position bias confounds intrinsic attractiveness with positional exposure, leading to biased estimates and suboptimal decisions. Therefore, to improve long-run assortment performance under the MNL framework, we should extend the model to incorporate position effects and jointly optimize assortment and product placement.

Motivated by these observations, we consider the dynamic assortment and positioning problem under the MNL model, in which the platform sequentially decides which products to display and at which positions, observes customer feedback, and adapts its decisions over time.

A classic and widely used specification for position effects is the \emph{multiplicative position effects} model,  in which the attraction of product $i$ displayed at position $k$ takes the separable form $v_i\theta_k$, with $v_i$ denoting the intrinsic attraction of product $i$ and $\theta_k$ denoting the effect of position $k$. This specification is analytically convenient and provides a first-order approximation of how ranking shapes demand.

\begin{figure}[h!]
  \begin{center}
    \includegraphics[width=0.6\textwidth]{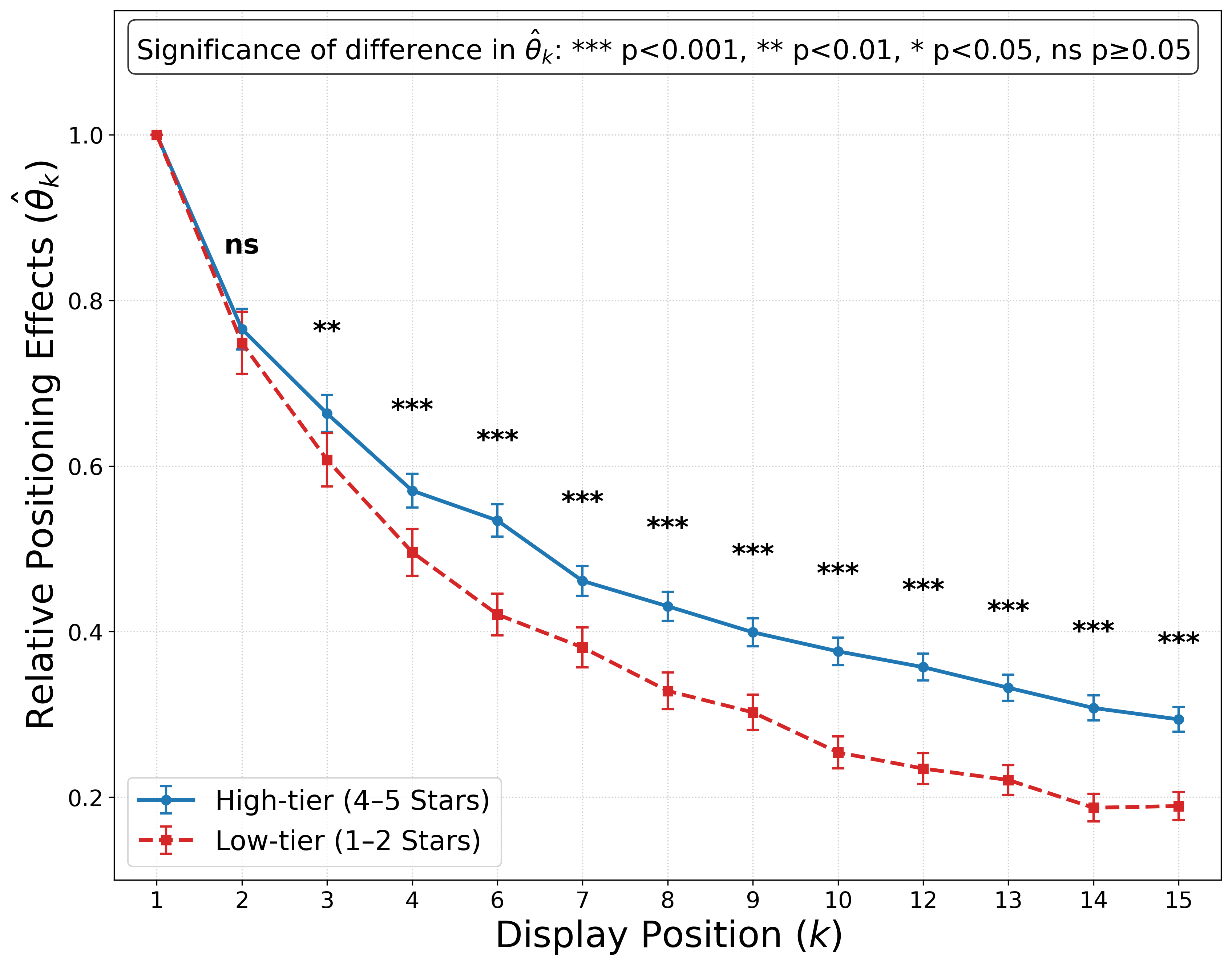}
    \caption{Heterogeneous Evidence of the Position Effects.}
    \label{fig:Heterogeneous_Effects_Evidence}
  \end{center}
\end{figure}
While this separable specification is a deliberate simplification of complex position effects, it can introduce bias on real-world platforms by imposing homogeneous position sensitivity across all items and ruling out heterogeneous product-position synergies. Such synergies are common in practice, where the effect of a position often depends on product-specific attributes such as brand recognition or visual quality. 
For instance, high-tier or well-known items may retain substantial demand even in lower-ranked positions, whereas lower-tier items may be far more sensitive to display rank. We validate this using the Expedia Hotel Searches dataset.\footnote{This dataset comprises approximately 400,000 unique search queries with over 9.9 million hotel impressions, including hotel attributes, display positions, and click/booking outcomes. URL: \url{https://www.kaggle.com/c/expedia-personalized-sort}.} Specifically, we compute the average click-through rate at each position for two distinct categories, ``high-tier'' hotels (4-5 stars) and ``low-tier'' hotels (1-2 stars), and normalize each category's curve by its own click-through rate at the first position. This normalization does not assume that the first-position effect equals 1. Instead, because intrinsic attractions $v_i$ are unobserved, absolute click-through levels confound product attractions with position effects; anchoring at position 1 isolates the relative position effects ($\hat{\theta}_k$). Our empirical analysis reveals a clear difference in position effects between the two groups (see Figure~\ref{fig:Heterogeneous_Effects_Evidence}); because the dataset contains too few observations at positions $k=5$ and $k=11$, we exclude these two positions when generating the plot.  This evidence underscores the necessity of a general position-effect model to capture heterogeneous product-position synergies.

Motivated by this empirical evidence, we propose the \emph{general position effects} model that assigns an independent attraction parameter $v_{i,k}$ to every product-position pair. This formulation accommodates arbitrary heterogeneous product-position synergies and subsumes the multiplicative specification as a special case, while allowing position sensitivity to vary flexibly across products.

Together, these two position effects models span a wide modeling range, from the most structured multiplicative model, with only $N+K$ parameters and a strong separability assumption that yields tractability but restricts heterogeneity, to the fully general model, with $NK$ free parameters that accommodate arbitrary product--position interactions at the cost of a larger effective dimension. The two position effects models thus represent distinct operational regimes, the multiplicative model is cheap to learn but may misfit real-world data, whereas the general model is expressive but more data-hungry, and jointly covering them is necessary for a complete understanding of the problem. The existing literature, however, covers these two position effects models unevenly. For the multiplicative model, \cite{abeliuk2016assortment} analyze the static (offline) assortment and positioning problem, and \cite{luo2025rate} study the dynamic learning setting through an epoch-based approach, establishing a regret upper bound of $\widetilde{\mathcal{O}}(\sqrt{KNT})$ against an information-theoretic lower bound of $\Omega(\sqrt{NT})$; thus a $\sqrt{K}$ gap remains between the known upper and lower bounds under identical problem settings. For the general position effects model, to the best of our knowledge, neither the static nor the dynamic setting has been analyzed.

Our goal in this paper is to provide a comprehensive and systematic study of the dynamic assortment and positioning problem, from the classical multiplicative position effects model to a fully general effects model, with an emphasis on deriving regret-optimal learning algorithms for both position effects models. Specifically, we develop algorithms whose regret upper bounds match the corresponding information-theoretic lower bounds in each setting, closing the existing $\sqrt{K}$ gap under the multiplicative model and providing the first sharp characterization under the general position effects model.

Beyond these theoretical results, another appealing feature of our algorithms is their operational flexibility. Our algorithms are \emph{round-based}: they refresh both parameter estimates and display decisions after every single customer interaction, by extracting pairwise-comparison information from each observation. This design complements the epoch-based paradigm, which obtains accurate estimates by aggregating observations within each epoch.
The price of this aggregation is prompt responsiveness: under the epoch-based paradigm, the platform must keep showing essentially the same display until a ``no-purchase'' event occurs before it can adjust its decisions, and thus cannot react promptly to online feedback. The round-based perspective, by contrast, aligns with a direction emphasized by many modern platforms, where the ability to respond to user feedback and refresh displayed content at a per-interaction granularity is often valuable in practice. For instance, the engineering blog of the large-scale discovery platform Pinterest describes how incorporating real-time user-action signals into its Homefeed ranking model yields a more responsive feed and measurable gains in engagement.\footnote{See \url{https://medium.com/pinterest-engineering/how-pinterest-leverages-realtime-user-actions-in-recommendation-to-boost-homefeed-engagement-volume-165ae2e8cde8}, accessed on Apr 20, 2026.}

\subsection{Our Contributions}
Our main contributions can be summarized as follows. Taken together, they provide a systematic treatment of the dynamic assortment and positioning problem under both position effects models, including regret guarantees, algorithm design, computational subroutines, and empirical validation on synthetic and real-world data.

\textbf{Pairwise MLE and sharpened regret for the multiplicative position effects model.}
Under the multiplicative position effects model, we develop a round-based algorithm, P2MLE-UCB, built on a new \emph{pairwise MLE} tailored to the cross-position structure of the problem. In each round, for each displayed product, we condition on the event that the observed choice is either $i$ or the outside option; this yields a Bernoulli observation whose success probability depends only on $i$'s own attraction parameter and its position effect. Pooling these product-versus-outside-option observations across all positions where product $i$ has been displayed gives a cross-position log-likelihood for $v_i^\star$, and the pairwise MLE is defined as the clipped root of the corresponding score equation. Because each interaction is used immediately, P2MLE-UCB updates both parameter estimates and display decisions in every round, rather than waiting for epoch termination.

For analysis, we apply concentration inequality to the cross-position score function. The clipping step is part of estimator construction and is crucial for proving a strictly positive lower bound on the score derivative, which allows us to convert score concentration into variance-aware confidence intervals for the pairwise MLE. Combined with a regret decomposition that tracks effective sample sizes across positions, these intervals yield a regret bound of $\tilde{O}(\sqrt{NT})$ for P2MLE-UCB, matching the information-theoretic lower bound $\Omega(\sqrt{NT})$ and removing the extraneous $\sqrt{K}$ factor of prior epoch-based method in \citep{luo2025rate}. This provides the first minimax-optimal regret guarantee for the joint dynamic assortment and positioning problem under the multiplicative position effects model.

\textbf{Regret-optimal and efficient algorithms for the general position effects model.}
To capture the heterogeneous product-position synergies often observed in practice, we introduce a general position effects model that assigns an independent attraction parameter to every product-position pair. 
For this richer model, we propose GP2-UCB, a round-based algorithm that builds Bernstein-type UCBs on each product--position attraction from pairwise feedback and offers the optimistic assortment-and-positioning each round. We establish a minimax lower bound of $\Omega(\sqrt{KNT})$ for the general position effects model and prove that GP2-UCB attains a matching $\tilde{O}(\sqrt{KNT})$ upper bound, giving the first rate-optimal characterization in this setting. The principal technical contribution here lies in the per-round optimization: unlike the position-free MNL literature, the joint assortment-and-positioning subproblem under independent product-position attractions is a fractional two-sided assignment, for which we design an iterative subroutine based on Dinkelbach's method and maximum-weight bipartite matching that converges superlinearly and scales to large catalogs.

\textbf{Superior numerical performance under both synthetic and real-world data.}
We evaluate our proposed algorithms in two complementary numerical studies. First, on synthetic datasets covering both the multiplicative and the general position effects models, we validate the theoretical regret bounds across a range of configurations of $N$ and $K$, and the underlying attraction structure. Second, we use the Expedia Hotel Search dataset to extract ground-truth parameters and simulate realistic e-commerce environments. In both studies, our algorithms attain consistently lower cumulative regret, faster convergence, and better scalability than state-of-the-art epoch-based benchmarks.

\subsection{Related Literature}

In this section, we introduce the following streams of literature related to our paper.

\textbf{Position-aware online learning.}
Position bias in ranked displays is a central issue in recommendation and search systems, and has been extensively studied in both academic and industrial settings \citep{kveton2015cascading, li2016contextual, cheung2019thompson, joachims2017unbiased, wang2018position}. A consistent finding is that ignoring position effects confounds intrinsic product attractions with display advantage, leading to biased estimation and suboptimal decisions. 
In online settings, the Cascade Bandit line of work is a canonical position-aware framework \citep{kveton2015cascading, li2016contextual, cheung2019thompson}. However, its stopping-based reward model is often restrictive for assortment applications, where users may compare multiple displayed items before making a choice. This gap motivates our MNL-based joint assortment-and-positioning formulation, which captures substitution across simultaneously displayed products.

\textbf{Joint assortment optimization and positioning under MNL.}
Early work on joint assortment and positioning under MNL is mostly static, assuming known parameters \citep{abeliuk2016assortment,gallego2020approximation,aouad2021display}. 
\cite{abeliuk2016assortment} adopt a multiplicative position-effects model, which keeps the optimization tractable. \cite{gallego2020approximation} study product framing: products are arranged across virtual pages, a customer views the top $x$ pages with probability $\lambda(x)$, and then chooses under a general choice model. \cite{aouad2021display} study ranked lists in which a customer considers the first $k$ positions with probability $\lambda_k$ and then makes an MNL choice among those $k$ products. Despite these modeling differences, all of these works assume that the platform knows all parameters.
Notably, in this paper we also develop an efficient optimization subroutine, but for the more general position effects model where each product-position pair has its own attraction parameter, thereby accommodating heterogeneous product-position synergies.

In many real-world applications, especially for newly introduced products or items with limited data, attraction parameters are unknown and must be learned from data in an online manner. Estimating these parameters accurately is particularly challenging when only limited initial observations are available.
Recently, \cite{luo2025rate} study the multiplicative position-effect model for dynamic assortment and positioning under MNL. Their method is epoch-based: the platform repeatedly displays the same assortment-position pair until a ``no-purchase'' event occurs. 
Their regret bound, $\widetilde{\mathcal{O}}(\sqrt{KNT})$, is not tight, as it contains an extra $\sqrt{K}$ factor relative to the lower bound $\Omega(\sqrt{NT})$. Besides, they did not consider the heterogeneous product-position synergies commonly observed in real-world platforms. Please refer to Table~\ref{tab:comparison-prior-work} for a side-by-side comparison of \cite{luo2025rate} and the algorithms developed in this paper across model class, feedback regime, and regret upper and lower bounds.

\begin{table}[h]
\centering
\caption{Detailed Comparisons with \cite{luo2025rate}. ``Mult.''\ denotes the multiplicative position-effects model; ``General'' denotes the general position-effects model.}
\label{tab:comparison-prior-work}
\renewcommand{\arraystretch}{1.35}
\setlength{\tabcolsep}{6pt}
\small
\begin{tabular}{l c c c c}
\hline\hline
\multirow{2}{*}{\textbf{Algorithm}} & \multirow{2}{*}{\textbf{Model}} & $\bm{\theta}^\star$ & \textbf{Feedback} & \textbf{Regret}\\
& & \textbf{known?} & \textbf{regime} & \textbf{(upper / lower)}\\
\hline
TLR-UCB \citep{luo2025rate} & Mult.\  & Yes & Epoch & $\widetilde{O}(\sqrt{KNT})$ \,/\, $\Omega(\sqrt{NT})$\\
EI-TLR \citep{luo2025rate}  & Mult.\  & No  & Epoch heuristic & \textbf{none} \,/\, ---\\
\hline
P2MLE-UCB \textbf{(ours)}   & Mult.\  & Yes & Round & $\widetilde{O}(\sqrt{NT})$ \,/\, $\Omega(\sqrt{NT})$\\
GP2-UCB \textbf{(ours)}     & Mult.\  & \textbf{No}  & Round & $\widetilde{O}(\sqrt{KNT})$ \,/\, $\Omega(\sqrt{NT})$\\
GP2-UCB \textbf{(ours)}     & General & --- & Round & $\widetilde{O}(\sqrt{KNT})$ \,/\, $\Omega(\sqrt{KNT})$\\
\hline\hline
\end{tabular}

\end{table}

\textbf{Other works on assortment selection.}
There is a stream of work studying dynamic assortment selection under the MNL model, where attraction parameters are unknown in advance.
Early work studied settings with independent demand across products \citep{caro2007dynamic}, and subsequent landmark studies established ``explore-then-commit'' strategies for the capacitated MNL assortment problem \citep{rusmevichientong2010dynamic,saure2013optimal}. This was further refined by UCB-type algorithms \citep{agrawal2019mnl} and Thompson sampling methods \citep{agrawal2017thompson}. Related extensions also explored the uncapacitated MNL model \citep{chen2021optimal}, nested logit choice \citep{chen2021dynamic}, joint dynamic pricing and assortment decisions \citep{miao2021dynamic}, and joint assortment and inventory decisions \citep{liang2026online}. 

Many existing dynamic assortment policies require the platform to offer the same assortment repeatedly until a ``no-purchase'' event occurs to obtain unbiased estimates of the attraction parameters. Breaking from this paradigm, \cite{saha2024stop} recently propose using pairwise comparison information, or rank-breaking, to update parameters after every round.

Our algorithms also operate in the round-based regime, but address the position-aware MNL framework. 
For the multiplicative model, display positions couple product attractions with the display layout; we handle this through a cross-position clipped pairwise MLE that pools position-dependent observations. 
For the general model, GP2-UCB naturally lifts the rank-breaking UCB idea from products to product--position pairs. The principal new challenge is computational: the per-round decision no longer reduces to top-$m$ selection as in the position-free setting, but requires joint assortment-and-positioning optimization. We address this with an efficient optimization subroutine and establish a rate-optimal regret characterization, including a new $\Omega(\sqrt{KNT})$ lower bound.

\cite{dong2025pasta, hanimproved, han2025learning} study offline assortment optimization under the MNL model and propose offline estimation frameworks.
A stream of work studies richer choice models that capture behavioral phenomena absent from the standard MNL model, such as impatient customers, repeated customer interactions with history-dependent choice probabilities, and multi-stage sequential decisions with commitment \citep{gao2021assortment, chen2023assortment, xu2023assortment}, as well as more flexible model structures including Markov-chain choice model, multi-item basket purchases, multi-attribute contextual choice model, and focal Luce model \citep{li2025online, jasin2024assortment, najafi2026assortment, jiang2026assortment}.

\subsection{Notations and Organization}

\noindent
\textbf{Notations.} For any positive integer $n$, we let $[n]=\{1,\dots,n\}$. We use $[N]$ to denote the product catalog and $[K]$ to denote the available display positions. The indicator function $\mathbf{1}\{\mathcal{A}\}$ takes the value 1 if event $\mathcal{A}$ occurs and 0 otherwise. We adopt the standard asymptotic notations $O(\cdot)$ and $\Omega(\cdot)$ to describe the scaling of functions. To simplify expressions, we use $\tilde{O}(\cdot)$ and $\tilde{\Omega}(\cdot)$ to hide logarithmic factors in $N,K$ and $T$.

\noindent
\textbf{Organization.} The remainder of the paper is organized as follows. Section~\ref{sec:problem_formulation} introduces the formal problem setup and choice models. Section~\ref{sec:multiplicative_analysis} studies the multiplicative position effects model and presents the P2MLE-UCB algorithm with sharp regret guarantees. Section~\ref{sec:general_analysis} develops the GP2-UCB algorithm for the general positioning effects model and establishes minimax-optimal regret. Section~\ref{sec:simulations} reports numerical experiments on synthetic data, and Section~\ref{sec:expedia_case} presents a empirical study using real-world data.

\section{Problem Formulation}
\label{sec:problem_formulation}
We study a sequential decision-making problem in which a platform manages a catalog of $N$ products $[N]$ and displays them across $K$ available positions $[K]$, with $K\le N$ (e.g., the slots of a ranked list or a visual grid). The interaction unfolds over a horizon of discrete rounds $t=1,\dots,T$. At the beginning of each round, a customer arrives, and the platform's decision $(S_t,\sigma_t)$ consists of two coupled components. First, it selects an assortment of products $S_t \subseteq [N]$ such that $|S_t|\le K$. Second, it assigns each $i\in S_t$ to a distinct slot via an injective mapping $\sigma_t:S_t\to[K]$. We collect the feasible decisions in
\[
  \mathcal{F} = \bigl\{(S,\sigma): S\subseteq[N],\ |S|\le K,\ \sigma:S\to[K]\text{ is injective}\bigr\}.
\]

Upon being presented with $(S_t,\sigma_t)$, the customer either purchases a product $i\in S_t$ or leaves without purchasing (the outside option, denoted $i=0$).
To model this choice behavior, we let $\alpha_{i,\sigma_t(i)} \in(0,1]$ denote the attraction of product $i$ when it is placed at position $\sigma_t(i)$.  Under the Multinomial Logit (MNL) model, the purchase probabilities are given by
\begin{equation} \label{MNLmodel_choice_probability}
  \mathbb{P}(i\mid S_t,\sigma_t) =
  \begin{cases}
    \frac{\alpha_{i,\sigma_t (i)}}{1 + \sum_{j\in S_t} \alpha_{j,\sigma_t (j)}} & \mathrm{if}\ i \in S_t, \\
    \frac{1}{1 + \sum_{j\in S_t} \alpha_{j,\sigma_t (j)}} & \mathrm{if}\ i = 0.
  \end{cases}
\end{equation}
Assuming each product $i$ generates a known revenue $r_i \in [0,1]$, the expected revenue for any feasible decision $(S,\sigma)\in\mathcal{F}$ (with the dependence on $t$ suppressed when irrelevant) is
\[
R(S,\sigma)=\sum_{i\in S} r_i \,\mathbb{P}(i\mid S,\sigma).
\]

As the rounds proceed, the platform continuously learns from user feedback. Let $i_t \in S_t \cup \{0\}$ denote the customer's actual choice at time $t$; conditional on $(S_t,\sigma_t)$ and the true attraction parameters, $i_t$ is drawn independently from \eqref{MNLmodel_choice_probability}. The historical data available to the platform at the start of round $t$ is collected in the set $\mathcal{H}_{t-1} = \{(S_s,\sigma_s,i_s) : s = 1, \dots, t-1\}$. Based on this history, the platform employs an \textit{admissible policy} $\pi = \{\pi_t\}_{t=1}^T$, where each (possibly randomized) decision rule $\pi_t$ maps the past observations to a feasible decision in $\mathcal{F}$, i.e.,
\[
  (S_t,\sigma_t) = \pi_t(\mathcal{H}_{t-1}) \in \mathcal{F}.
\]

The platform's objective is to design an admissible policy that minimizes its \textit{cumulative regret} over the horizon $T$. We define the cumulative regret relative to an optimal static strategy $(S^\star,\sigma^\star)$ that maximizes the expected revenue under the true attraction parameters:
\[
  (S^\star, \sigma^\star) \in \arg\max_{(S,\sigma)\in\mathcal{F}} R(S,\sigma);
\]
any tie-breaking rule yields the same regret. The existence and efficient computation of $(S^\star,\sigma^\star)$ are well-studied for the multiplicative case \citep{abeliuk2016assortment,luo2025rate}; the corresponding subroutine for the general model is developed in Section~\ref{sec:general_analysis}.
Accordingly, the expected cumulative regret of a policy $\pi$ measures the total shortfall in revenue compared to this optimal benchmark:
\[
  \mathrm{Regret}(T,\pi)
  \;=\;
  \mathbb{E}_\pi\!\left[
    \sum_{t=1}^T
    \left(
      R(S^\star,\sigma^\star)
      -
      R(S_t,\sigma_t)
    \right)
  \right],
\]
where the expectation is taken over the randomness of the customer choices $\{i_t\}$ and any internal randomization in $\pi$.

\subsection{Multiplicative Position Effects Model}
\label{subsec:multiplicative_model}

We begin by considering the multiplicative position effects model, a parsimonious and widely adopted approach for incorporating position bias into the MNL model. Under this position effects model, the overall attraction of a displayed product is factored into two components: the intrinsic attraction of the product and the position effect of the displayed position.

Formally, for any product $i \in [N]$ placed at position $k \in [K]$, the attraction parameter is expressed as the multiplication
\[
  \alpha_{i,k} = v_i^\star \, \theta^\star_k,
\]
where $v_i^\star \in (0,1]$ represents the intrinsic attraction parameter of product $i$, and $\theta_k^\star \in (0,1]$ captures the effect of position $k$. Note that the pair $(v_i^\star,\theta_k^\star)$ is only identifiable up to a positive scalar, since $(c\,v_i^\star,\theta_k^\star/c)$ yields the same MNL probabilities and revenues; throughout, we adopt the standard normalization $\max_{k\in[K]}\theta_k^\star=1$, which is consistent with $\theta_k^\star\in(0,1]$ and renders the parameters well-defined. To facilitate our subsequent regret analysis (see Section~\ref{sec:multiplicative_analysis}), we also define the minimum position effect $\theta_{\min}^\star = \min_{k \in [K]} \theta_k^\star$.

Substituting $\alpha_{i,k}=v_i^\star\theta_k^\star$ into \eqref{MNLmodel_choice_probability} immediately yields the choice probabilities under the multiplicative model.

\subsection{General Position Effects Model}
\label{subsec:general_model}

As discussed in the introduction, the multiplicative position effects model is a useful but restrictive simplification, as it cannot capture the heterogeneous product-position synergies often observed in practice (see Figure~\ref{fig:Heterogeneous_Effects_Evidence}). To overcome these limitations, we introduce the general position effects model. Rather than decoupling the intrinsic attraction of a product from its placement, this flexible model assigns a distinct, independent attraction parameter to every possible product-position pair $(i,k) \in [N] \times [K]$:
\[
  \alpha_{i,k} = v_{i,k}^\star \in (0,1].
\]
In particular, the multiplicative position effects model of Section~\ref{subsec:multiplicative_model} is a strict subclass of this general model, recovered by setting $v_{i,k}^\star = v_i^\star\,\theta_k^\star$, which justifies the term ``general''.
This unified parameter $v^\star_{i,k}$ directly encapsulates the overall attractiveness of product $i$ when it is displayed specifically at position $k$. Under this general model, the MNL choice probabilities retain the same functional form introduced in Eq.~\eqref{MNLmodel_choice_probability}, with $\alpha_{i,\sigma_t(i)}$ replaced by $v_{i,\sigma_t(i)}^\star$.

The multiplicative model has $N+K$ free parameters, whereas the general model has $NK$. This gap in parametric dimension foreshadows the gap between the $\widetilde{O}(\sqrt{NT})$ regret achievable under the multiplicative model (Section~\ref{sec:multiplicative_analysis}) and the $\widetilde{O}(\sqrt{KNT})$ rate that we establish under the general model (Section~\ref{sec:general_analysis}), and motivates studying both: the multiplicative model is cheaper to learn, while the general model is more flexible.

\section{Algorithm Design and Regret Analysis for Multiplicative Effects}
\label{sec:multiplicative_analysis}

In this section, we focus on the multiplicative position effects model (introduced in Section~\ref{subsec:multiplicative_model}) under the assumption that the position effects $\theta^\star_k$ are already known to the platform. This assumption is well motivated in practice: under the multiplicative model, the position effect depends only on the position and not on the specific product displayed, so the platform can estimate $\theta^\star_k$ offline by aggregating click-through or purchase rates---a strategy we validate in our empirical study on the Expedia Hotel Search dataset.
 Because such estimates capture a stable positional attribute, they remain valid even as the product catalog evolves, which allows the platform to disentangle a new product's intrinsic attraction from positional bias without re-estimating $\theta^\star_k$ from scratch.

In this setting, the learning challenge is to estimate the unknown intrinsic attraction parameters $\{v_i^\star\}_{i \in [N]}$ while simultaneously optimizing assortment and positioning decisions. A natural first attempt is to estimate $v_i^\star$ by dividing product $i$'s epoch-level purchase count (from the traditional epoch-based method) by the corresponding position effect $\theta_k^\star$. However, this approach fails because the learning process is adaptive: the assortments $(S_t,\sigma_t)$ offered in each round are chosen based on past observations, which induces dependence between product-position placements and the resulting feedback. Consequently, the positions at which a product is displayed across epochs are themselves dependent, so the resulting ratio-based estimators are not independent across epochs, and standard concentration inequalities for independent data do not apply.

Our proposed policy, P2MLE-UCB, addresses these difficulties by using a pairwise maximum likelihood estimator that pools cross-positional feedback and decouples the intrinsic attractions from the confounding effects of adaptive positioning.

\begin{algorithm}
    \caption{Position-aware Pairwise MLE UCB (P2MLE-UCB)}
    \label{alg:P2MLE-UCB}
    \begin{algorithmic}[1]
        \State \textbf{Initialization:} Input the horizon length $T$, the position effects $\bm{\theta}^\star= (\theta_1^\star,\dots,\theta_K^\star)^\top$. Set confidence parameter $\delta = 2/(3NT)$. Initialize $n_{i,k}^t, w_{i,k}^t$ and $D^t_i$ to $0$ for all $i \in [N],  k \in [K]$.
        
        \For{$t=1, \dots, T$}
            \For{$i \in [N]$}
                \If{$D_i^t > 0$}
                    \State Let $S^t_i (v)= \sum_{k=1}^K (w_{i,k}^t - n_{i,k}^t p_k (v))$ where $p_k (v)= \frac{v \theta^\star_k}{1+ v \theta^\star_k}$ for $v\ge 0$.
                    \State Let $\tilde{v}^t_i \ge 0$ be such that $S^t_i (\tilde{v}^t_i)=0$. {\color{blue}\Comment{Pairwise MLE}}
                    \State Let $\hat{v}_i^{t}= \min(\tilde{v}^t_i, 1)$. {\color{blue}\Comment{Clipping Mechanism}}
                    \State Set $v_i^{t,\mathrm{ucb}}$ as in Eq.~\eqref{eq:ucbdefinition}. {\color{blue}\Comment{UCB Construction}}
                \Else
                    $\ $ Set $v_i^{t,\mathrm{ucb}}=1$.
                \EndIf
            \EndFor
            \State Offer $(S_t,\sigma_t)=\arg \max_{(S,\sigma)} R(S,\sigma,\bm{v}^{t,\mathrm{ucb}}, \bm{\theta}^\star)$ and observe the customer choice $i_t$.
            \State Update $n_{i,k}^t$ by Eq.~\eqref{eq:definition_of_n^t_{i,k}}, $w_{i,k}^t$ by Eq.~\eqref{eq:definition_of_w^t_{i,k}}, for all $k \in [K]$. Update $D^t_i$ by Eq.~\eqref{eq:Dit}.
        \EndFor
    \end{algorithmic}
\end{algorithm}

\noindent\underline{\textbf{Pairwise maximum likelihood estimator.}} P2MLE-UCB differs from existing approaches in two respects. First, instead of the epoch-based paradigm that repeats an assortment until a no-purchase event occurs, we extract information from every customer interaction by treating each observed choice as a pairwise comparison between the displayed product and the outside option. Second, we center confidence intervals around a maximum likelihood estimator (MLE) of the intrinsic attractions, rather than applying concentration inequalities to a biased plug-in estimator.

To operationalize the pairwise approach, we track interaction statistics for each product $i \in [N]$ at every position $k \in [K]$. Specifically, we define
\begin{equation} \label{eq:definition_of_n^t_{i,k}}
n_{i,k}^t = \sum_{s=1}^{t-1} \mathbf{1}\{ i_s \in \{i,0\}, \ i \in S_s, \ \sigma_s(i) = k \}
\end{equation}
as the total number of times the product $i$ was displayed at position $k$ and the customer purchased $i$ or left without buying anything. Among these relevant interactions, we record the number of successful purchases as
\begin{equation} \label{eq:definition_of_w^t_{i,k}}
w_{i,k}^t = \sum_{s=1}^{t-1} \mathbf{1}\{ i_s = i, \ i \in S_s, \ \sigma_s(i) = k \}.
\end{equation}
Conditional on the event $\{i_s \in \{i,0\}, i \in S_s, \sigma_s (i)= k\}$, i.e., product $i$ was displayed at position $k$ and the customer either purchases $i$ or leaves, the purchase probability under the multiplicative model simplifies to
\[
p_{i,k} = \frac{v_i^\star \theta_k^\star}{1 + v_i^\star \theta_k^\star}.
\]
Because the assortment $(S_s,\sigma_s)$ is chosen adaptively from the history, the pairwise sample size $n_{i,k}^t$ is a random count rather than a fixed design size, and our concentration arguments treat the centered wins as a martingale-difference sequence (see Appendix~\ref{sec:probabilistic_setup}).

Building on this pairwise structure, we estimate the intrinsic attraction parameter $v^\star_i$ by pooling the feedback across all display positions. Let
\begin{equation}\label{eq:Dit}
    D_i^t = \sum_{k=1}^K n_{i,k}^t \theta_k^\star
\end{equation}
represent the effective exposure of product $i$ up to round $t$. Assuming $D_i^t> 0$, we form the (conditional) log-likelihood for $v_i^\star$ by treating each pairwise observation as Bernoulli with the corresponding conditional success probability $p_k(v_i^\star)$:
\[
\mathcal{L}^t_i(v) = \sum_{k=1}^K \left[ w_{i,k}^t \log\!\left(\frac{v\theta_k^\star}  {1+v\theta_k^\star}\right) + (n_{i,k}^t - w_{i,k}^t)
\log\!\left(\frac{1}{1+v\theta_k^\star}\right)
\right].
\]
Taking the derivative with respect to $v$ and simplifying yields the following score function:
\begin{equation} \label{eq:score_function}
S^t_i(v)
=
\sum_{k=1}^K \bigl( w_{i,k}^t - n_{i,k}^t \, p_k(v) \bigr),
\qquad
p_k(v) = \frac{v\theta_k^\star}{1+v\theta_k^\star}.
\end{equation}
Because the score function $S^t_i(v)$ is strictly decreasing in $v$, it admits a unique root $\tilde{v}_i^t$.

\noindent\underline{\textbf{Clipping mechanism.}} 
Relying on this unconstrained root, however, introduces both theoretical and practical difficulties. Translating concentration of the score function into a parameter error bound requires a uniform lower bound on its derivative, i.e., a positive-curvature (Fisher-information) condition on the log-likelihood. For the unconstrained MLE $\tilde{v}_i^t$, the derivative of $S^t_i$ scales as $1/(1+v'\theta_k^\star)^2$ and can become arbitrarily small as $v'$ grows, so even if $S^t_i(v^\star_i)$ concentrates tightly around zero, the parameter error $|\tilde{v}_i^t - v_i^\star|$ cannot be meaningfully controlled. This yields vacuous confidence intervals and precludes a sharp regret analysis. Practically, such estimates also cause numerical instability and can drive the algorithm to over-explore products that appear high-performing due to early noise.

To address this issue, we introduce a \emph{clipped} MLE by projecting the estimate onto the interval $[0,1]$:
\begin{equation} \label{eq:definition_of_clipped_MLE}
\hat{v}_i^t = \min(\tilde{v}_i^t, 1).
\end{equation}
The clipping keeps the estimate bounded; as we show next, this suffices to obtain a uniform lower bound on the derivative of the score function and hence sharp concentration bounds, while also stabilizing the algorithm in practice.

\noindent\underline{\textbf{Construction of the clipped MLE's UCB.}}  We now use the clipped MLE to build an upper confidence bound (UCB) on $v^\star_i$. If product $i$ has no effective exposure ($D^t_i = 0$), we set $v_i^{t,\mathrm{ucb}} = 1$ to encourage exploration. Otherwise, we build the UCB around $\hat{v}^t_i$ by exploiting a problem-specific structural property, the strict monotonicity of the score function $S^t_i(v)$ in $v$, which, together with the clipping operation, yields the following contraction inequality.

\begin{lemma}[Contraction of clipped MLE] \label{lem:contraction_of_clipped_MLE}
If $D^t_i = \sum_{k} n^t_{i,k} \theta_k^\star > 0$, let $S^t_i(v)$ be the score function defined in Eq.~\eqref{eq:score_function} and let $\hat{v}^t_i$ be the clipped MLE defined in Eq.~\eqref{eq:definition_of_clipped_MLE}. Then we have
\begin{equation} \label{eq:bound_estimation_error_by_|S|}
\left|S^t_i(v^\star_i)\right| \ge \frac{1}{4} D^t_i \cdot |\hat{v}^t_i-v^\star_i|. 
\end{equation}
\end{lemma}

This inequality provides a direct bridge to translate the concentration of the score function $S^t_i(v^\star_i)$ into a high-probability confidence interval for the underlying intrinsic attraction parameter $v^\star_i$. This relationship is formalized in the following lemma:

\begin{lemma}\label{concentrationwithknowntheta}
Let $c= 2(\lceil\log_2 (T/\theta^\star_{\mathrm{min}})\rceil+1)$. For $\delta \in (0,1)$, we have with probability at least $1-\delta$ that, simultaneously for all $t\in [T]$, $D^t_i>0$ implies that
\[
|\hat{v}^t_i-v^\star_i|\leq C_1\sqrt{\frac{v^\star_i \log(c/\delta)}{D^t_i}}+C_2\frac{\log(c/\delta)}{D^t_i},
\]
where $C_1= 8, C_2=8/3$. Consequently, we have
\[
\hat{v}^t_i\leq 2 v^\star_i + C_3\frac{\log(c/\delta)}{D^t_i}.
\]
where $C_3=56/3$.
\end{lemma}

Building on this concentration result, we construct the UCB for $v^\star_i$.

\begin{lemma}[UCB construction]\label{lem:ucb_construction_known_theta}
Fix an $i\in [N]$. Define 
\begin{equation}\label{eq:ucbdefinition}
    v_i^{t,\mathrm{ucb}} := \hat{v}^t_i + C_4 \sqrt{\frac{\hat{v}^t_i \log(c/\delta)}{D^t_i}} + C_5 \frac{\log(c/\delta)}{D^t_i},
\end{equation}
where $C_4 := 2 C_1 = 16, C_5 := C_1^2 + C_2 + 2 C_1 \sqrt{C_2} = (200 + 32 \sqrt6)/3$, then with probability at least $1-\delta$, simultaneously for all $t\in [T]$, $D^t_i >0$ implies that
\[
v^\star_i \leq v_i^{t,\mathrm{ucb}} \leq v^\star_i + C_6 \sqrt{\frac{v^\star_i \log(c/\delta)}{D^t_i}} + C_7 \frac{\log(c/\delta)}{D^t_i},
\]
where $C_6 := \sqrt{2} C_4  + C_1 = 8 + 16 \sqrt{2}, C_7 := C_4 \sqrt{C_3} + C_5 +C_2 = (208 + 32 \sqrt{6} + 32\sqrt{42})/3$.
\end{lemma}

Overall, P2MLE-UCB departs from existing algorithms for position-aware MNL bandits, which rely on epoch-based learning that treats a sequence of identical assortments as a single aggregated observation. By instead using every customer choice and centering data-dependent confidence bounds around a clipped MLE, our algorithm adapts to the effective exposure of each product and removes the $\sqrt{K}$ factor present in prior regret bounds.

\subsection{Regret Analysis}
\label{subsec:regret_mult}

Before stating the regret bound, we record a structural dominance lemma (Lemma~\ref{lem:structural_dominance}) that underpins the optimism step in both Theorem~\ref{regretboundofalg:P2MLE-UCB} (multiplicative model) and Theorem~\ref{regretboundofgeneralmodel} (general model). Because the MNL expected revenue $R(S,\sigma;(\alpha_{i,k}))$ is not coordinatewise monotone in the entries of $(\alpha_{i,k})$ for a fixed $(S,\sigma)$, raising the attraction of a low-revenue product on a fixed assortment can divert probability mass away from higher-revenue products and reduce $R$, an entrywise UCB on the attractions does not directly imply a UCB on $R(S^\star,\sigma^\star)$ for the fixed offline-optimal $(S^\star,\sigma^\star)$. Lemma~\ref{lem:structural_dominance} shows that, nevertheless, an entrywise UCB on the attractions implies a UCB on the \emph{optimal} revenue, where the optimization runs over $(S,\sigma)\in\mathcal{F}$ on each side. This is precisely the property required by the optimism-in-the-face-of-uncertainty argument.

The following theorem bounds the cumulative regret of P2MLE-UCB.

\begin{theorem} \label{regretboundofalg:P2MLE-UCB}
  Let $V_{\mathrm{max}}:= \max_{(S,\sigma)} \sum_{i \in S} v^\star_i \theta^\star_{\sigma (i)}$ denote the maximum total attraction of any feasible assortment. The expected cumulative regret of P2MLE-UCB (Algorithm \ref{alg:P2MLE-UCB}) is bounded by
  \[
    \mathrm{Regret}(T,\pi) \le 1 + N \cdot \max\{2 \log(3 N T^2) (1+V_{\mathrm{max}})^2, 2 K (1+V_{\mathrm{max}}) \}
    + \tilde{C} \sqrt{N T},
  \]
  where
  \[
    \tilde{C} = \left(C_6 (\theta^\star_{\mathrm{min}})^{-1/2} \sqrt{\log (3c NT /2)} + C_7 (\theta^\star_{\mathrm{min}})^{-1} \log (3 c NT /2) \right) \sqrt{2(1+\log T)},
  \]
  and $c= 2(\lceil\log_2 (T/\theta^\star_{\mathrm{min}})\rceil+1)$.
\end{theorem}

Theorem~\ref{regretboundofalg:P2MLE-UCB} gives a cumulative regret of order $\tilde{O}(\sqrt{NT})$, which matches the $\Omega(\sqrt{NT})$ minimax lower bound of \cite{luo2025rate} and is therefore rate-optimal for this setting. It also improves on the $\tilde{O}(\sqrt{KNT})$ upper bound of \cite{luo2025rate} for the multiplicative model with known position effects by a factor of $\sqrt{K}$, a gap that matters when $K$ is large, e.g., for long search result pages or multi-slot recommendation carousels. The improvement is obtained without stronger structural assumptions or forced exploration; it comes from (i) the round-based pairwise estimator, which uses every customer interaction rather than an entire epoch, and (ii) tighter MLE-based confidence intervals.

\noindent\underline{\textbf{Discussion on unknown position effects.}}
Under the multiplicative model, $\theta^\star_k$ depends only on the display position $k$ and not on the specific product shown, which makes it natural to treat $\bm{\theta}^\star$ and $\{v^\star_i\}$ asymmetrically in practice. In typical online deployments, the display interface itself is largely stable while the catalog churns: new products with unknown intrinsic attractions arrive frequently, whereas the slot-level visibility profile $\bm{\theta}^\star$ is a property of the page layout. Platforms can therefore obtain a reasonable estimate of $\bm{\theta}^\star$ once from historical interaction data and reuse it across catalogs, rather than relearning it for each new product set. Thus, we view ``known $\bm{\theta}^\star$'' as a reasonable modeling abstraction in this offline-then-online workflow, well motivated when historical data are abundant; standard estimation procedures, such as joint MLE on randomized-positioning logs, can deliver accurate estimates of $\bm{\theta}^\star$ in practice. Our empirical analysis on Expedia data uses a randomized-positioning subset in this spirit.

When no such historical data are available and $\bm{\theta}^\star$ must be learned online from scratch, our framework still accommodates this scenario in two ways. The first is a simple explore-then-commit (ETC) strategy: a brief initial exploration phase offers randomized assortments (following the heuristic of \cite{luo2025rate}) to obtain $\hat{\bm{\theta}}$, after which the estimate is plugged into P2MLE-UCB for the remainder of the horizon. While practical, this two-phase approach inherits the limitation of the EI-TLR heuristic of \cite{luo2025rate}, namely that no formal regret guarantee is established for the unknown-$\bm{\theta}^\star$ regime under the multiplicative model. The second is to fall back on the general-model algorithm of Section~\ref{subsec:general_model}, which makes no separability assumption: treating each product--position pair as carrying an independent attraction $v^\star_{i,k}$ that happens to satisfy $v^\star_{i,k} = v^\star_i\theta^\star_k$ (and noting that $v^\star_i,\theta^\star_k\in(0,1]$ automatically yields $v^\star_{i,k}\in(0,1]$), GP2-UCB applies without modification and yields a provable $\tilde{O}(\sqrt{KNT})$ regret, matched by the minimax lower bound $\Omega(\sqrt{KNT})$ established for the (strictly larger) general-model class in Theorem~\ref{thm:generalmodellb}. To the best of our knowledge, this is the \emph{first} formal regret guarantee for the multiplicative model with unknown $\bm{\theta}^\star$.
Whether the sharper $\tilde{O}(\sqrt{NT})$ rate, achievable when $\bm{\theta}^\star$ is known, can also be attained when $\bm{\theta}^\star$ is unknown by fully exploiting the multiplicative (rank-one) structure $v^\star_i\theta^\star_k$ remains a technically challenging open problem; the intrinsic difficulty is discussed in the conclusion.

\section{Algorithm Design and Regret Analysis for General Effects}
\label{sec:general_analysis}

We now present GP2-UCB (Algorithm~\ref{alg:GP2-UCB}), our online learning algorithm for the general position effects model.
GP2-UCB combines pairwise interaction tracking with Bernstein-type confidence bounds and optimistic assortment optimization. At the per-pair UCB level, the construction extends the round-based rank-breaking template of \cite{saha2024stop}, with statistics and UCBs indexed by product--position pairs $(i,k)$. The per-round decision, however, is a joint assortment-and-positioning assignment whose fractional combinatorial structure no longer reduces to a top-$m$ selection and is handled by the subroutine of Section~\ref{subsec:static_opt}. We will show (Theorems~\ref{regretboundofgeneralmodel}--\ref{thm:generalmodellb}) that GP2-UCB attains the minimax-optimal regret rate up to logarithmic factors.

\begin{algorithm}
\caption{General Position-aware Pairwise UCB (GP2-UCB)}
\label{alg:GP2-UCB}
\begin{algorithmic}[1]
\State \textbf{Initialization:} Input the horizon length $T$. Set $\delta=2/(3KNT)$. Set $L = \log(2(\lceil\log_2T\rceil+1)/\delta)$. Initialize $n^t_{i,k},w^t_{i,k}$ by $0$ for all $(i,k)\in [N]\times [K]$.

\For{$t =1, 2, \dots, T$}
    \For{each $(i,k) \in [N] \times [K]$}
        \If{$n_{i,k}^t > 0$} 
            \State Let $\hat{p}_{i,k}^t= \frac{w_{i,k}^t}{n_{i,k}^t}$ and let $p_{i,k}^{t,\mathrm{ucb}} := \min\left\{ \hat p_{i,k}^t + 2 \sqrt{\frac{\hat p_{i,k}^t(1-\hat p_{i,k}^t) L }{n_{i,k}^t}} + \frac{6 L}{n_{i,k}^t},\ \frac{1}{2} \right\}$.
            
            \State Construct the UCB of $v_{i,k}^\star$ as $v_{i,k}^{t,\mathrm{ucb}} := \frac{p_{i,k}^{t,\mathrm{ucb}}}{1- p_{i,k}^{t,\mathrm{ucb}}}$. {\color{blue}\Comment{UCB Construction}}
        \Else $\ $ Set $v_{i,k}^{t,\mathrm{ucb}}=1$.
        \EndIf
        
    \EndFor
    \State Offer $(S_t, \sigma_t) = \arg\max_{(S,\sigma)} R(S, \sigma, \bm{v}^{t,\mathrm{ucb}})$ and observe the customer choice $i_t$.
    \State Update $n_{i,k}^t$ by Eq.~\eqref{eq:definition_of_n^t_{i,k}}, $w_{i,k}^t$ by Eq.~\eqref{eq:definition_of_w^t_{i,k}} for all $(i,k) \in [N] \times [K]$.
\EndFor
\end{algorithmic}
\end{algorithm}

Like P2MLE-UCB, GP2-UCB is driven by pairwise product-vs-outside-option feedback; the key difference is that now each product--position pair carries its own unknown parameter $v^\star_{i,k}$, so statistics are indexed by pairs rather than pooled across positions. Given the interaction history $(S_s,\sigma_s,i_s)_{s=1}^{t-1}$, we define $n^t_{i,k},w^t_{i,k}$ the same as in Eqs.~(\ref{eq:definition_of_n^t_{i,k}}, \ref{eq:definition_of_w^t_{i,k}}).

Conditional on the rounds counted by $n^t_{i,k}$, i.e., those at which product $i$ is displayed at position $k$ and the customer either selects $i$ or leaves, the MNL choice reduces to a Bernoulli trial with success probability $p_{i,k} := v^\star_{i,k}/(1+v^\star_{i,k})$. Consequently, whenever a pair has been adequately sampled ($n^t_{i,k}>0$), the empirical ratio
\[
\hat p_{i,k}^t := \frac{w_{i,k}^t}{n_{i,k}^t}
\]
serves as an unbiased estimator of $p_{i,k}$ conditional on the filtration.

To balance exploration and exploitation, GP2-UCB constructs optimistic estimates through the pairwise success probability
\(p_{i,k}=v^\star_{i,k}/(1+v^\star_{i,k})\). If \(n^t_{i,k}=0\), we set
\(v_{i,k}^{t,\mathrm{ucb}}=1\). Otherwise, with
\(\hat p^t_{i,k}=w^t_{i,k}/n^t_{i,k}\), we use a variance-sensitive confidence bound for \(p_{i,k}\). Since the effective sample size \(n^t_{i,k}\) is generated adaptively, the bound is derived by applying Freedman's inequality to the associated martingale and a peeling argument; the variance term \(p_{i,k}(1-p_{i,k})\) is then replaced by its empirical counterpart \(\hat p^t_{i,k}(1-\hat p^t_{i,k})\). This yields
\[
p_{i,k}^{t,\mathrm{ucb}} :=
\min\left\{
\hat p_{i,k}^t
+ 2 \sqrt{\frac{\hat p_{i,k}^t(1-\hat p_{i,k}^t)L}{n_{i,k}^t}}
+ \frac{6L}{n_{i,k}^t},
\ \frac{1}{2}
\right\}.
\]
Compared with a Hoeffding-type radius, this bound adapts to the variance of the pairwise Bernoulli observations: its leading term scales as
\(\sqrt{p_{i,k}(1-p_{i,k})L/n^t_{i,k}}\), and is therefore substantially tighter when \(p_{i,k}\) is small.

Since $v^\star_{i,k}\le 1 \Leftrightarrow p_{i,k}\le 1/2$, the clip at $1/2$ projects $p_{i,k}^{t,\mathrm{ucb}}$ back into the feasible region for $p_{i,k}$, and the monotonicity of $x\mapsto x/(1-x)$ on $[0,1)$ then turns a UCB on $p_{i,k}$ into a UCB on $v^\star_{i,k}$:
\[
v_{i,k}^{t,\mathrm{ucb}} := \frac{p_{i,k}^{t,\mathrm{ucb}}}{1-p_{i,k}^{t,\mathrm{ucb}}}.
\]

Equipped with these optimistic utilities, GP2-UCB proceeds in each round $t$ to select the assortment and positioning that maximize expected revenue:
\[
(S_t,\sigma_t) \in \arg\max_{(S,\sigma)\in \mathcal{F}} R(S,\sigma; \bm{v}^{t,\mathrm{ucb}}).
\]
Although solving this fractional, combinatorial optimization problem at every step might seem daunting, we demonstrate in Section~\ref{subsec:static_opt} that it can be solved efficiently.

\subsection{Efficient Static Assortment-and-Positioning Optimization Subroutine}
\label{subsec:static_opt}

We now address the computational core of our online learning algorithm: the static assortment optimization problem under the general position effects  model. 
Recall that Algorithm~\ref{alg:GP2-UCB} must compute the assortment and positioning that maximize expected revenue under optimistic attraction estimates. This is where the position-aware setting differs from the round-based MNL literature: the position-free subproblem in \cite{saha2024stop} reduces to a top-$m$ selection, and the multiplicative model still benefits from the separable structure $v_i\theta_k$. In the general model, however, independent product--position attractions make both the selected set \(S\subseteq[N]\) and the injection \(\sigma:S\to[K]\) decision-relevant, turning the subproblem into a fractional two-sided assignment. An efficient subroutine for this joint optimization is therefore essential for making round-based learning operational in the general position-aware regime.

Given revenues $r_i \ge 0$ for each product $i\in [N]$ and attraction parameters $v_{i,k}\in [0,1]$ for each product-position pair $(i,k)\in [N]\times [K]$, our goal is to find the assignment that maximizes the expected revenue:
\[
  (S^\star,\sigma^\star)= \arg \max_{(S,\sigma) \in \mathcal{F}} R(S,\sigma,(v_{i,k})).
\]
To solve this, we define a binary indicator variable $x_{i,k}$, where $x_{i,k}=1$ if product $i$ is assigned to position $k$, and $0$ otherwise. The ``$1$'' in the denominator below corresponds to the outside option (with $v_0=1$ and $r_0=0$ absorbed). The resulting fractional integer program is:
\begin{align*}
  \max_{\bm{x}}\, & f(\bm{x})=\frac{\sum_{i=1}^N \sum_{k=1}^K r_i x_{i,k} v_{i,k}}{1+\sum_{i=1}^N \sum_{k=1}^K x_{i,k} v_{i,k}}\\
  \mathrm{s.t.} &\sum_{i=1}^N x_{i,k} \le 1, \forall k \in [K]\\
  & \sum_{k=1}^K x_{i,k} \le 1, \forall i \in [N]\\
  & x_{i,k} \in \{0,1\}.
\end{align*}

The objective function $f(\bm{x})$ belongs to the class of fractional combinatorial optimization problems. We solve this efficiently using a root-finding approach known as Dinkelbach's method \citep{dinkelbach1967nonlinear}.

Define a function $F(\lambda)$ as the maximum value of the numerator minus $\lambda$ times the denominator:
\begin{align*}
  F(\lambda) &= \max_{x \in \mathcal{X}} \left[\sum_{i,k} r_i v_{i,k} x_{i,k} - \lambda \left( 1 + \sum_{i,k} v_{i,k} x_{i,k} \right)\right] \\
  &= \max_{x \in \mathcal{X}} \left[ \sum_{i,k} (r_i - \lambda) v_{i,k} x_{i,k} \right] - \lambda
\end{align*}
where $\mathcal{X}$ is the set of valid matching constraints. The optimal revenue $R^\star$ is the unique value such that $F(R^\star)=0$.

For a fixed $\lambda$, finding $x$ that maximizes $\sum_{i,k} (r_i - \lambda) v_{i,k} x_{i,k}$ is equivalent to finding a maximum weight bipartite matching. Consider a bipartite graph with products on one side and positions on the other. An edge between product $i$ and position $k$ carries weight $w_{i,k}(\lambda)=(r_i-\lambda)v_{i,k}$, and we only include edges with positive weight since negative edges would only decrease the objective. This subproblem can be solved optimally in $O(NK^2)$ time (assuming $N\ge K$), for example by the Jonker--Volgenant algorithm \citep{jonker1987shortest} or by the rectangular Hungarian method \citep{bourgeois1971extension}.

\begin{algorithm}
  \caption{Static Optimization}
  \label{alg:staticopt}
  \begin{algorithmic}[1]
    \State \textbf{Initialize:} Set $\lambda$ to any lower bound of the optimal revenue $R^\star$ (e.g., $\lambda=0$ or the revenue of any feasible assortment), and choose a tolerance $\varepsilon \ge 0$.
    \State \textbf{Repeat:}
    \State \hspace*{\algorithmicindent} Compute edge weights $w_{i,k}= (r_i - \lambda) v_{i,k}$ for all $i,k$.
    \State \hspace*{\algorithmicindent} Solve the maximum weight bipartite matching to obtain assignment $\bm{x}^\star$.
    \State \hspace*{\algorithmicindent} If $\bigl|\sum_{i,k} w_{i,k} x^\star_{i,k}-\lambda\bigr|\le \varepsilon$, then \textbf{stop} (current $\lambda$ is optimal).
    \State \hspace*{\algorithmicindent} Otherwise, update $\lambda \leftarrow \frac{\sum r_i v_{i,k} x^\star_{i,k}}{1+ \sum v_{i,k} x^\star_{i,k}}$.

    \State \textbf{Return} the optimal assignment $\bm{x}^\star$.
  \end{algorithmic}
\end{algorithm}

With $\varepsilon=0$, Algorithm~\ref{alg:staticopt} returns the exact optimum $R^\star$ of the fractional integer program in finitely many iterations, and the iterates $\lambda^{(s)}\to R^\star$ superlinearly \citep{schaible1976fractional}.
In practice, this implies that the optimal solution is typically achieved in only a small number of iterations. Since each iteration reduces to a maximum-weight bipartite matching, Algorithm~\ref{alg:staticopt} serves as an efficient inner loop for GP2-UCB (Algorithm~\ref{alg:GP2-UCB}).

\subsection{Regret Analysis}
\label{subsec:regret_gen}

We begin our analysis by establishing the statistical concentration properties of our UCB estimators. First, we bound the estimation error of the success probabilities $p_{i,k}$ for each product-position pair.

\begin{lemma} \label{lemmaconcentrationofp}
Let $(i,k) \in [N] \times [K]$, and $\delta \in (0,1)$. Let $L = \log(2(\lceil\log_2T\rceil+1)/\delta)$. Then with probability at least $1-\delta$, simultaneously for all $t\in [T]$, $n^t_{i,k}>0$ implies that
\begin{equation}
p_{i,k} \le p_{i,k}^{t,\mathrm{ucb}} \le p_{i,k} + 4 \sqrt{\frac{p_{i,k} (1-p_{i,k}) L}{n_{i,k}^t}} + \frac{14 L}{n_{i,k}^t}. \label{concentrationofp}
\end{equation}
\end{lemma}

Building on this, the following lemma translates these probability bounds into tight confidence intervals for the underlying attraction parameters $v^\star_{i,k}$. The statement has two pieces: a deterministic burn-in threshold $n_0(v^\star_{i,k})$ below which the UCB is only guaranteed to be valid (not tight), and a Bernstein-type error bound that controls the asymptotic rate once $n^t_{i,k}$ clears this threshold.

\begin{lemma} \label{lemmaconcentrationofv}
Let $\delta\in (0,1)$, let $\tilde{L} = \log(2KN(\lceil\log_2 T\rceil+1)/\delta)$ and define the burn-in threshold
\[
n_0(v^\star_{i,k}) := 80 \tilde{L} (1 + v_{i,k}^\star).
\]
Then, with probability at least $1-\delta$, the following holds simultaneously for all $t \in [T]$ and $(i,k) \in [N] \times [K]$: $v_{i,k}^\star \le v_{i,k}^{t,\mathrm{ucb}}$ and, $n^t_{i,k} \ge n_0(v^\star_{i,k})$ implies that
\[
v_{i,k}^{t,\mathrm{ucb}} - v_{i,k}^\star \le 8 (1+ v_{i,k}^\star) \sqrt{\frac{ v_{i,k}^\star \tilde{L}}{n_{i,k}^t}} + 28 \tilde{L} \frac{(1+ v_{i,k}^\star)^2}{n_{i,k}^t}.
\]
\end{lemma}
Intuitively, the burn-in term $n_0(v^\star_{i,k})$ absorbs the cost of insufficient initial exposure, while the Bernstein-type bound captures the asymptotic rate at which the UCB tightens around $v^\star_{i,k}$. This two-part structure is what produces the additive $KN(1+K)^2$ lower-order term in Theorem~\ref{regretboundofgeneralmodel}.

Because $v_{i,k}^{t,\mathrm{ucb}}$ is a valid upper confidence bound on $v_{i,k}^\star$, selecting the assortment that maximizes expected revenue under $\bm{v}^{t,\mathrm{ucb}}$ yields an implicit trade-off between exploration and exploitation: insufficiently sampled product-position pairs retain wide confidence bounds and are over-weighted in the maximization, while well-sampled pairs are evaluated at estimates close to their true utilities.

The following theorem bounds the expected cumulative regret of GP2-UCB.

\begin{theorem} \label{regretboundofgeneralmodel}
The expected cumulative regret of GP2-UCB (Algorithm~\ref{alg:GP2-UCB}) is at most
\begin{align*}
\text{Regret}(T, \pi) &\le 1 + KN\cdot \max  \{2 \log(3 K N T^2) (1+V_{\mathrm{max}})^2,  160 \log (3 K N T (\lceil \log_2 T \rceil + 1)) (1+K)^2\} \\
& \quad + 4 \sqrt{21 (1 + \log T) \cdot \log(3 K N T (\lceil \log_2 T \rceil + 1) )} \cdot \sqrt{K N T} \\
&= \tilde{O}(\sqrt{K N T}).
\end{align*}
\end{theorem}
The leading term of the above regret bound is of order $\tilde{O}(\sqrt{KNT})$. To demonstrate the optimality, we also establish a fundamental limit on the performance of any online learning policy under the general position effects  model in the following theorem.

\begin{theorem} \label{thm:generalmodellb}
Suppose $K\le N/4$ and $T \ge \frac{4}{243} KN$. There exists an absolute constant $\eta\ge 10^{-3}$ such that: for any admissible policy $\pi$, there exists an instance of the general position effects model such that the regret of $\pi$ is at least $\eta \sqrt{K N T}$.
\end{theorem}

Comparing Theorem~\ref{thm:generalmodellb} with Theorem~\ref{regretboundofgeneralmodel} yields two implications. First, GP2-UCB is rate-optimal: its $\tilde{O}(\sqrt{KNT})$ upper bound matches the $\Omega(\sqrt{KNT})$ lower bound up to logarithmic factors. Second, relative to the $\sqrt{NT}$ rate achievable under the multiplicative model with known $\bm{\theta}^\star$ (Section~\ref{subsec:multiplicative_model}), the general model incurs an additional $\sqrt{K}$ factor, which Theorem~\ref{thm:generalmodellb} shows cannot be removed: no policy can achieve a rate better than $\Omega(\sqrt{KNT})$ without additional structural assumptions on how display positions affect product attraction.

\section{Numerical Experiments with Synthetic Data}
\label{sec:simulations}

In this section, we conduct a series of numerical experiments on synthetic datasets to evaluate the practical performance of our proposed policies. Our primary objective is to validate the theoretical regret bounds established in the preceding sections and to compare our algorithms against state-of-the-art benchmarks in the assortment learning literature.

Our experimental suite is divided into three parts:
\begin{enumerate}
  \item Known position effects: we evaluate P2MLE-UCB (Algorithm~\ref{alg:P2MLE-UCB}) against two benchmarks: TLR-UCB, proposed by \cite{luo2025rate}, and A-UCB-V, a variant of the UCB algorithm in \cite{agrawal2019mnl} adapted by exploiting the known position effects in the static assortment optimization, which is also a benchmark used by \cite{luo2025rate}.
  
  \item Unknown position effects: we evaluate E-P2MLE-UCB, an explore-then-exploit variant of P2MLE-UCB, against the EI-TLR benchmark proposed by \cite{luo2025rate}.
  
  \item General position effects: we evaluate GP2-UCB (Algorithm~\ref{alg:GP2-UCB}) against A-UCB-Gen, a benchmark we obtain by adapting the UCB algorithm of \cite{agrawal2019mnl} to the general model.
\end{enumerate}

\subsection{Comparison Under Known Position Effects} \label{sectionfig:comparisonunderknownpositioneffects}

We begin our empirical evaluation by considering the setting described in Section~\ref{subsec:multiplicative_model}, where the position effects $\theta^\star_k$ are known to the platform. In this scenario, we compare the performance of our proposed P2MLE-UCB (Algorithm~\ref{alg:P2MLE-UCB}) against two representative benchmark policies from the literature.

The first benchmark is TLR-UCB, introduced by \cite{luo2025rate}, which is specifically designed for the multiplicative effects model. The second benchmark is A-UCB-V, a variant of the UCB algorithm proposed by \cite{agrawal2019mnl}, which we adapt to incorporate known position effects during both the estimation of the intrinsic attraction parameters and the static assortment optimization to ensure a fair comparison. All algorithms are evaluated on three synthetic instances with increasing complexity, as detailed below.

In Example 1, we consider a small-scale setting with $N=3$ products and $K=2$ display positions. The position effects vector is set to $\theta^\star=(1,1/2)^\top$, the intrinsic attraction parameters are $v^\star=(1/4,2/5,4/5)^\top$, and the associated revenues are $r=(4/5,3/4,1/2)^\top$. Example 2 expands the set of products to $N=5$ and the number of positions to $K=3$. The position effects vector is set to $\theta^\star=(1,1/2,1/3)^\top$, the intrinsic attraction parameters are $v^\star=(1,4/5,3/5, 2/5, 1/5)^\top$, and the associated revenues are $r=(2/5, 1/5, 4/5, 3/5 ,1/5)^\top$. The parameters chosen in Examples 1 and 2 are the same as those in \cite{luo2025rate} for consistency with previous work. Finally, Example 3 scales up to $N=30$ products and $K=10$ positions, offering a more challenging environment for learning and optimization. In this case, the position effects follow a linear decay $\theta_k=(11-k)/10$, the intrinsic attraction parameters are set as $v_i=(31- i)/30$, and the revenues are given by $r_i=(i+9)/40$.

\begin{figure}
  \centering

  \begin{subfigure}[b]{0.45\textwidth}
    \centering
    \includegraphics[width=\textwidth]{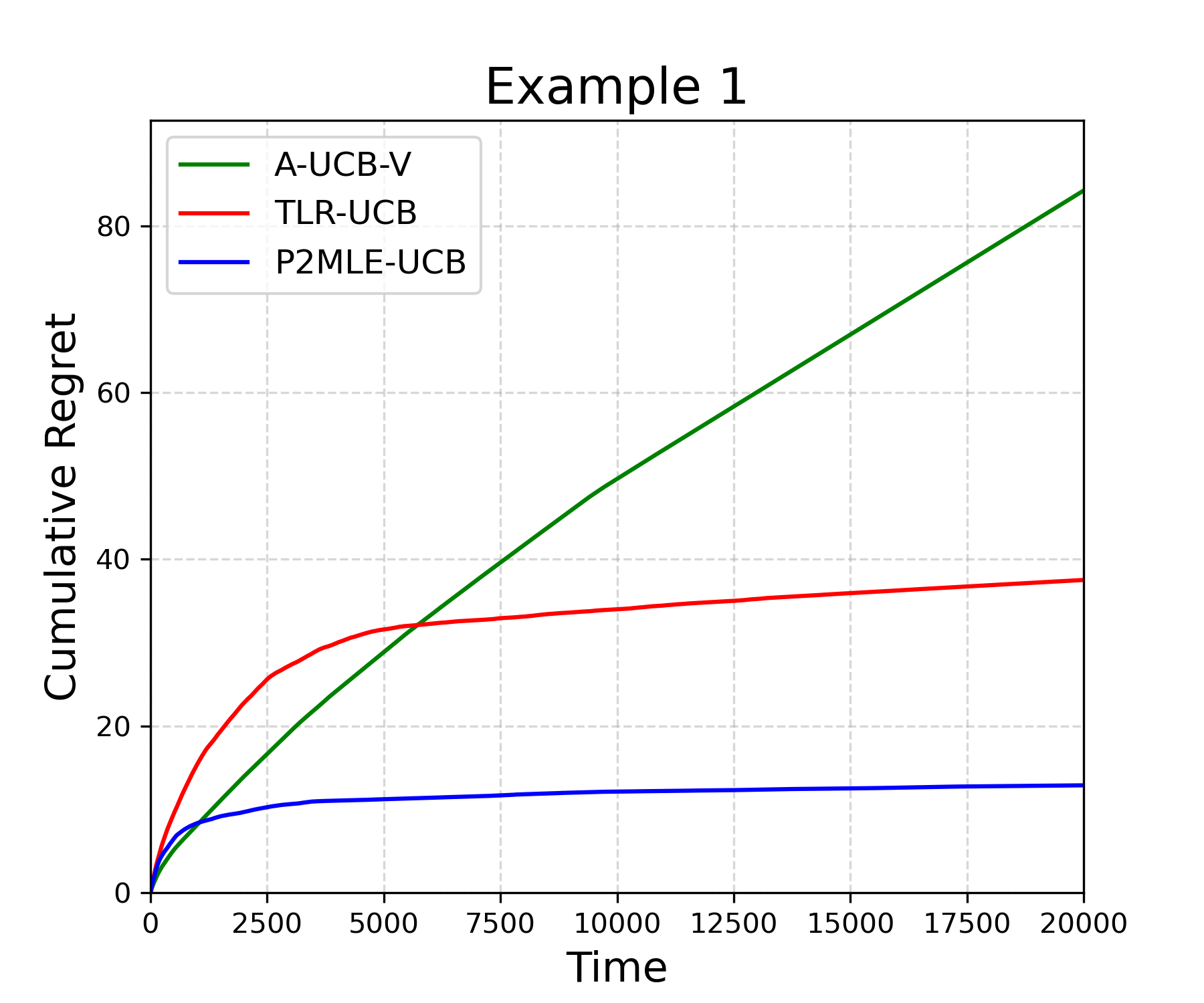}
  \end{subfigure}
  \hfill
  \begin{subfigure}[b]{0.45\textwidth}
    \centering
    \includegraphics[width=\textwidth]{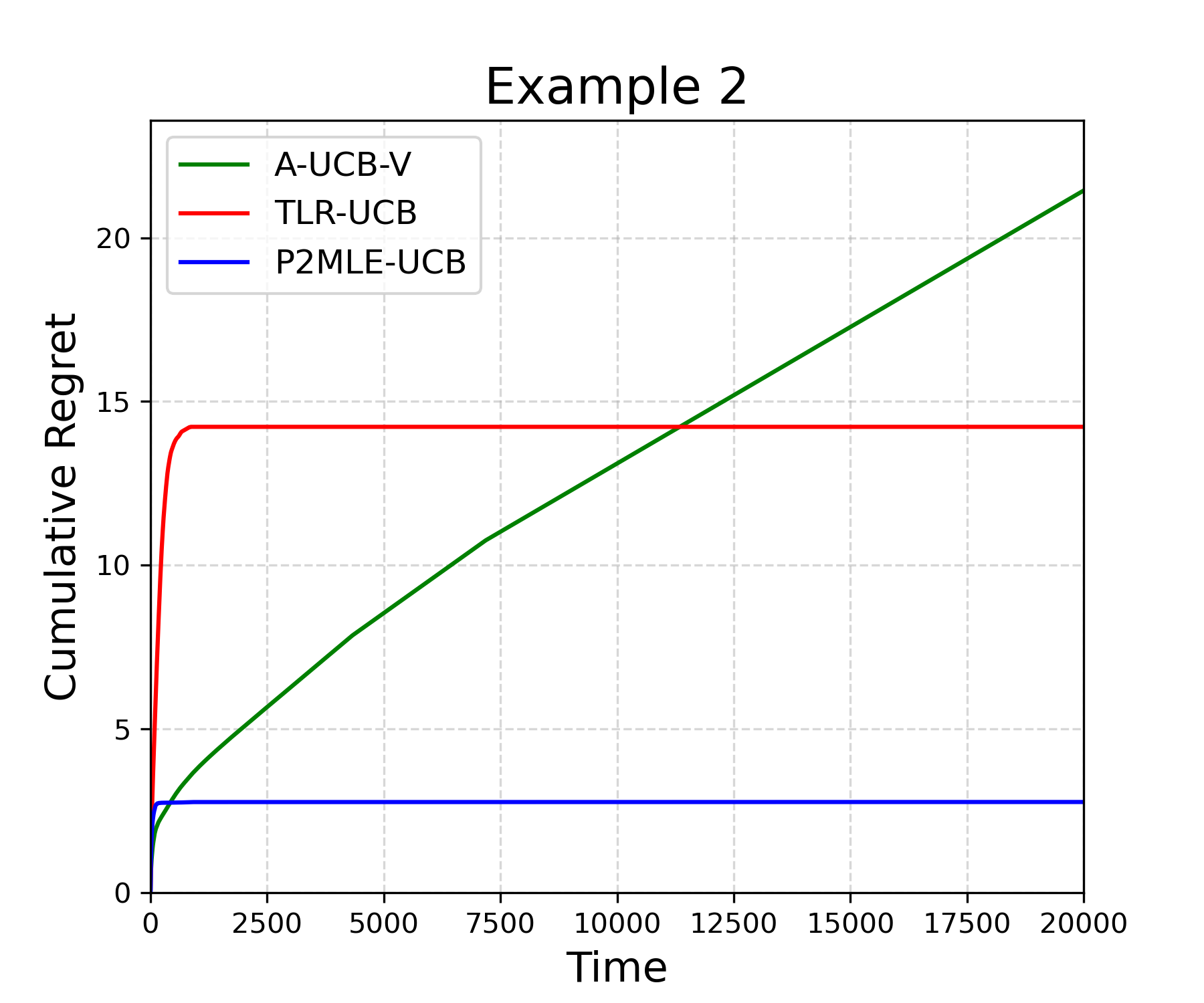}
  \end{subfigure}
  \hfill
  \\
  \begin{subfigure}[b]{0.45\textwidth}
    \centering
    \includegraphics[width=\textwidth]{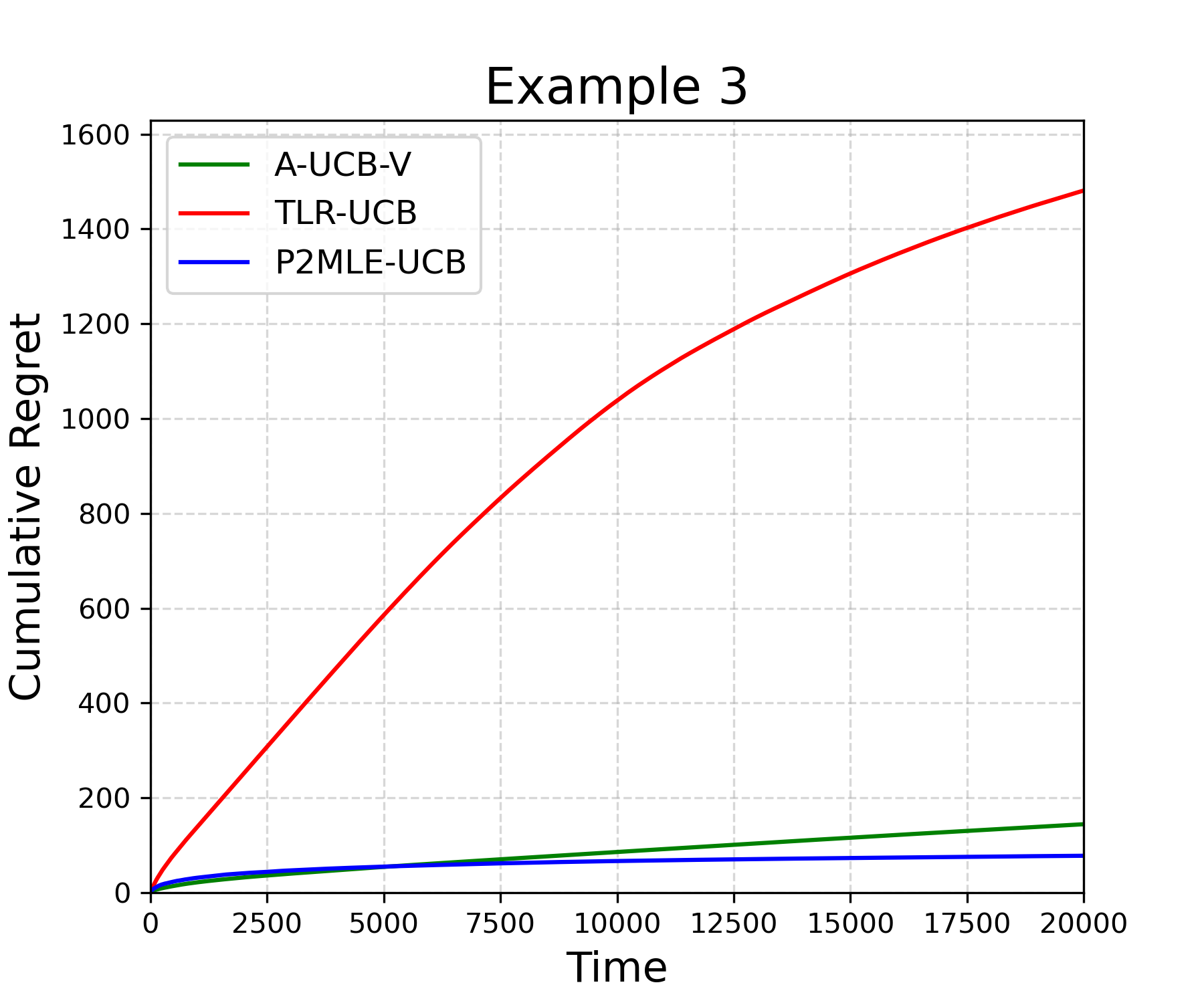}
  \end{subfigure}

  \caption{Regret Comparison Under Known Position Effects}
  \label{fig:comparisonunderknownpositioneffects}
\end{figure}

Figure~\ref{fig:comparisonunderknownpositioneffects} reports the cumulative regret curves for each algorithm, averaged over 50 independent runs, across the three examples. P2MLE-UCB achieves lower cumulative regret than both TLR-UCB and A-UCB-V across all settings. Moreover, its regret curve flattens quickly, indicating that near-optimal assortments are identified early in the horizon.

In Examples 1 and 2, TLR-UCB requires a substantial number of rounds before achieving lower regret than A-UCB-V. In contrast, P2MLE-UCB overtakes A-UCB-V quickly and then stabilizes with little subsequent growth.

In the larger-scale Example 3, the separation among the algorithms becomes more pronounced. Both A-UCB-V and TLR-UCB continue to accumulate regret throughout the horizon, with TLR-UCB requiring an even longer burn-in period than in Examples 1 and 2 before it would overtake A-UCB-V. P2MLE-UCB, in contrast, maintains a nearly flat regret trajectory after a short initial period.

\subsection{Comparison Under Unknown Position Effects} \label{subsectioncomparisonthetaunknown}

We now consider the setting where the position effects $\bm{\theta}^\star$ are unknown to the platform. In this scenario, the learner must simultaneously estimate both the intrinsic product attractions and the position effects, which introduces additional identifiability challenges.

For this setting, \cite{luo2025rate} proposed an explore-then-exploit policy called EI-TLR. The algorithm first executes an exploration phase of $J_0(T)$ rounds to obtain initial estimates $\hat{\bm{\theta}}$ of the position effects. After this phase, it treats $\hat{\bm{\theta}}$ as the true position effects and switches to TLR-UCB, which is their algorithm for the case of known position effects. To evaluate the robustness of our approach to estimation errors in the position effects, we construct a comparable variant of our own algorithm. Specifically, we introduce E-P2MLE-UCB, which follows the same two-phase structure: an initial exploration phase to estimate $\bm{\theta}^\star$, followed by an exploitation phase that runs the standard P2MLE-UCB (Algorithm \ref{alg:P2MLE-UCB}) with the estimated $\hat{\bm{\theta}}$ treated as the true position effects. This design allows us to isolate the impact of position effect estimation errors on the subsequent UCB-based learning.

We compare E-P2MLE-UCB against EI-TLR. To ensure a fair comparison, we synchronize the exploration budgets of both algorithms by setting the number of exploration rounds to $J_0(T) = \lceil c \cdot T^{1/2} \rceil$ with $c = 0.1$, following the same scheme used in \cite{luo2025rate}. This choice balances the need for sufficiently accurate initial estimates against the regret incurred during exploration. All other experimental parameters are kept identical to those in Examples 1–3 from the previous subsection.

\begin{figure}
  \centering

  \begin{subfigure}[b]{0.45\textwidth}
    \centering
    \includegraphics[width=\textwidth]{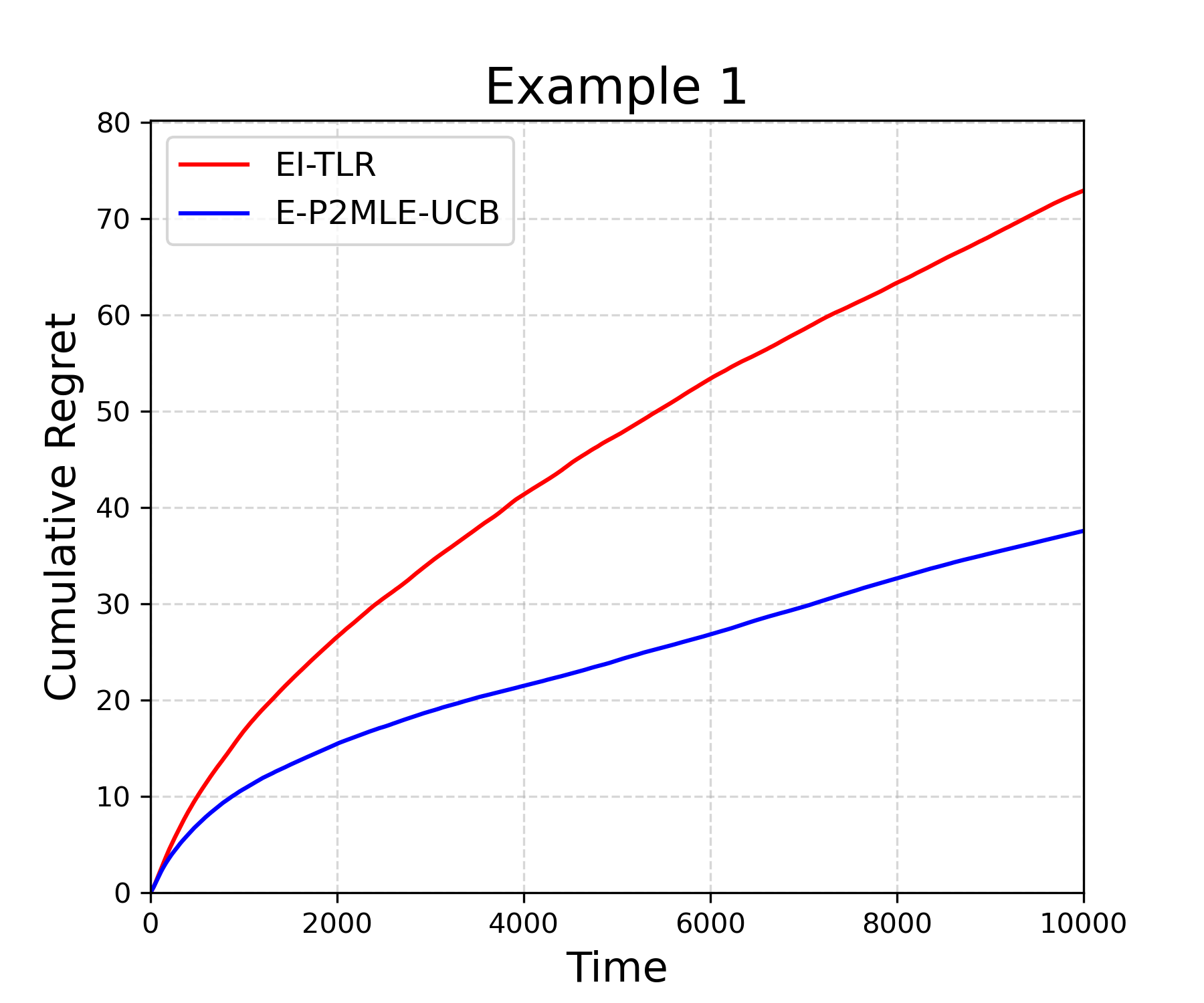}
  \end{subfigure}
  \hfill
  \begin{subfigure}[b]{0.45\textwidth}
    \centering
    \includegraphics[width=\textwidth]{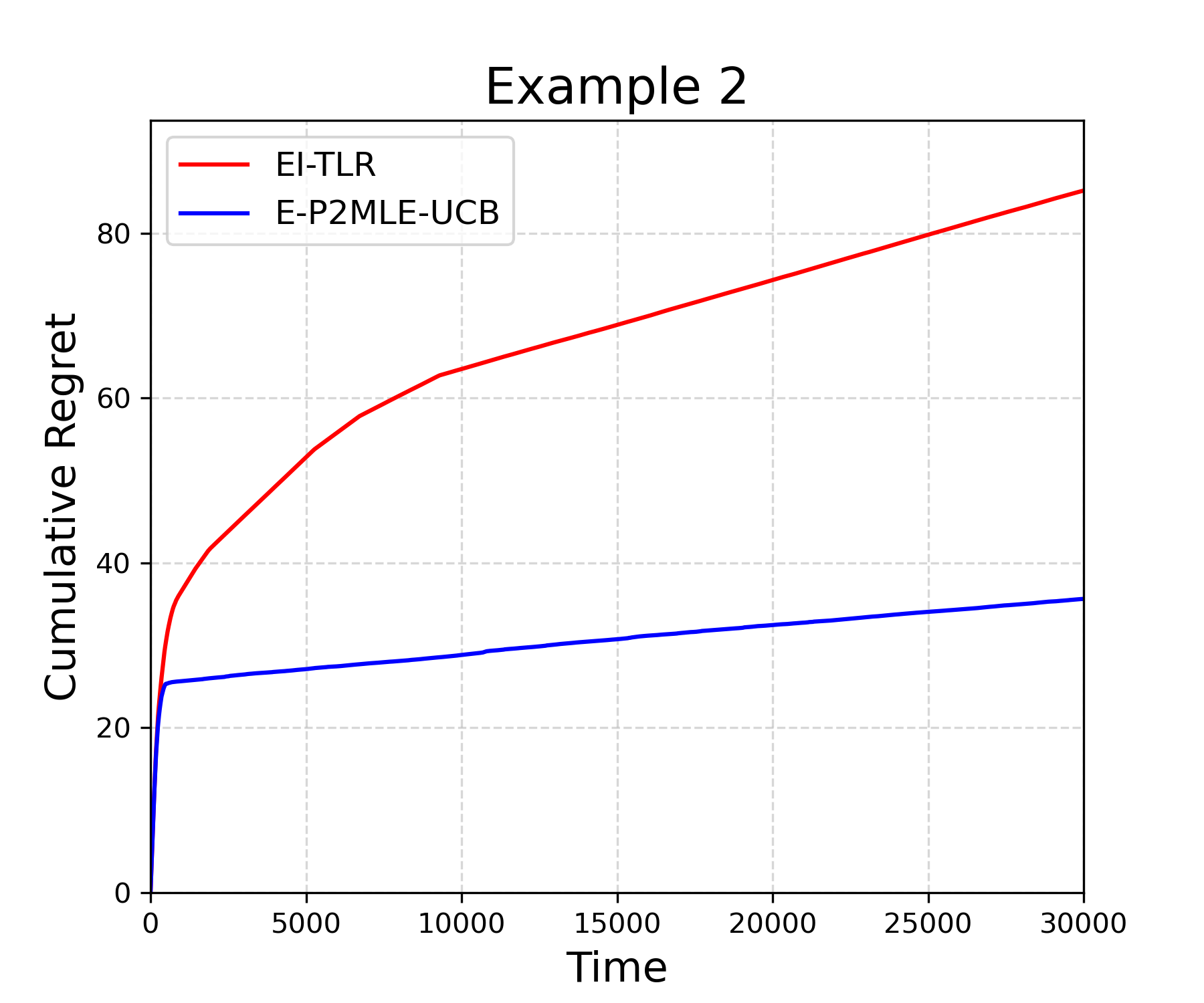}
  \end{subfigure}
  \hfill
  \\
  \begin{subfigure}[b]{0.45\textwidth}
    \centering
    \includegraphics[width=\textwidth]{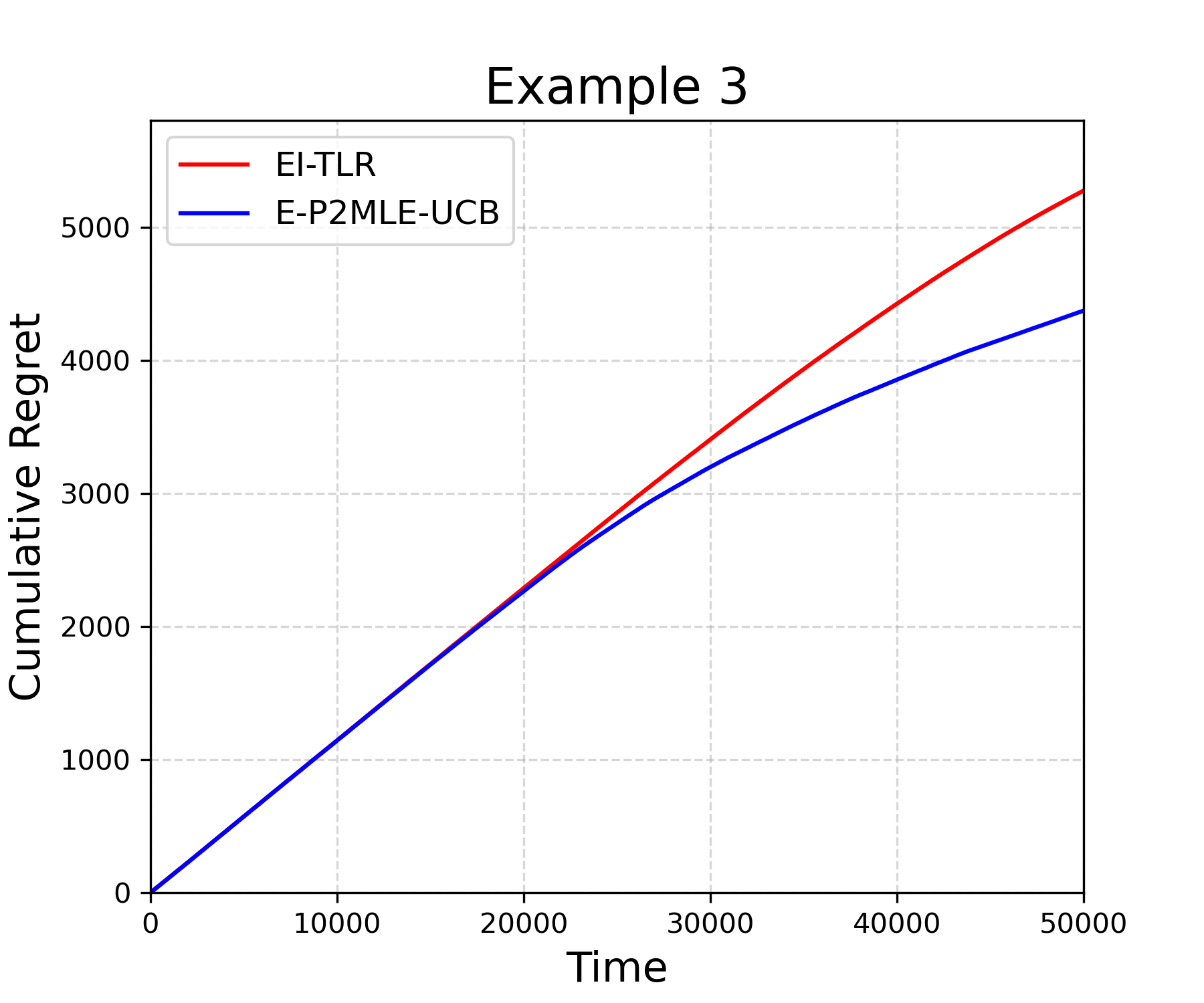}
  \end{subfigure}

  \caption{Regret Comparison Under Unknown Position Effects}
  \label{fig:comparisonunderunknownpositioneffects}
\end{figure}

Figure~\ref{fig:comparisonunderunknownpositioneffects} shows the cumulative regret curves for both algorithms under unknown position effects, with each curve representing the average over 50 independent replications. Both methods exhibit an initial linear growth phase corresponding to the forced exploration, during which suboptimal assortments are offered to collect information about $\bm{\theta}^\star$. After the transition to the exploitation phase, the regret trajectory becomes sublinear for both methods, with a visible kink at the phase boundary.

Despite using identical exploration budgets and thus comparable initial estimates $\hat{\bm\theta}$, E-P2MLE-UCB achieves lower regret than EI-TLR during the exploitation phase across all three examples. This suggests that the UCB updates of E-P2MLE-UCB are less sensitive to residual error in $\hat{\bm\theta}$ than the transition-layer mechanism used by EI-TLR. Taken together with Section~\ref{sectionfig:comparisonunderknownpositioneffects}, the advantages of pairwise estimation extend to the unknown-position-effects setting.

\subsection{Comparison Under the General Model}
\label{subsec:synthetic_general_model}

In this set of synthetic experiments, we evaluate our algorithms under the general position effects model introduced in Section~\ref{subsec:general_model}. In this setting, the learner needs to estimate $N \times K$ independent attraction parameters $v_{i,k}^\star$, each corresponding to a specific product-position pair.

Our proposed policy for this setting is GP2-UCB (Algorithm~\ref{alg:GP2-UCB}), which constructs Bernstein-style confidence bounds for purchase probabilities and solves the resulting fractional assortment optimization problem using the iterative bipartite matching subroutine developed in Section~\ref{subsec:static_opt}. We compare GP2-UCB against a benchmark policy, A-UCB-Gen, which adapts the epoch-based UCB algorithm from \cite{agrawal2019mnl} to the general model. Specifically, A-UCB-Gen treats each product-position pair $(i,k)$ as a distinct ``virtual item'', estimates utilities based on no-purchase epochs, and solves the same optimization problem using the subroutine described in Section~\ref{subsec:static_opt} to ensure a fair comparison.

We construct three test instances with attraction parameter matrices that do not follow a multiplicative structure. In these instances, a product may have different relative performance across positions (e.g., performing better in a middle slot than in a top slot).

Example 4 involves $N= 5$ products and $K=3$ positions, with revenue vector $r= (0.9, 0.8, 0.9, 0.6, 0.5)^\top$. Example 5 scales up to $N= 8$ products and $K=4$ positions, with revenue vector $r=(0.9, 0.8, 0.9, 0.6, 0.5, 0.7, 0.4, 0.3)^\top$. Example 6 further increases the complexity to $N= 10$ products and $K=5$ positions, with revenue vector $r=(0.9, 0.8, 0.9, 0.7, 0.6, 0.5, 0.7, 0.4, 0.6, 0.3)^\top$. Their respective attraction parameter matrices $(v_{i,k}^\star)$ are given below:
\[
\setlength{\arraycolsep}{8pt}
\renewcommand{\arraystretch}{1.1}
\underbrace{
\begin{pmatrix}
0.4 & 0.1& 0.1\\
0.1& 0.5& 0.1\\
0.2& 0.2& 0.6\\
0.3& 0.1& 0.4\\
0.1& 0.1& 0.1
\end{pmatrix}
}_{\text{Example 4}}, \quad
\underbrace{
\begin{pmatrix}
0.8& 0.6& 0.5& 0.2\\
0.1& 0.5& 0.9& 0.3\\
0.6& 0.2& 0.6& 0.1\\
0.3& 0.1& 0.4& 0.5\\
0.7& 0.1& 0.1& 0.8\\
0.2& 0.5& 0.4& 0.6\\
0.4& 0.3& 0.8& 0.2\\
0.1& 0.1& 0.1& 0.1
\end{pmatrix}
}_{\text{Example 5}}, \quad 
\underbrace{
\begin{pmatrix}
0.8& 0.6& 0.5& 0.2& 0.1\\
0.4& 0.5& 0.9& 0.3& 0.2\\
0.6& 0.3& 0.6& 0.1& 0.3\\
0.3& 0.7& 0.4& 0.5& 0.4\\
0.7& 0.1& 0.2& 0.8& 0.5\\
0.3& 0.5& 0.4& 0.6& 0.4\\
0.4& 0.4& 0.8& 0.2& 0.3\\
0.6& 0.1& 0.2& 0.1& 0.1\\
0.2& 0.3& 0.1& 0.4& 0.2\\
0.5& 0.4& 0.3& 0.1& 0.1
\end{pmatrix}
}_{\text{Example 6}}.
\]

\begin{figure}
     \centering

     \begin{subfigure}[b]{0.45\textwidth}
         \centering
         \includegraphics[width=\textwidth]{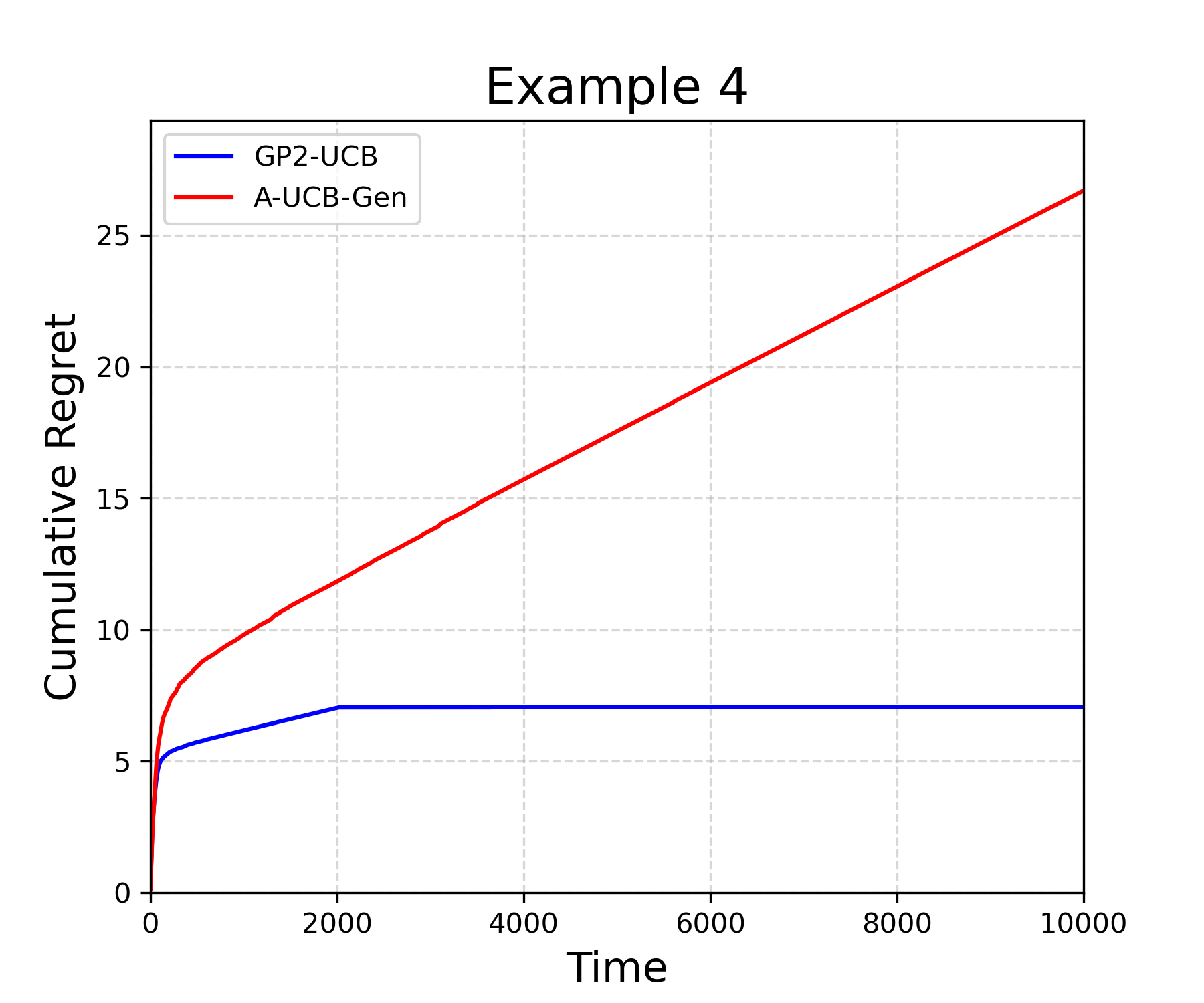}
     \end{subfigure}
     \hfill
     \begin{subfigure}[b]{0.45\textwidth}
         \centering
         \includegraphics[width=\textwidth]{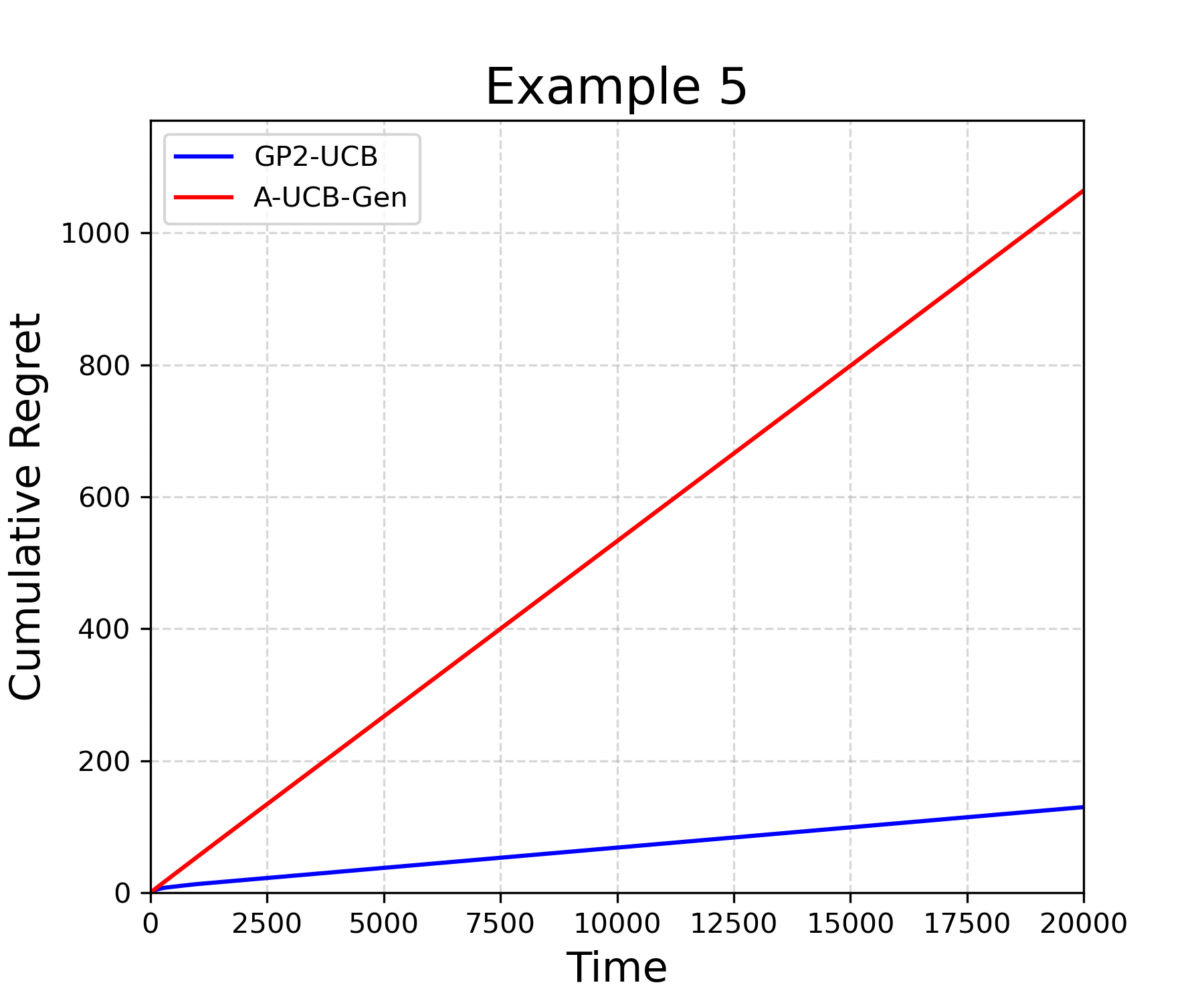}
     \end{subfigure}
     \hfill
     \\
     \begin{subfigure}[b]{0.45\textwidth}
         \centering
         \includegraphics[width=\textwidth]{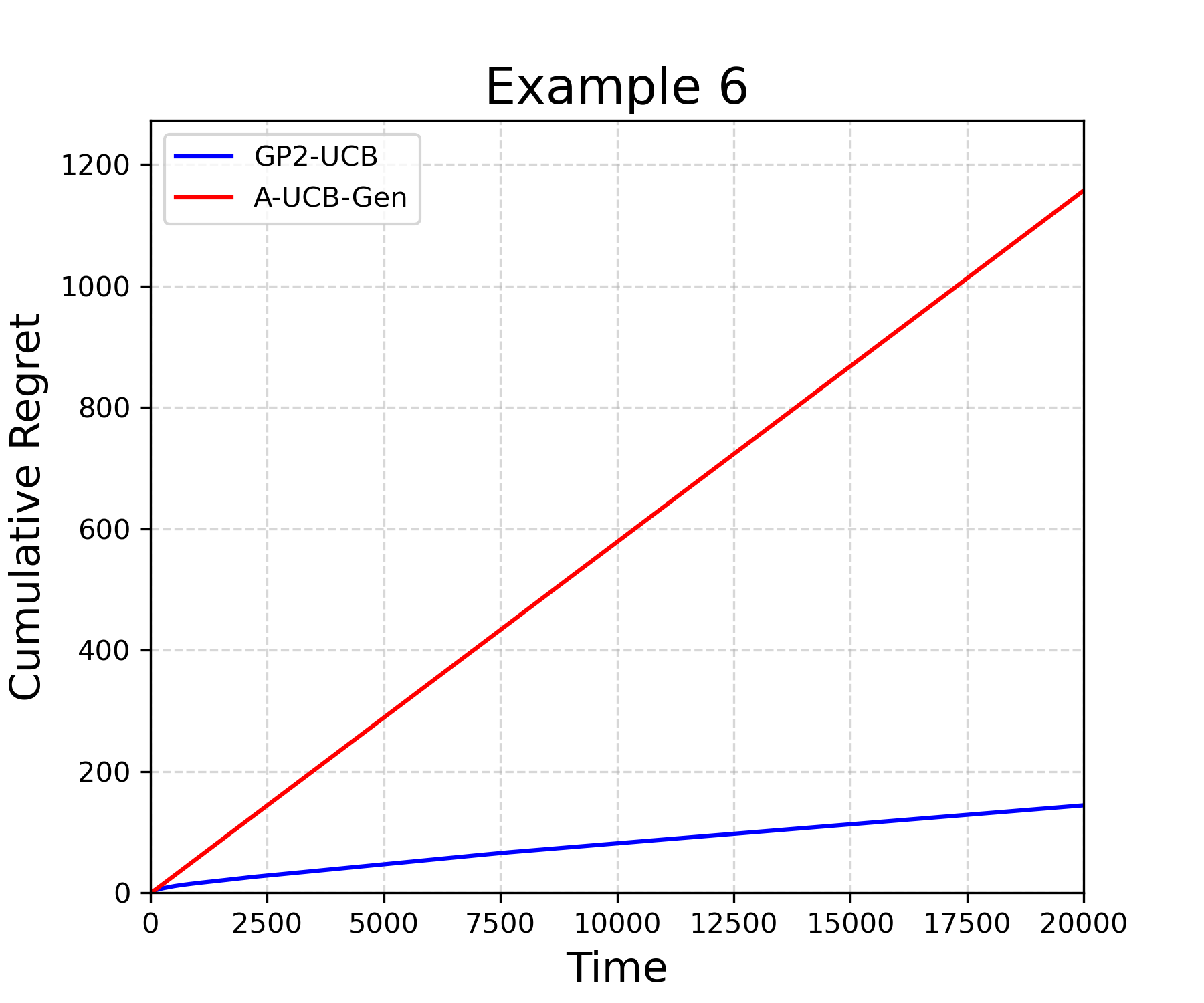}
     \end{subfigure}

     \caption{Regret Comparison Under the General Model}
     \label{fig:comparisonundergeneralmodel}
\end{figure}

Figure~\ref{fig:comparisonundergeneralmodel} shows the cumulative regret curves, averaged over 50 independent runs, for GP2-UCB and A-UCB-Gen across these three examples. In all instances, GP2-UCB achieves sublinear regret after an initial period. A-UCB-Gen exhibits regret growth that is close to linear over most of the time horizon. The difference in regret between the two algorithms increases with the problem scale, particularly in Examples 5 and 6.

These results indicate that the advantages of round-based pairwise estimation persist under the general position effects model. Despite the larger parameter space of $N\times K$ independent attractions, GP2-UCB attains substantially lower regret than the epoch-based benchmark A-UCB-Gen across all three instances.

\section{Empirical Study with Expedia Hotel Search Dataset}
\label{sec:expedia_case}

To complement the synthetic experiments in Section~\ref{sec:simulations}, we now evaluate our algorithms on test instances derived from a large-scale, real-world dataset. This allows us to assess whether the performance gains reported earlier persist when the position effects, intrinsic attractions, and revenue structures are calibrated to distributions observed in actual e-commerce traffic.

\subsection{Dataset Description}

We evaluate the proposed algorithms using the Expedia Hotel Search Dataset, originally released for the ICDM 2013 competition. The dataset contains approximately 10 million hotel impressions from the Expedia website. Each row in the dataset corresponds to a specific hotel within a specific search result.

We use the following columns from the dataset. The column `\texttt{prop\_id}' uniquely identifies a hotel. The column `\texttt{position}' indicates the display rank of the hotel in the search result page, starting from $1$. The column `\texttt{click\_bool}' is a binary indicator equal to $1$ if the user clicked on the hotel listing and 0 otherwise. The column `\texttt{random\_bool}' is a binary indicator equal to $1$ if the search results were displayed in a random order and 0 if they were sorted by Expedia's internal ranking algorithm; approximately 30\% of searches are randomized. The column `\texttt{price\_usd}' gives the displayed price of the hotel for the given search, in US dollars.

We restrict our analysis to the subset of searches where $\texttt{random\_bool} = 1$. In this subset, the display order is random, so the observed relationship between `\texttt{position}' and `\texttt{click\_bool}' reflects pure position bias rather than the effect of algorithmic ranking. This allows us to estimate position effects $\theta_k^\star$ and intrinsic product attractions $v_i^\star$ without confounding. We use `\texttt{price\_usd}' as a proxy for revenue $r_i$, normalized to the unit interval after capping extreme values.

In the following subsections, we describe how we extract ground-truth parameters for the multiplicative position effects model from these columns and construct simulation environments.

\subsection{Parameter Extraction Methodology}
\label{subsec:parameter_extraction}

We estimate empirical calibration parameters $(\bm{v}^\star,\bm{\theta}^\star, r)$ from the Expedia dataset using the randomized subset of searches identified by the `\texttt{random\_bool}' flag. In this subset, the display order is random, so the observed relationship between position and click probability is less confounded by algorithmic ranking. 

We first estimate the position effects $\theta^\star_k$ for the multiplicative model. For each display position $k \in [K]$, we compute the average click-through rate across all impressions in the randomized subset.\footnote{In principle, the most principled way to recover $\bm{\theta}^\star$ under MNL is a joint maximum-likelihood estimator that simultaneously fits $\bm{\theta}^\star$ and $\{v^\star_i\}$ using the full assortment of co-displayed hotels in each search session. The Expedia release, however, records only per-impression click outcomes and does not preserve enough session-level structure to reconstruct the offered assortments needed for joint MLE. We therefore fall back on the per-position average click-through rate as a practical proxy; this proxy is exactly proportional to $\theta^\star_k$ when the co-displayed assortment is balanced across positions, which the randomization in the \texttt{random\_bool} subset is designed to approximate.}
This provides a useful signal that is typically proportional to positional prominence. We then normalize these values so that $\theta^\star_k \in (0,1]$ for all $k$, with $\theta^\star_1 = 1$ for the top position.

Next, we estimate the intrinsic product attraction parameters $v^\star_i$. For each hotel identified by `\texttt{prop\_id}', we compute its average click rate across all occurrences in the randomized subset. 
This quantity is proportional to the hotel's inherent appeal, and for simulation seeding, this proportional signal is sufficient, so we normalize the resulting values by the maximum observed average click rate to obtain $v^\star_i \in (0,1]$. Additionally, we impose a lower bound of $0.01$ to avoid numerical instability during UCB updates, a precaution that proves useful when products have very low empirical purchase probabilities.

Finally, we construct the revenue parameters $r_i$ using the `\texttt{price\_usd}' column. We cap prices at the 95th percentile of the empirical distribution, which is approximately 400 USD in the dataset, and then divide all prices by this cap. This yields normalized revenues in $[0,1]$, consistent with the model assumption that $r_i \in [0,1]$.

The resulting parameters capture three features that synthetic instances cannot reproduce: a position-effects decay calibrated to observed user behavior, an attraction distribution that mirrors hotel-level click rates on the Expedia platform, and revenue magnitudes derived from actual posted prices. The Expedia-derived instances therefore offer a testbed for evaluating the algorithms under conditions that approximate practical deployment.

\subsection{Known Position Effects}

We construct test instances using the parameters extracted from the Expedia dataset. For each instance, we randomly select $N$ products from the set of hotels with $v^\star_i \ge 0.1$ to ensure sufficient statistical signal.
The position effects $\theta^\star_k$ are taken directly from the estimated values for the first $K$ positions. Following this procedure, we define three instances of increasing complexity: Example 7 with $N=30$ products and $K=8$ positions, Example 8 with $N=50$ products and $K=15$ positions, and Example 9 with $N=70$ products and $K=20$ positions. In all instances, the position effects $\theta^\star_k$ are treated as known, so the learning task focuses on the intrinsic attraction parameters $v^\star_i$.

\begin{figure}
  \centering
  \begin{subfigure}[b]{0.45\textwidth}
    \centering
    \includegraphics[width=\textwidth]{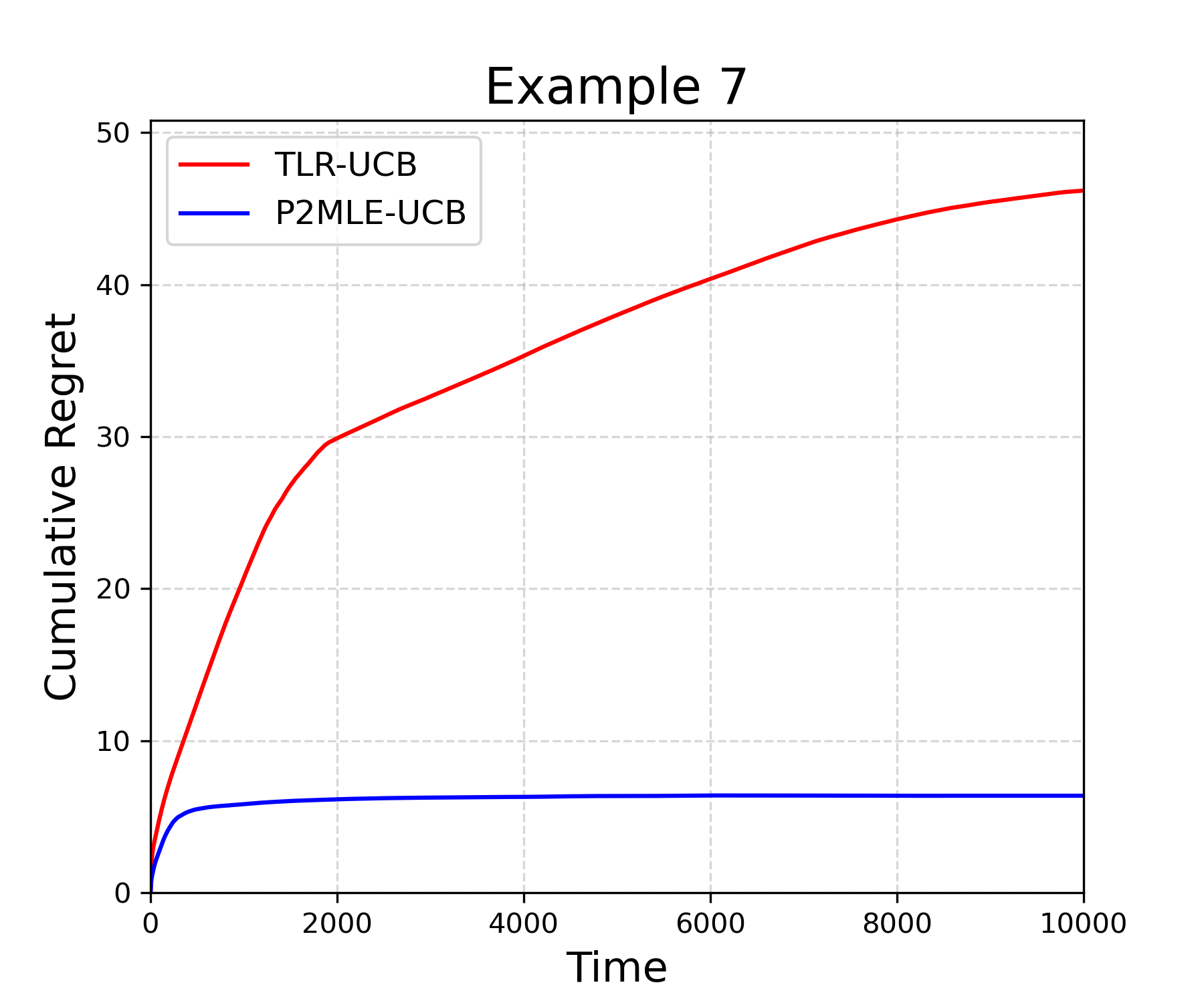}
  \end{subfigure}
  \hfill
  \begin{subfigure}[b]{0.45\textwidth}
    \centering
    \includegraphics[width=\textwidth]{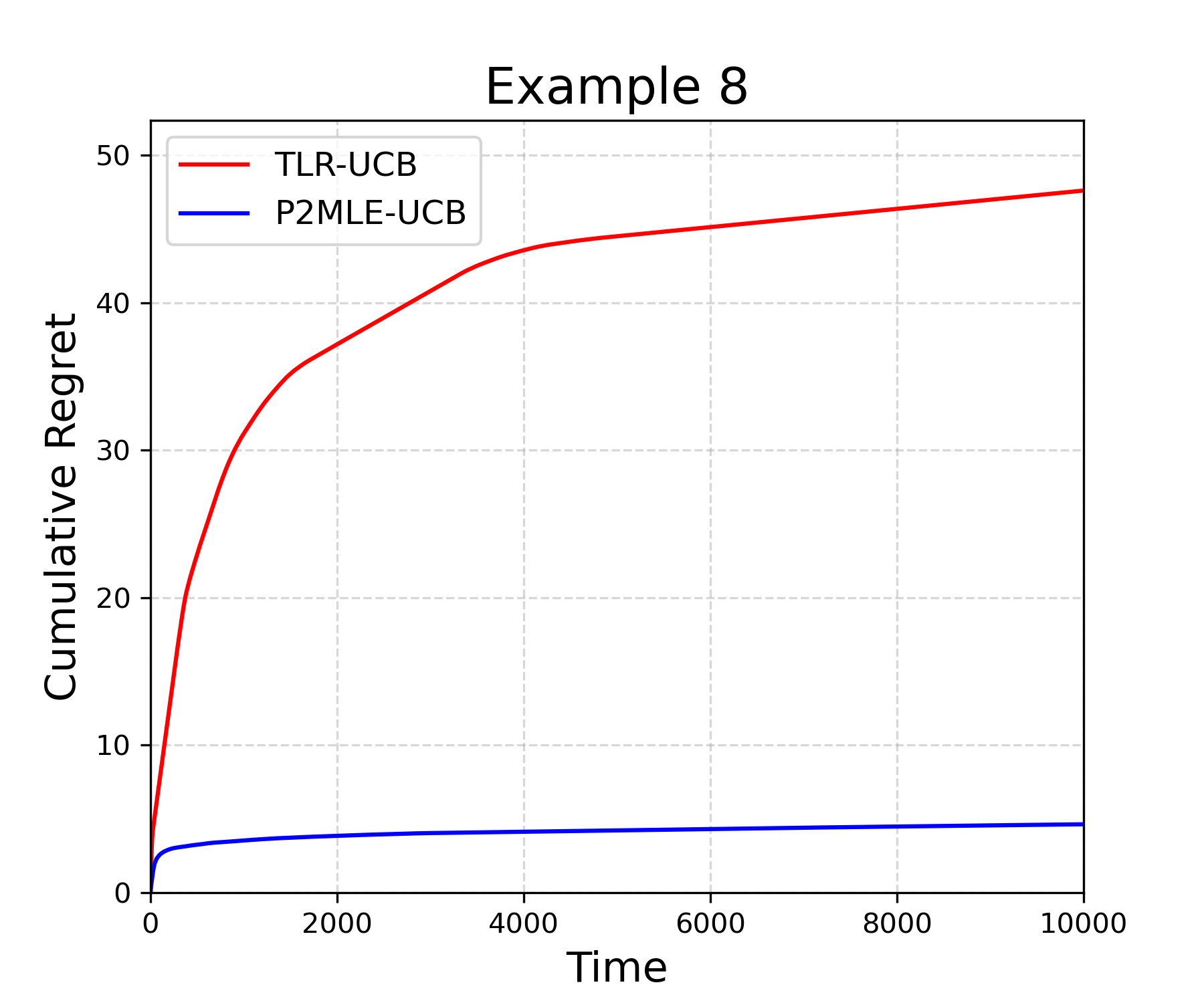}
  \end{subfigure}
  \hfill
  \\
  \begin{subfigure}[b]{0.45\textwidth}
    \centering
    \includegraphics[width=\textwidth]{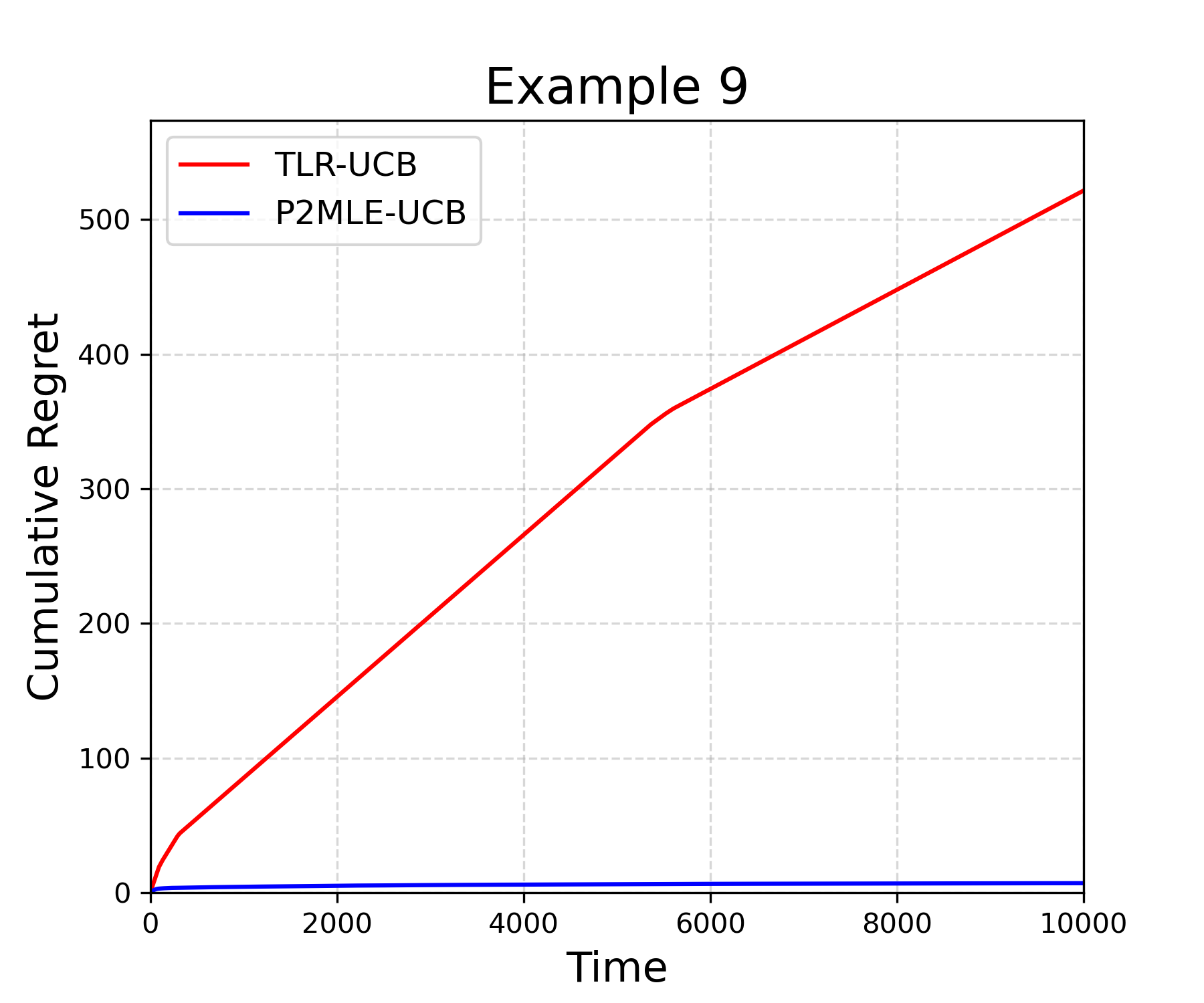}
  \end{subfigure}

  \caption{Regret Comparison Using the Derived Parameters}
  \label{fig:comparisonuserealdataparameters}
\end{figure}

Figure~\ref{fig:comparisonuserealdataparameters} shows the cumulative regret curves, averaged over 50 independent runs, for P2MLE-UCB and TLR-UCB across the three instances. In all instances, P2MLE-UCB achieves lower cumulative regret than TLR-UCB. The regret of P2MLE-UCB increases during the first approximately 1,000 rounds and then grows slowly afterward. TLR-UCB continues to accumulate regret steadily over the entire horizon.

The difference in regret between the two algorithms increases with the scale of the instance. In Example 9, which has 70 products and 20 positions, P2MLE-UCB maintains a nearly flat regret trajectory after the initial rounds, while the regret of TLR-UCB continues to grow throughout the horizon.

These results indicate that the performance advantage of P2MLE-UCB observed in synthetic experiments persists when instances are calibrated to real-world data.

\subsection{Unknown Position Effects}

We now evaluate performance when position effects are unknown. Using the same Expedia-derived parameters, we construct three additional instances: Example 10 with $N=20$ products and $K=5$ positions, Example 11 with $N=30$ products and $K=8$ positions, and Example 12 with $N=50$ products and $K=15$ positions. We compare E-P2MLE-UCB against EI-TLR, the explore-then-exploit benchmark from \cite{luo2025rate}.

\begin{figure}
  \centering

  \begin{subfigure}[b]{0.45\textwidth}
    \centering
    \includegraphics[width=\textwidth]{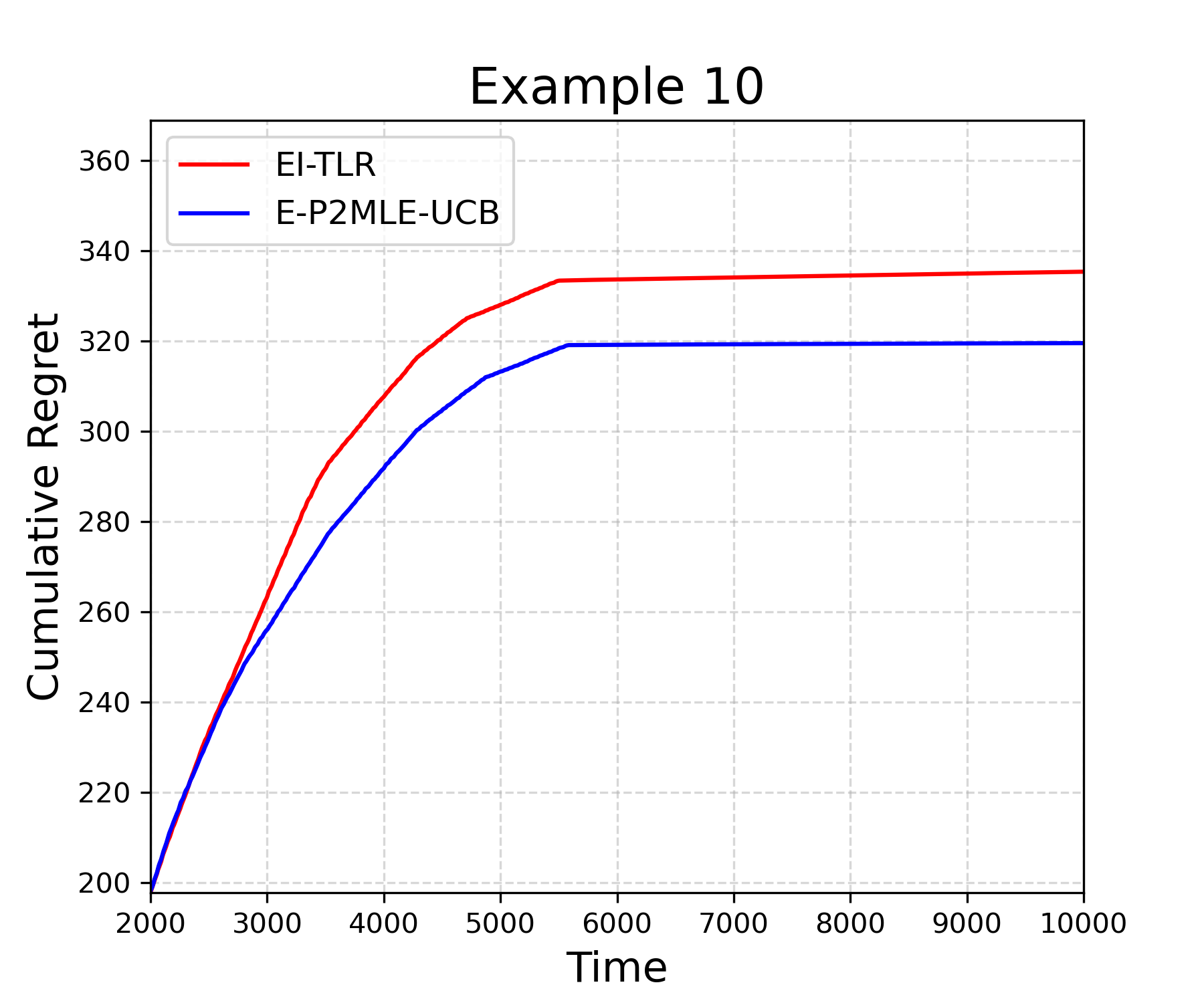}
  \end{subfigure}
  \hfill
  \begin{subfigure}[b]{0.45\textwidth}
    \centering
    \includegraphics[width=\textwidth]{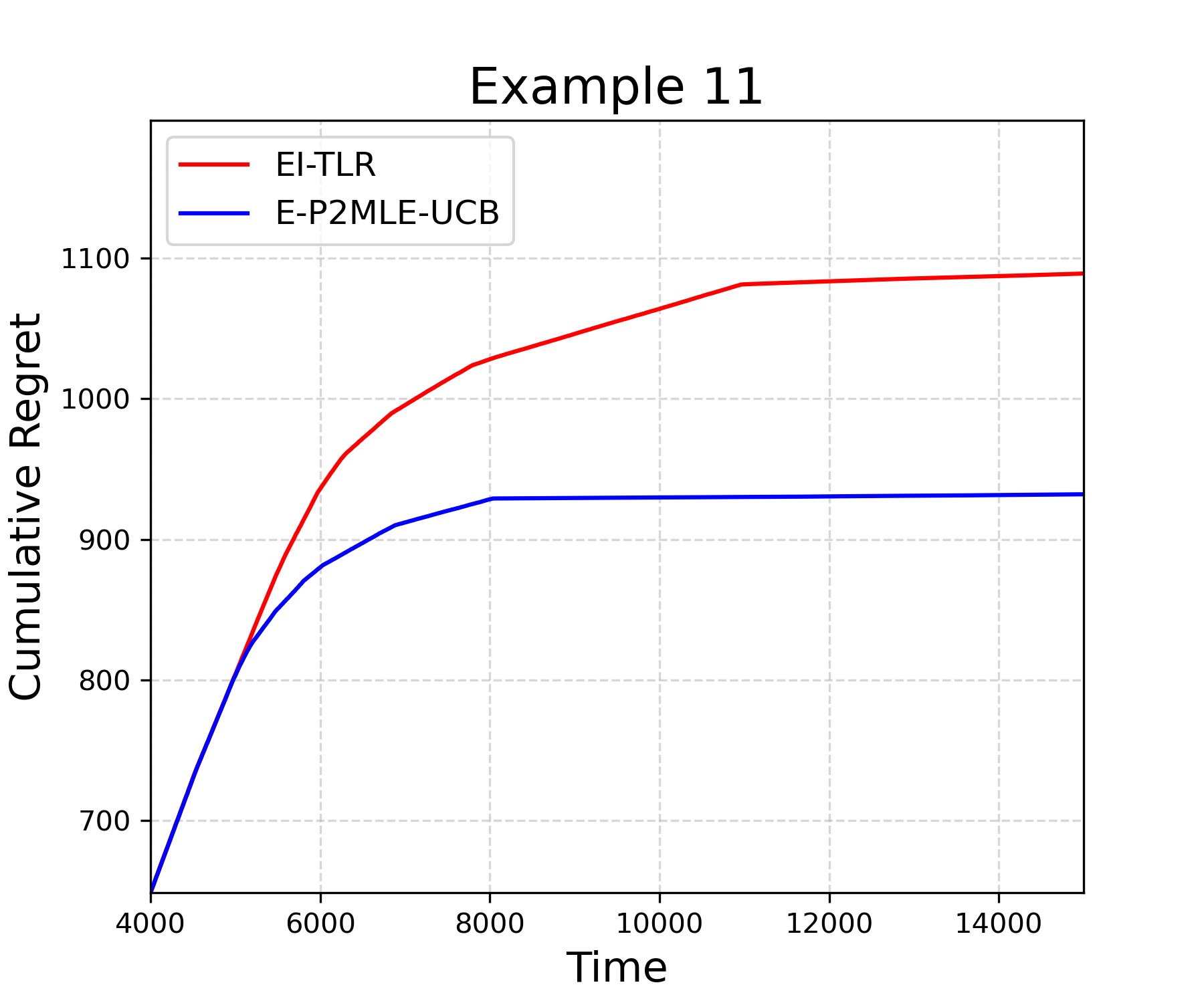}
  \end{subfigure}
  \hfill
  \\
  \begin{subfigure}[b]{0.45\textwidth}
    \centering
    \includegraphics[width=\textwidth]{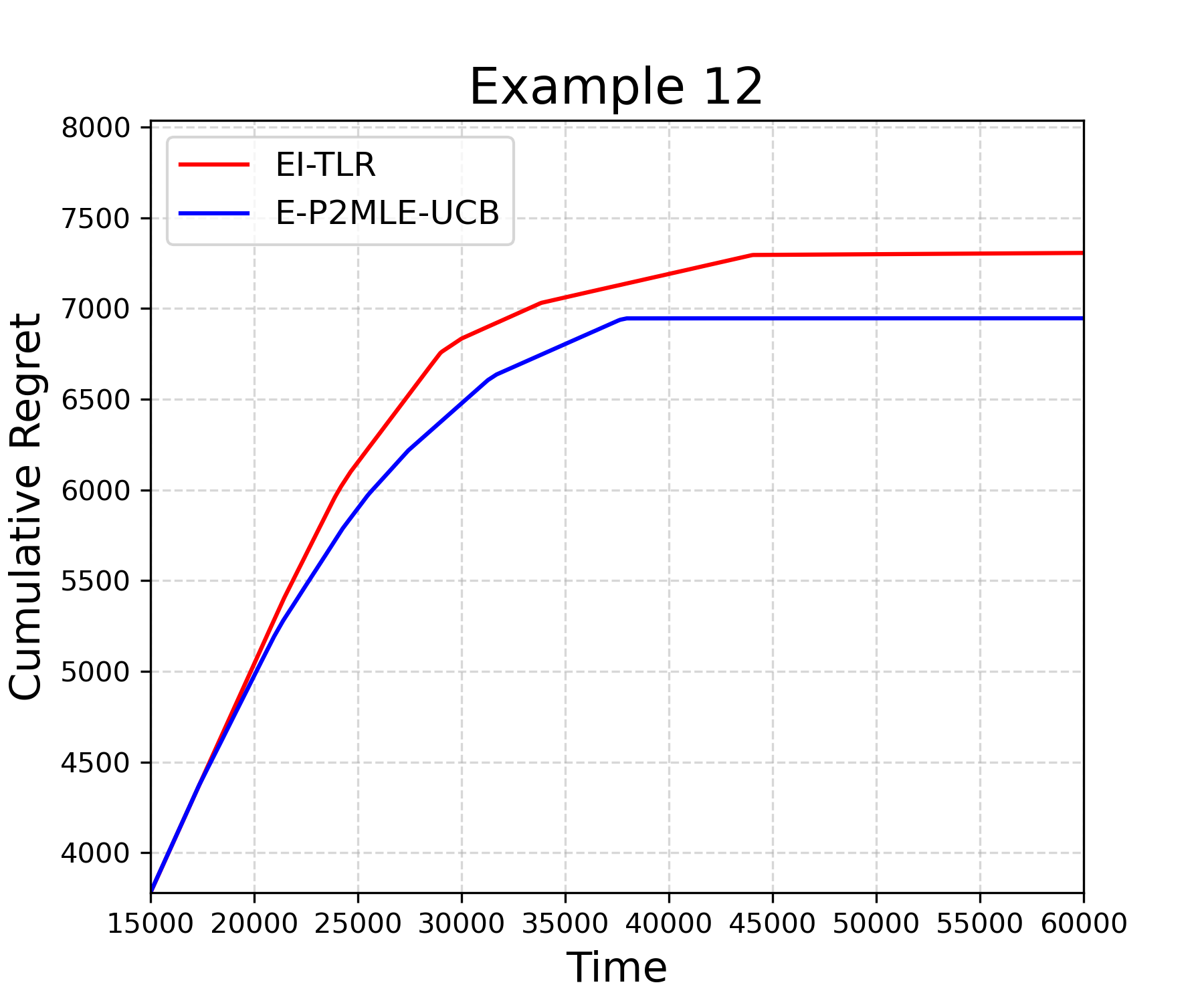}
  \end{subfigure}

  \caption{Regret Comparison on Expedia-Derived Instances with Unknown Position Effects}
  \label{fig:comparison_use_real_data_parameters_theta_unknown}
\end{figure}

Figure~\ref{fig:comparison_use_real_data_parameters_theta_unknown} shows the cumulative regret curves, each averaged over 50 independent replications. Both algorithms exhibit an initial linear increase in regret during the forced exploration phase, followed by a transition to sublinear growth. In all three examples, E-P2MLE-UCB achieves lower regret than EI-TLR after the transition.

As in Section~\ref{subsectioncomparisonthetaunknown}, these results indicate that the second-stage UCB updates of E-P2MLE-UCB are less sensitive to residual error in $\hat{\bm\theta}$ than the transition-layer mechanism of EI-TLR. The gap widens with problem scale, from Example 10 to Example 12, consistent with the synthetic-data findings.

\section{Conclusion}

We provided a systematic study of the dynamic joint assortment selection and positioning problem under the Multinomial Logit (MNL) model, from the classical multiplicative position effects model to a newly introduced fully general position effects model that captures heterogeneous product--position synergies. For each of the two position effects models, we developed a round-based learning algorithm and established matching regret upper and lower bounds, thereby delivering the first rate-optimal characterization for both position effects models: for the multiplicative model, our P2MLE-UCB algorithm achieves a regret of $\tilde{O}(\sqrt{NT})$, matching the lower bound and eliminating the $\sqrt{K}$ gap in prior epoch-based work; for the general model, we established a minimax lower bound of $\Omega(\sqrt{KNT})$ and proposed GP2-UCB with a matching upper bound, supported by an efficient subroutine based on Dinkelbach's method and maximum-weight bipartite matching. Extensive experiments on synthetic data and the Expedia Hotel Search dataset further confirm that our algorithms consistently outperform state-of-the-art epoch-based benchmarks.

Two directions are worth highlighting for future research. First, as discussed in Section~\ref{sec:multiplicative_analysis}, when the position effects $\bm{\theta}^\star$ in the multiplicative model are unknown and must be learned jointly with the intrinsic attractions $\{v_i^\star\}$, the fallback general-model algorithm GP2-UCB delivers an $\tilde{O}(\sqrt{KNT})$ guarantee; this is, to our knowledge, the first formal regret guarantee for the multiplicative model with unknown $\bm{\theta}^\star$. Whether the sharper $\tilde{O}(\sqrt{NT})$ rate can be attained by fully exploiting the multiplicative (rank-one) structure $v_i^\star\theta_k^\star$ when $\bm{\theta}^\star$ is unknown remains a technically challenging open problem. The core difficulty is that the attraction matrix $V^\star=(v^\star_i\theta^\star_k)_{i\in[N],\,k\in[K]}$ is a rank-one factor model, and recovering it from MNL feedback is a joint estimation problem with two coupled obstacles. 
(i) \emph{Coupled factor estimation from pairwise feedback:} after conditioning on $\{i_t\in\{i,0\}\}$, the pairwise observation at position $k$ identifies only the composite attraction $v_i^\star \theta_k^\star$. Hence the cross-position pairwise MLE developed in Section~\ref{sec:multiplicative_analysis}, which isolates $v_i^\star$ when $\bm{\theta}^\star$ is known, no longer decouples the product parameters from the position effects. Learning now requires disentangling the rank-one factors $\bm{v}^\star$ and $\bm{\theta}^\star$ under adaptive product--position assignments.
(ii) \emph{Adaptive design and exploration coverage:} achieving the $\sqrt{NT}$ rate requires a design that gathers enough information on each $\theta^\star_k$ without spending $\Omega(\sqrt{KNT})$ regret on forced exploration; current explore-then-commit constructions cannot avoid this tradeoff. Closing this gap likely requires either a new estimator that respects the bilinear structure (e.g., a rank-one constrained MLE or a spectral-method warm start) or fundamentally different exploration design. Second, extending the framework to the contextual MNL setting with position effects is an appealing direction, as it would enable personalized rankings that jointly account for customer-specific features and positional prominence.

\ACKNOWLEDGMENT{The authors are grateful to Shunxing Yan for valuable discussions and suggestions.}

\bibliographystyle{informs2014}
\bibliography{refs.bib}

\begin{thebibliography}{38}
\providecommand{\natexlab}[1]{#1}
\providecommand{\url}[1]{\texttt{#1}}
\providecommand{\urlprefix}{URL }

\bibitem[{Abeliuk et~al.(2016)Abeliuk, Berbeglia, Cebrian, \protect\BIBand{}
  Van~Hentenryck}]{abeliuk2016assortment}
Abeliuk A, Berbeglia G, Cebrian M, Van~Hentenryck P (2016) Assortment
  optimization under a multinomial logit model with position bias and social
  influence. \emph{4OR} 14(1):57--75.

\bibitem[{Agrawal et~al.(2017)Agrawal, Avadhanula, Goyal, \protect\BIBand{}
  Zeevi}]{agrawal2017thompson}
Agrawal S, Avadhanula V, Goyal V, Zeevi A (2017) Thompson sampling for the
  mnl-bandit. \emph{Conference on learning theory}, 76--78 (PMLR).

\bibitem[{Agrawal et~al.(2019)Agrawal, Avadhanula, Goyal, \protect\BIBand{}
  Zeevi}]{agrawal2019mnl}
Agrawal S, Avadhanula V, Goyal V, Zeevi A (2019) Mnl-bandit: A dynamic learning
  approach to assortment selection. \emph{Operations Research}
  67(5):1453--1485.

\bibitem[{Aouad \protect\BIBand{} Segev(2021)}]{aouad2021display}
Aouad A, Segev D (2021) Display optimization for vertically differentiated
  locations under multinomial logit preferences. \emph{Management Science}
  67(6):3519--3550.

\bibitem[{Bourgeois \protect\BIBand{} Lassalle(1971)}]{bourgeois1971extension}
Bourgeois F, Lassalle JC (1971) An extension of the munkres algorithm for the
  assignment problem to rectangular matrices. \emph{Communications of the ACM}
  14(12):802--804.

\bibitem[{Caro \protect\BIBand{} Gallien(2007)}]{caro2007dynamic}
Caro F, Gallien J (2007) Dynamic assortment with demand learning for seasonal
  consumer goods. \emph{Management science} 53(2):276--292.

\bibitem[{Cesa-Bianchi \protect\BIBand{} Lugosi(2006)}]{cesa2006prediction}
Cesa-Bianchi N, Lugosi G (2006) \emph{Prediction, learning, and games}
  (Cambridge university press).

\bibitem[{Chen et~al.(2023{\natexlab{a}})Chen, Dong, Wang, Feng, Wang,
  \protect\BIBand{} He}]{chen2023bias}
Chen J, Dong H, Wang X, Feng F, Wang M, He X (2023{\natexlab{a}}) Bias and
  debias in recommender system: A survey and future directions. \emph{ACM
  Transactions on Information Systems} 41(3):1--39.

\bibitem[{Chen et~al.(2023{\natexlab{b}})Chen, Gao, Wang, \protect\BIBand{}
  Wang}]{chen2023assortment}
Chen N, Gao P, Wang C, Wang Y (2023{\natexlab{b}}) Assortment optimization for
  the multinomial logit model with repeated customer interactions.
  \emph{Available at SSRN 4526247} .

\bibitem[{Chen et~al.(2021{\natexlab{a}})Chen, Shi, Wang, \protect\BIBand{}
  Zhou}]{chen2021dynamic}
Chen X, Shi C, Wang Y, Zhou Y (2021{\natexlab{a}}) Dynamic assortment planning
  under nested logit models. \emph{Production and Operations Management}
  30(1):85--102.

\bibitem[{Chen \protect\BIBand{} Wang(2018)}]{chen2018note}
Chen X, Wang Y (2018) A note on a tight lower bound for capacitated mnl-bandit
  assortment selection models. \emph{Operations Research Letters}
  46(5):534--537.

\bibitem[{Chen et~al.(2021{\natexlab{b}})Chen, Wang, \protect\BIBand{}
  Zhou}]{chen2021optimal}
Chen X, Wang Y, Zhou Y (2021{\natexlab{b}}) Optimal policy for dynamic
  assortment planning under multinomial logit models. \emph{Mathematics of
  Operations Research} 46(4):1639--1657.

\bibitem[{Cheung et~al.(2019)Cheung, Tan, \protect\BIBand{}
  Zhong}]{cheung2019thompson}
Cheung WC, Tan V, Zhong Z (2019) A thompson sampling algorithm for cascading
  bandits. \emph{The 22nd International Conference on Artificial Intelligence
  and Statistics}, 438--447 (PMLR).

\bibitem[{Craswell et~al.(2008)Craswell, Zoeter, Taylor, \protect\BIBand{}
  Ramsey}]{craswell2008experimental}
Craswell N, Zoeter O, Taylor M, Ramsey B (2008) An experimental comparison of
  click position-bias models. \emph{Proceedings of the 2008 international
  conference on web search and data mining}, 87--94.

\bibitem[{Dinkelbach(1967)}]{dinkelbach1967nonlinear}
Dinkelbach W (1967) On nonlinear fractional programming. \emph{Management
  science} 13(7):492--498.

\bibitem[{Dong et~al.(2025)Dong, Mo, Qi, Shi, Fang, \protect\BIBand{}
  Tarokh}]{dong2025pasta}
Dong J, Mo W, Qi Z, Shi C, Fang EX, Tarokh V (2025) Pasta: A unified framework
  for offline assortment learning. \emph{arXiv preprint arXiv:2510.01693} .

\bibitem[{Freedman(1975)}]{freedman1975tail}
Freedman DA (1975) On tail probabilities for martingales. \emph{the Annals of
  Probability} 100--118.

\bibitem[{Gallego et~al.(2020)Gallego, Li, Truong, \protect\BIBand{}
  Wang}]{gallego2020approximation}
Gallego G, Li A, Truong VA, Wang X (2020) Approximation algorithms for product
  framing and pricing. \emph{Operations Research} 68(1):134--160.

\bibitem[{Gao et~al.(2021)Gao, Ma, Chen, Gallego, Li, Rusmevichientong,
  \protect\BIBand{} Topaloglu}]{gao2021assortment}
Gao P, Ma Y, Chen N, Gallego G, Li A, Rusmevichientong P, Topaloglu H (2021)
  Assortment optimization and pricing under the multinomial logit model with
  impatient customers: Sequential recommendation and selection.
  \emph{Operations research} 69(5):1509--1532.

\bibitem[{Han et~al.(2025{\natexlab{a}})Han, Blanchet, \protect\BIBand{}
  Zhou}]{hanimproved}
Han Y, Blanchet J, Zhou Z (2025{\natexlab{a}}) Improved confidence regions and
  optimal algorithms for online and offline linear mnl bandits. \emph{The
  Thirty-ninth Annual Conference on Neural Information Processing Systems}.

\bibitem[{Han et~al.(2025{\natexlab{b}})Han, Zhong, Lu, Blanchet,
  \protect\BIBand{} Zhou}]{han2025learning}
Han Y, Zhong H, Lu M, Blanchet J, Zhou Z (2025{\natexlab{b}}) Learning an
  optimal assortment policy under observational data. \emph{arXiv preprint
  arXiv:2502.06777} .

\bibitem[{Jasin et~al.(2024)Jasin, Lyu, Najafi, \protect\BIBand{}
  Zhang}]{jasin2024assortment}
Jasin S, Lyu C, Najafi S, Zhang H (2024) Assortment optimization with
  multi-item basket purchase under multivariate mnl model. \emph{Manufacturing
  \& Service Operations Management} 26(1):215--232.

\bibitem[{Jiang et~al.(2026)Jiang, Wang, Xue, \protect\BIBand{}
  Zhang}]{jiang2026assortment}
Jiang B, Wang Z, Xue C, Zhang N (2026) Assortment optimization in the presence
  of focal effect: Operational insights and efficient algorithms.
  \emph{Production and Operations Management} 35(3):1133--1150.

\bibitem[{Joachims et~al.(2017)Joachims, Swaminathan, \protect\BIBand{}
  Schnabel}]{joachims2017unbiased}
Joachims T, Swaminathan A, Schnabel T (2017) Unbiased learning-to-rank with
  biased feedback. \emph{Proceedings of the tenth ACM international conference
  on web search and data mining}, 781--789.

\bibitem[{Jonker \protect\BIBand{} Volgenant(1987)}]{jonker1987shortest}
Jonker R, Volgenant A (1987) A shortest augmenting path algorithm for dense and
  sparse linear assignment problems. \emph{Computing} 38(4):325--340.

\bibitem[{Kveton et~al.(2015)Kveton, Szepesvari, Wen, \protect\BIBand{}
  Ashkan}]{kveton2015cascading}
Kveton B, Szepesvari C, Wen Z, Ashkan A (2015) Cascading bandits: Learning to
  rank in the cascade model. \emph{International conference on machine
  learning}, 767--776 (PMLR).

\bibitem[{Li et~al.(2025)Li, Luo, Huang, \protect\BIBand{} Shi}]{li2025online}
Li S, Luo Q, Huang Z, Shi C (2025) Online learning for constrained assortment
  optimization under markov chain choice model. \emph{Operations research}
  73(1):109--138.

\bibitem[{Li et~al.(2016)Li, Wang, Zhang, \protect\BIBand{}
  Chen}]{li2016contextual}
Li S, Wang B, Zhang S, Chen W (2016) Contextual combinatorial cascading
  bandits. \emph{International conference on machine learning}, 1245--1253
  (PMLR).

\bibitem[{Liang et~al.(2026)Liang, Mao, \protect\BIBand{}
  Wang}]{liang2026online}
Liang Y, Mao X, Wang S (2026) Online joint assortment-inventory optimization
  under mnl choices. \emph{Operations Research} .

\bibitem[{Luo et~al.(2025)Luo, Sun, \protect\BIBand{} Liu}]{luo2025rate}
Luo Y, Sun WW, Liu Y (2025) Rate-optimal online learning for dynamic assortment
  selection with positioning. \emph{Operations Research} .

\bibitem[{Miao \protect\BIBand{} Chao(2021)}]{miao2021dynamic}
Miao S, Chao X (2021) Dynamic joint assortment and pricing optimization with
  demand learning. \emph{Manufacturing \& Service Operations Management}
  23(2):525--545.

\bibitem[{Najafi et~al.(2026)Najafi, Jasin, Uichanco, \protect\BIBand{}
  Zhao}]{najafi2026assortment}
Najafi S, Jasin S, Uichanco J, Zhao J (2026) Assortment and price optimization
  under a multiattribute (contextual) choice model. \emph{Operations Research}
  .

\bibitem[{Rusmevichientong et~al.(2010)Rusmevichientong, Shen,
  \protect\BIBand{} Shmoys}]{rusmevichientong2010dynamic}
Rusmevichientong P, Shen ZJM, Shmoys DB (2010) Dynamic assortment optimization
  with a multinomial logit choice model and capacity constraint.
  \emph{Operations research} 58(6):1666--1680.

\bibitem[{Saha \protect\BIBand{} Gaillard(2024)}]{saha2024stop}
Saha A, Gaillard P (2024) Stop relying on no-choice and do not repeat the
  moves: Optimal, efficient and practical algorithms for assortment
  optimization. \emph{arXiv preprint arXiv:2402.18917} .

\bibitem[{Saur{\'e} \protect\BIBand{} Zeevi(2013)}]{saure2013optimal}
Saur{\'e} D, Zeevi A (2013) Optimal dynamic assortment planning with demand
  learning. \emph{Manufacturing \& Service Operations Management}
  15(3):387--404.

\bibitem[{Schaible(1976)}]{schaible1976fractional}
Schaible S (1976) Fractional programming. ii, on dinkelbach's algorithm.
  \emph{Management science} 22(8):868--873.

\bibitem[{Wang et~al.(2018)Wang, Golbandi, Bendersky, Metzler,
  \protect\BIBand{} Najork}]{wang2018position}
Wang X, Golbandi N, Bendersky M, Metzler D, Najork M (2018) Position bias
  estimation for unbiased learning to rank in personal search.
  \emph{Proceedings of the eleventh ACM international conference on web search
  and data mining}, 610--618.

\bibitem[{Xu \protect\BIBand{} Wang(2023)}]{xu2023assortment}
Xu Y, Wang Z (2023) Assortment optimization for a multistage choice model.
  \emph{Manufacturing \& Service Operations Management} 25(5):1748--1764.

\end{thebibliography}
\ECSwitch

\makeatletter
\renewcommand{\theHequation}{EC.\theequation}
\renewcommand{\theHlemma}{EC.\thelemma}
\renewcommand{\theHtheorem}{EC.\thetheorem}
\renewcommand{\theHsection}{EC.\thesection}
\makeatother

\begin{center}
  \Large{\bf Supplementary Materials to {``Learning in Position-Aware Multinomial Logit Bandits: From Multiplicative to General Position Effects''} }
\end{center}

\vspace{10pt}

\section{Technical Lemmas}\label{sec:technical_lemmas}

For self-containedness, we collect here the two martingale concentration inequalities that are used throughout the appendix. Both are stated in their standard one-sided form; two-sided versions follow by symmetry and a union bound.

\begin{lemma}[Freedman's inequality, \citealt{freedman1975tail}] \label{lem:freedman}
Let $\{Y_k : k=0,1,2,\ldots\}$ be a real-valued martingale adapted to a filtration $\{\mathcal{F}_k\}_{k\ge 0}$, with difference sequence $X_k := Y_k - Y_{k-1}$ ($k\ge 1$). Assume that the difference sequence is uniformly bounded:
\[
X_k \le R \quad \text{almost surely, for all } k = 1, 2, 3, \ldots.
\]
Define the predictable quadratic variation process
\[
W_k := \sum_{j=1}^{k} \mathbb{E}\!\left[X_j^2 \,\big|\, \mathcal{F}_{j-1}\right], \qquad k = 1, 2, 3, \ldots.
\]
Then, for all $t \ge 0$ and $\sigma^2 > 0$,
\[
\mathbb{P}\Bigl\{\exists\, k \ge 0 : Y_k - Y_0 \ge t \text{ and } W_k \le \sigma^2\Bigr\} \;\le\; \exp\!\left(-\frac{t^2/2}{\sigma^2 + Rt/3}\right).
\]
\end{lemma}

When the difference sequence $\{X_k\}$ consists of independent random variables, the predictable quadratic variation $W_k$ becomes deterministic and Freedman's inequality reduces to the usual Bernstein inequality.

\begin{lemma}[Azuma--Hoeffding inequality, Lemma A.7 of \citealt{cesa2006prediction}] \label{lem:azuma}
Let $\{X_k : k=0,1,2,\ldots\}$ be a real-valued martingale adapted to a filtration $\{\mathcal{F}_k\}_{k \ge 0}$. Assume that there exist predictable random variables $\{Z_k\}_{k \ge 1}$ (i.e., $Z_k$ is $\mathcal{F}_{k-1}$-measurable) and deterministic constants $\{c_k\}_{k \ge 1}$ such that the increments almost surely satisfy:
\[
Z_k \le X_k - X_{k-1} \le Z_k + c_k \quad \text{for all } k \ge 1.
\]
Then, for any positive integer $N$ and any $\varepsilon > 0$,
\[
\mathbb{P}(X_N - X_0 \ge \varepsilon) \le \exp\left( - \frac{2\varepsilon^2}{\sum_{k=1}^N c_k^2} \right).
\]
By symmetry, the same bound holds for the lower tail $\mathbb{P}(X_N - X_0 \le -\varepsilon)$.
\end{lemma}

When $\{X_k\}$ is a sum of bounded independent random variables, Lemma~\ref{lem:azuma} recovers the classical Hoeffding inequality.

\section{Proofs Omitted in Section \ref{sec:multiplicative_analysis}}

\subsection{Probabilistic setup for adaptive sampling} \label{sec:probabilistic_setup}

Throughout the proofs in this appendix, we work on a single probability space carrying the entire interaction history. Define the filtration
\[
\mathcal{F}_{t-1} := \sigma\bigl(S_1,\sigma_1,i_1,\ldots, S_{t-1},\sigma_{t-1},i_{t-1},\, S_t,\sigma_t\bigr), \qquad t\ge 1,
\]
so that $\mathcal{F}_{t-1}$ records all assortments and choices observed strictly before round $t$ together with the assortment and positioning offered at round $t$. Conditional on $(S_t,\sigma_t)$, the customer's choice $i_t$ follows the MNL distribution induced by the multiplicative (or general) attraction model.

For each $(i,k)\in[N]\times[K]$, let
\[
A_s^{(i,k)} := \mathbf{1}\{i\in S_s,\ \sigma_s(i)=k,\ i_s\in\{i,0\}\}
\]
denote the indicator of an active pairwise observation for $(i,k)$ in round $s$. We define the centered increments
\[
X_s^{(i,k)} := A_s^{(i,k)} \left(\mathbf{1}\{i_s=i\} - p_{i,k}\right), \qquad \text{where} \quad p_{i,k} := \frac{v_i^\star \theta_k^\star}{1+v_i^\star\theta_k^\star}.
\]

To apply concentration inequalities efficiently using the fully empirical, data-dependent counts
\[
n_{i,k}^t = \sum_{s=1}^{t-1} A_s^{(i,k)}, \qquad w_{i,k}^t - n_{i,k}^t\, p_{i,k} = \sum_{s=1}^{t-1} X_s^{(i,k)},
\]
we analyze concentration for a fixed product $i$ using a product-specific refined filtration that augments $\mathcal{F}_{s-1}$ with the indicator of an active pairwise comparison at round $s$:
\[
\mathcal{G}_{s-1}^{(i)} := \sigma\bigl(\mathcal{F}_{s-1},\, \mathbf{1}\{i_s \in \{i,0\}\}\bigr), \qquad s\ge 1.
\]
This filtration $\mathcal{G}_{s-1}^{(i)}$ reveals whether product $i$ was involved in an active pairwise comparison at round $s$ (making $A_s^{(i,k)}$ fully $\mathcal{G}_{s-1}^{(i)}$-measurable) but hides the final purchase decision $i_s$. Conditional on $\mathcal{G}_{s-1}^{(i)}$, if $A_s^{(i,k)}=1$, the probability of the customer purchasing $i$ rather than $0$ remains exactly $p_{i,k}$.

The sequence $\{X_s^{(i,k)}\}_{s\ge 1}$ is then a bounded martingale difference sequence with respect to $\{\mathcal{G}_{s-1}^{(i)}\}_{s\ge 1}$ because
\begin{align*}
\mathbb{E}\!\left[X_s^{(i,k)} \,\big|\, \mathcal{G}_{s-1}^{(i)}\right]
&= \mathbb{E}\!\left[A_s^{(i,k)} \left(\mathbf{1}\{i_s=i\} - p_{i,k}\right) \,\big|\, \mathcal{G}_{s-1}^{(i)}\right] \\
&= A_s^{(i,k)} \left(\mathbb{E}\!\left[\mathbf{1}\{i_s=i\} \,\big|\, \mathcal{G}_{s-1}^{(i)}\right] - p_{i,k}\right) \\
&= A_s^{(i,k)} \left(p_{i,k} - p_{i,k}\right) = 0.
\end{align*}
Here, we used the fact that $A_s^{(i,k)}$ is $\mathcal{G}_{s-1}^{(i)}$-measurable so it factors out of the expectation, and conditionally on $\mathcal{G}_{s-1}^{(i)}$ the event $A_s^{(i,k)}=1$ implies $i_s \in \{i,0\}$, under which the probability of purchasing $i$ (rather than $0$) is $\frac{v_i^\star \theta_k^\star}{1+v_i^\star \theta_k^\star} = p_{i,k}$.

Crucially, its predictable quadratic variation under this finer sub-filtration exactly matches the realized variance path:
\begin{align*}
\mathbb{E}\!\left[(X_s^{(i,k)})^2 \,\big|\, \mathcal{G}_{s-1}^{(i)}\right]
&\;=\; \mathbb{E}\!\left[(A_s^{(i,k)})^2 \left(\mathbf{1}\{i_s=i\} - p_{i,k}\right)^2 \,\big|\, \mathcal{G}_{s-1}^{(i)}\right] \\
&\;=\; A_s^{(i,k)} \cdot \mathbb{E}\!\left[\left(\mathbf{1}\{i_s=i\} - p_{i,k}\right)^2 \,\big|\, \mathcal{G}_{s-1}^{(i)}\right] \\
&\;=\; A_s^{(i,k)} \cdot p_{i,k}(1-p_{i,k}).
\end{align*}
where the second equality relies on $A_s^{(i,k)}$ being $\mathcal{G}_{s-1}^{(i)}$-measurable and $(A_s^{(i,k)})^2 = A_s^{(i,k)}$, and the final equality holds because, conditionally on $\mathcal{G}_{s-1}^{(i)}$ and $A_s^{(i,k)}=1$, the indicator $\mathbf{1}\{i_s=i\}$ is a Bernoulli random variable with success probability $p_{i,k}$, which consequently has variance $p_{i,k}(1-p_{i,k})$.

\subsection{Proof of Lemma~\ref{lem:contraction_of_clipped_MLE}}

We abbreviate the indices $i,t$ for simplicity in this proof. Now we analyze the deviation $\left|S(v^\star)-S(\hat{v})\right|$.

We proceed by cases on whether the unconstrained root $\tilde v$ lies in $[0,1]$. Since $S(0)=\sum_k w_k\ge 0$ and $S$ is strictly decreasing, $\tilde v\ge 0$, so only the clipping from above can be active.

\emph{Case 1: $\tilde{v}\in[0,1]$, i.e., $\hat{v}=\tilde{v}$.} Then $S(\hat{v})=0$ and trivially $|S(v^\star)-S(\hat{v})|=|S(v^\star)|$.

\emph{Case 2: $\tilde{v}>1$, i.e., $\hat{v}=1$.} By the strict monotonicity of $S$ and the ordering $v^\star\le 1<\tilde{v}$, we have $S(v^\star)\ge S(1)>S(\tilde{v})=0$. In particular, both $S(v^\star)$ and $S(1)$ are nonnegative, so
\[
|S(v^\star)-S(\hat{v})|=S(v^\star)-S(1) < S(v^\star)=|S(v^\star)|.
\]
Combining the two cases yields the contraction property
\[
|S(v^\star)-S(\hat{v})|\le|S(v^\star)|.
\]

By the Mean Value Theorem, there exists an intermediate value $v'$ between $v^\star$ and $\hat{v}$ such that:
\[
\left|S(v^\star)\right| \geq \left|S(v^\star)-S(\hat{v})\right| = \left|S^{\prime}(v^{\prime})\right| \cdot \left|v^\star-\hat{v}\right|.
\]
Evaluating the derivative $S'(v)$, we obtain:
\[
\left|S'(v')\right|= \sum_k\frac{n_k\theta^\star_k}{\left(1+v^{\prime}\theta^\star_k\right)^2}.
\]
Here, the power of the clipped MLE becomes apparent: because $v^\star\le 1$ and $\hat v\le 1$, the intermediate value $v'$ satisfies $v'\le 1$. Since $\theta^\star_k \in (0,1]$ as well, the denominator $(1+v' \theta^\star_k)^2$ is at most $4$. This yields the lower bound:
\begin{equation}
\left|S(v^\star)\right| \ge \left(\frac14\sum_k n_k \theta^\star_k\right) |\hat{v}-v^\star|=\frac{1}{4} D \cdot |\hat{v}-v^\star|.
\end{equation}

\subsection{Proof of Lemma~\ref{concentrationwithknowntheta}}

We abbreviate the indices $i$ for simplicity in this proof, all quantities $S^t, n^t_k, w^t_k, p_k, D^t$ are understood to be indexed by the fixed product $i$. Recall that the score function at $v^\star$ up to round $t$ is
\[
  S^t(v^\star)=\sum_{k=1}^K \left(w^t_k-n^t_k p_k(v^\star)\right),
\]
Using the refined filtration introduced in Appendix~\ref{sec:probabilistic_setup}, we evaluate the score function as a sum of the centered increments $\{X_s^{(i,k)}\}_{s\ge 1}$ adapted to $\{\mathcal{G}_{s-1}^{(i)}\}_{s\ge 1}$. Abbreviating $A_s^{(k)} = A_s^{(i,k)}$ and $X_s^{(k)}=X_s^{(i,k)}$, we have $n_k=\sum_{s<t}A_s^{(k)}$ and
\[
  S^t(v^\star) \;=\; \sum_{k=1}^K\sum_{s=1}^{t-1} X_s^{(k)} \;=\; \sum_{s=1}^{t-1} X_s, \qquad X_s := \sum_{k=1}^K X_s^{(k)}.
\]
Because product $i$ can be displayed at most once in any assortment, the $A_s^{(k)}$ indicator functions for different positions $k$ are mutually exclusive at any round $s$. As a result, the sum $X_s = \sum_{k=1}^K X_s^{(k)}$ has at most one non-zero term, and the increments satisfy $|X_s|\le 1$. By construction, $\mathbb{E}[X_s\mid\mathcal{G}_{s-1}^{(i)}]=0$, so $\{X_s\}_{s\ge 1}$ is a true bounded martingale-difference sequence with respect to the refined filtration $\{\mathcal{G}_{s-1}^{(i)}\}$.

We next compute the predictable quadratic variation of $S^t(v^\star)$ with respect to $\{\mathcal{G}_{s-1}^{(i)}\}$. Due to the mutual exclusivity of $A_s^{(k)}$ across $k$, the cross-terms vanish and we have $\mathbb{E}[X_s^2\mid\mathcal{G}_{s-1}^{(i)}] = \sum_{k=1}^K \mathbb{E}[(X_s^{(k)})^2\mid\mathcal{G}_{s-1}^{(i)}] = \sum_{k=1}^K A_s^{(k)} p_k(v^\star)(1-p_k(v^\star))$,
\begin{align*}
W_t &\;:=\; \sum_{s=1}^{t-1}\mathbb{E}[X_s^2\mid\mathcal{G}_{s-1}^{(i)}] \;=\; \sum_{k=1}^K\sum_{s=1}^{t-1} A_s^{(k)}\, p_k(v^\star)(1-p_k(v^\star)) \\
& \;=\; \sum_{k=1}^K n^t_k\, p_k(v^\star)(1-p_k(v^\star)) \;=\; \sum_{k=1}^K \frac{n^t_k v^\star \theta_k^\star}{(1+v^\star\theta_k^\star)^2}.
\end{align*}
Using the definition $D^t = \sum_k n^t_k \theta_k^\star$ and the fact that $(1+v^\star\theta^\star_k)^2 \ge 1$, we obtain the upper bound
\[
W_t \le D^t v^\star.
\]
Since $S^t(v^\star)=\sum_{s=1}^{t-1} X_s$ is a martingale with bounded increments $|X_s|\le 1$ and predictable quadratic variation $W_t$, we apply Freedman's inequality (Lemma~\ref{lem:freedman}, with $R=1$) together with a union bound for the two-sided event: for every $\varepsilon>0$ and every deterministic $\sigma^2>0$,
\[
  \mathbb{P}\bigl[\exists t \ge 1 : |S^t(v^\star)|\ge \varepsilon \text{ and } W_t \le \sigma^2\bigr] \;\le\; 2\exp\!\left(-\frac{\varepsilon^2}{2(\sigma^2+\varepsilon/3)}\right).
\]

Although $D^t v^\star$ bounds $W_t$ surely, $D^t$ is a random variable bounded by $T$, so we cannot directly plug $D v^\star$ in for $\sigma^2$. Instead, we use a standard geometric peeling argument over the interval $[\theta^\star_{\min}, T]$ for the possible values of $D^t > 0$. Divide the interval $(0, T]$ into grids $(s_{m+1}, s_m]$ where $s_m = T 2^{-m}$ for $m=0, 1, \dots, M$, with $M = \lceil \log_2 (T/\theta^\star_{\min}) \rceil$. For each $m$, solving the Freedman bound for $\varepsilon$ with deterministic variance upper-bound $s_m v^\star$ guarantees that with probability at least $1 - \frac{\delta}{M+1}$, simultaneously for all $t \ge 1$,
\[
  D^t \le s_m \implies |S^t(v^\star)| \le \frac{2}{3}\log\left(\frac{2(M+1)}{\delta}\right) + \sqrt{2 s_m v^\star \log\left(\frac{2(M+1)}{\delta}\right)}.
\]
Taking a union bound over all $m=0, \dots, M$ ensures this holds simultaneously for all grid levels. Whenever $D^t > 0$, the realized $D^t$ must fall into some $(s_{m+1}, s_m]$, giving $s_m = 2 s_{m+1} < 2D$. Substituting this upper bound back into the empirical bound, we obtain that with probability at least $1-\delta$, simultaneously for all $t\in [T]$, $D^t >0$ implies that
\[
  |S^t(v^\star)| \le \frac{2}{3}\log\left(\frac{2(\lceil \log_2 (T/\theta^\star_{\min}) \rceil+1)}{\delta}\right) + 2\sqrt{D^t v^\star \log\left(\frac{2(\lceil \log_2 (T/\theta^\star_{\min}) \rceil+1)}{\delta}\right)}.
\]
Combining this inequality with the variance bound above and Eq.~\eqref{eq:bound_estimation_error_by_|S|}, we obtain
\[
  \left| \hat{v}^t-v^\star \right|
  \le \left(\frac{1}{4} D^t \right)^{-1} \left| S^t(v^\star) \right|
  \le 8 \sqrt{\frac{v^\star \log(c/\delta)}{D^t}} + \frac{8}{3} \frac{\log(c/\delta)}{D^t},
\]
where we abbreviate $c = 2(\lceil \log_2 (T/\theta^\star_{\min}) \rceil+1)$.
Defining $C_1= 8$ and $C_2= \frac{8}{3}$ yields the first claim. For the second claim, applying the AM-GM inequality $2\sqrt{ab}\le a+b$ with $a=v^\star$ and $b=16\log(c/\delta)/D^t$ gives 
\[8\sqrt{v^\star\log(c/\delta)/D}\le v^\star + 16\log(c/\delta)/D,
\] 
so
\[
\left| \hat{v}^t-v^\star \right|
\le v^\star + \frac{16 \log(c/\delta)}{D^t} + \frac{8}{3} \frac{\log(c/\delta)}{D^t} = v^\star + \frac{56}{3} \frac{\log(c/\delta)}{D^t}.
\]
Defining $C_3=\frac{56}{3}$ completes the proof.

\subsection{Proof of Lemma \ref{lem:ucb_construction_known_theta}}

By Lemma \ref{concentrationwithknowntheta}, we know that with probability at least $1-\delta$, simultaneously for all $t \in [T]$, $D^t >0$ implies that
\begin{equation}
  |\hat{v}^t -v^\star| \leq C_1 \sqrt{\frac{v^\star \log(c/\delta)}{D^t}} + C_2 \frac{\log(c/\delta)}{D^t}. \label{eq:estimationerrorbound}
\end{equation}
Assume Eq.~\eqref{eq:estimationerrorbound} holds. For notational convenience, we abbreviate the index $t$, define
\[
  A= \frac{C_2 \log(c/\delta)}{D} \quad \text{and} \quad B=C_1 \sqrt{\frac{\log(c/\delta)}{D}},
\]
then we have
\[
  v^\star - \hat{v} \leq A + B \sqrt{v^\star}.
\]
To isolate $v^\star$, define a new variable $s=\sqrt{v^\star}$. Substituting $v^\star = s^2$ into the previous inequality yields
\[
  s^2- B s - (\hat{v}+ A) \leq 0.
\]
This is a quadratic inequality in $s$, which implies that $s$ must lie below the positive root of the quadratic equation. Solving for the root gives
\[
  s \leq \frac{B + \sqrt{B^2 + 4 (\hat{v} + A)}}{2} \leq B + \sqrt{\hat{v}+ A},
\]
where the second inequality used $\sqrt{x+y}\le \sqrt{x} + \sqrt{y}$ for $x,y \ge 0$, and $\hat{v}\ge 0$.
Squaring both sides gives
\[
  v^\star = s^2 \leq B^2 + \hat{v}+ A + 2 B \sqrt{\hat{v}+ A},
\]
we further bound the square-root term using $\sqrt{\hat{v}+A}\le \sqrt{\hat{v}} + \sqrt{A}$, then
\[
  v^\star \le \hat v + B^2 + A + 2B\sqrt{A} + 2B\sqrt{\hat{v}}.
\]
Substituting the definitions of $A$ and $B$, the previous bound becomes
\[
  v^\star \le \hat{v} + 2C_1 \sqrt{\frac{\hat{v} \log(c/\delta)}{D}} + \left(C_1^2 + C_2 + 2 C_1 \sqrt{C_2}\right) \frac{\log(c/\delta)}{D}.
\]
Define the upper confidence bound
\[
  v^{\mathrm{ucb}} := \hat{v} + C_4 \sqrt{\frac{\hat{v} \log(c/\delta)}{D}} + C_5 \frac{\log(c/\delta)}{D}
\]
where $C_4=2C_1, C_5=C_1^2 + C_2 + 2 C_1 \sqrt{C_2}$. Then we have that Eq.~\eqref{eq:estimationerrorbound} implies $v^\star \le v^{\mathrm{ucb}}$.
Finally, combining the definition of $v^{\mathrm{ucb}}$ and Eq.~\eqref{eq:estimationerrorbound} yields
\[
  v^{\mathrm{ucb}}-v^\star = C_4 \sqrt{\frac{\hat{v} \log(c/\delta)}{D}} + C_5 \frac{\log(c/\delta)}{D} + \hat{v} - v^\star \le C_4 \sqrt{\frac{\hat{v} \log(c/\delta)}{D}} + C_5 \frac{\log(c/\delta)}{D} + A + B \sqrt{v^\star}.
\]
We use the inequality $\hat{v}\le 2 v^\star + C_3 \frac{\log(c/\delta)}{D}$, which follows from the second claim of Lemma~\ref{concentrationwithknowntheta} (namely $|\hat v-v^\star|\le v^\star + C_3\log(c/\delta)/D$) via the triangle inequality. Substituting and simplifying yields
\[
  v^{\mathrm{ucb}}-v^\star \le (C_4 \sqrt{C_3} + C_5 +C_2) \frac{\log(2/\delta)}{D} + (\sqrt{2} C_4  + C_1) \sqrt{\frac{v^\star \log(2/\delta)}{D}}.
\]
By defining $C_6:=\sqrt{2} C_4  + C_1, C_7:= C_4 \sqrt{C_3} + C_5 +C_2$, we complete the proof.

\subsection{The Structural Dominance Lemma}

\begin{lemma} \label{lem:structural_dominance}
Let $(\alpha_{i,k}^\star), (\overline{\alpha}_{i,k}) \in \mathbb{R}_{\ge 0}^{N\times K}$ be two attraction matrices with nonnegative entries, and define
\[
R_0((S,\sigma), (\alpha_{i,k})) \;:=\; \frac{\sum_{i\in S} r_i\, \alpha_{i,\sigma(i)}}{1 + \sum_{j\in S} \alpha_{j,\sigma(j)}}, \qquad (S,\sigma)\in\mathcal{F}.
\]
Let $(S^\star,\sigma^\star) \in \arg\max_{(S,\sigma)\in\mathcal{F}} R_0((S,\sigma), (\alpha_{i,k}^\star))$ and $(\bar S, \bar\sigma) \in \arg\max_{(S,\sigma)\in\mathcal{F}} R_0((S,\sigma), (\overline{\alpha}_{i,k}))$. If $\overline{\alpha}_{i,k} \ge \alpha^\star_{i,k}$ for all $(i,k)\in[N]\times[K]$, then
\[
R_0((\bar S, \bar\sigma), (\overline{\alpha}_{i,k})) \;\ge\; R_0((S^\star,\sigma^\star), (\alpha_{i,k}^\star)).
\]
\end{lemma}
\proof{Proof.}
Let $\rho^\star := R_0((S^\star,\sigma^\star),(\alpha_{i,k}^\star))$ and $\bar\rho := R_0((\bar S,\bar\sigma),(\overline{\alpha}_{i,k}))$. By Dinkelbach's reformulation (cf.\ Section~\ref{subsec:static_opt}), for any nonnegative matrix $(\alpha_{i,k})$ the optimal revenue $\rho((\alpha_{i,k})) = \max_{(S,\sigma)\in\mathcal{F}} R_0((S,\sigma),(\alpha_{i,k}))$ is the unique nonnegative root of
\[
F(\rho; (\alpha_{i,k})) \;:=\; \max_{(S,\sigma)\in\mathcal{F}}\sum_{i\in S}(r_i - \rho)\,\alpha_{i,\sigma(i)} \;-\; \rho.
\]
The inner maximum is a max-weight bipartite matching with edge weights $(r_i-\rho)\,\alpha_{i,k}$. Since $(S,\sigma)$ is unconstrained from below (the empty assortment is feasible and yields zero), products $i$ with $r_i<\rho$ contribute non-positively and can be excluded without loss; equivalently,
\[
F(\rho;(\alpha)_{i,k}) + \rho \;=\; \max_{(S,\sigma)\in\mathcal{F}}\sum_{i\in S:\, r_i\ge\rho}(r_i-\rho)\,\alpha_{i,\sigma(i)},
\]
which is monotone non-decreasing in each entry $\alpha_{i,k}$. Therefore $F(\rho;(\overline{\alpha}_{i,k}))\ge F(\rho;(\alpha_{i,k}^\star))$ for every $\rho\ge 0$. Substituting $\rho=\rho^\star$ yields
\[
F(\rho^\star;(\overline{\alpha}_{i,k})) \;\ge\; F(\rho^\star;(\alpha_{i,k}^\star)) \;=\; 0.
\]
The map $\rho\mapsto F(\rho;(\overline{\alpha}_{i,k}))$ is strictly decreasing in $\rho$ (the matching value above is non-increasing in $\rho$, and the explicit $-\rho$ term is strictly decreasing), so its unique root $\bar\rho$ satisfies $\bar\rho\ge\rho^\star$, which is the claim. 
\Halmos
\endproof

\subsection{Proof of Theorem \ref{regretboundofalg:P2MLE-UCB}}

We first establish a concentration property relating the number of times a product is offered to the number of effective observations it receives.
\begin{lemma} \label{lem:offered_times_win_times_for_theta_known}
Let $\delta \in (0,1)$, then with probability at least $1-\delta$, we have 
\[
\tau_i^t < 2 \log\left(\frac{NT}{\delta}\right)(1+V_{\max})^2 \quad \text{or} \quad n_i^t \ge \frac{(1+v^\star_i \theta^\star_{\min}) \tau_i^t}{2(1+V_{\max})}
\]
simultaneously for all $t \in [T]$ and all $i \in [N]$, where $\tau_i^t =\sum_{k=1}^K \tau_{i,k}^t, \tau_{i,k}^t := \sum_{s=1}^{t-1} \mathbf{1}\{i \in S_s, \sigma_s (i)= k\}$ and $n_i^t:=\sum_{k=1}^K n_{i,k}^t$.
\end{lemma}
To streamline the analysis, we divide the remainder of the proof into four logical steps.
\noindent \newline
\textbf{Step I: good events and regret decomposition.}
We define two high-probability events controlling the accuracy of the UCB estimates and the relationship between the offered and observed counts.
\begin{align*}
\mathcal{E}_1 := \Bigg\{\text{for all } & (i,t) \in [N] \times [T] : D_i^t > 0 \text{ implies that } \\ 
& v_i^\star \le v_i^{t,\mathrm{ucb}} \le v_i^\star + C_6 \sqrt{\frac{v_i^\star \log(3 c N T /2)}{D_i^t}} + \frac{C_7 \log(3 c N T /2)}{D_i^t} \Bigg\}.
\end{align*}
\[
\mathcal{E}_2 := \left\{ \tau_i^t < 2 \log(3 N T^2)(1+V_{\max})^2 \quad \text{or} \quad n_i^t \ge \frac{(1+v^\star_i \theta^\star_{\min}) \tau_i^t}{2(1+V_{\max})}, \text{for all } (i,t) \in [N] \times [T] \right\}
\]
By Lemma~\ref{lem:ucb_construction_known_theta} and Lemma~\ref{lem:offered_times_win_times_for_theta_known}, we have
\[
\mathbb{P}[\mathcal{E}_1] \ge 1 - \frac{2}{3 T} \quad \text{and} \quad \mathbb{P}[\mathcal{E}_2] \ge 1 - \frac{1}{3 T}.
\]
Let $\mathcal{E} := \mathcal{E}_1 \cap \mathcal{E}_2$, then $\mathbb{P}[\mathcal{E}] \ge 1 - 1/T$.
Since the per-round regret is at most $1$, the regret incurred when $\mathcal{E}$ fails is negligible:
\[
\sum_{t=1}^T (R(S^\star, \sigma^\star, \bm{v}^\star, \bm{\theta}^\star) - R(S_t, \sigma_t, \bm{v}^\star, \bm{\theta}^\star)) \mathbf{1}\{\mathcal{E}^c\} \le 1.
\]
We next separate the rounds according to whether each offered product has been observed sufficiently many times. Define the following event
\[
\mathcal{E}_{(t)} := \{ \tau_i^t \ge \max\{2 \log(3 N T^2) (1+V_{\max})^2, 2 K (1+V_{\max})\}, \forall i \in S_t \}
\]
for each round $t$. The event $\mathcal E_{(t)}$ does not hold only if some product in $S_t$ has been offered only a small number of times. Let $M:= \max\{2 \log(3 N T^2) (1+V_{\max})^2, 2 K (1+V_{\max})\}$, then $i \in  S_t$ together with $\tau_i^t < M$ can occur at most $M$ times since each such round increments $\tau_i^t$ by 1. We obtain that
\[
\sum_{t=1}^T \mathbf{1}\{\mathcal{E}_{(t)}^c\} \le \sum_t \sum_i \mathbf{1}\{i \in S_t, \tau_i^t < M\} = \sum_i \sum_t  \mathbf{1}\{i \in S_t, \tau_i^t < M\} \le N \cdot M.
\]
Therefore 
\begin{align*}
& \sum_{t=1}^T (R(S^\star, \sigma^\star, \bm{v}^\star, \bm{\theta}^\star) - R(S_t, \sigma_t, \bm{v}^\star, \bm{\theta}^\star)) \mathbf{1}\{\mathcal{E}\} \mathbf{1}\{\mathcal{E}_{(t)}^c\} \\
& \le N \cdot \max\{2 \log(3 N T^2) (1+V_{\max})^2, 2 K (1+V_{\max})\}.
\end{align*}
It remains to bound the term 
\[
\sum_{t=1}^T (R(S^\star, \sigma^\star, \bm{v}^\star, \bm{\theta}^\star) - R(S_t, \sigma_t, \bm{v}^\star, \bm{\theta}^\star)) \mathbf{1}\{\mathcal{E}, \mathcal{E}_{(t)}\}.
\]
\noindent
\textbf{Step II: decoupling the error via optimism and Cauchy-Schwarz.}
Under the event $\mathcal{E} \cap \mathcal{E}_{(t)}$, the clipped MLE satisfies $\bm{v}^\star \le \bm{v}^{t,\mathrm{ucb}}$ entrywise. Setting $\alpha^\star_{i,k}:= v^\star_i\,\theta^\star_k$ and $\overline{\alpha}_{i,k}:= v^{t,\mathrm{ucb}}_i\,\theta^\star_k$, the entrywise inequality $\alpha^\star_{i,k}\le \overline{\alpha}_{i,k}$ holds (since $\theta^\star_k\ge 0$). Applying Lemma~\ref{lem:structural_dominance} with $(\alpha^\star_{i,k})$ and $(\overline{\alpha}_{i,k})$ as above yields the optimism inequality directly at the optimum:
\[
R(S^\star,\sigma^\star,\bm{v}^\star,\bm{\theta}^\star)\,\mathbf{1}\{\mathcal{E},\mathcal{E}_{(t)}\}
\;\le\; R(S_t,\sigma_t,\bm{v}^{t,\mathrm{ucb}},\bm{\theta}^\star)\,\mathbf{1}\{\mathcal{E},\mathcal{E}_{(t)}\},
\]
where we have also used that $(S_t,\sigma_t)$ is the maximizer of $R(\cdot,\cdot,\bm{v}^{t,\mathrm{ucb}},\bm{\theta}^\star)$ on $\mathcal{F}$ (so $R((\bar S,\bar\sigma),(\overline{\alpha}_{i,k}))$ in Lemma~\ref{lem:structural_dominance} equals $R(S_t,\sigma_t,\bm{v}^{t,\mathrm{ucb}},\bm{\theta}^\star)$ for this choice of $(\overline{\alpha}_{i,k})$).
Combining these inequalities gives
\begin{align*}
& \sum_{t=1}^T(R(S^\star,\sigma^\star,\bm{v}^\star,\bm{\theta}^\star)-R(S_t,\sigma_t,\bm{v}^\star,\bm{\theta}^\star))\mathbf{1}\{\mathcal{E}, \mathcal{E}_{(t)}\} \\
& \le \sum_{t=1}^T\left(R(S_t,\sigma_t,\bm{v}^\mathrm{t,ucb},\bm{\theta}^\star)-R(S_t,\sigma_t,\bm{v}^\star,\bm{\theta}^\star)\right)\mathbf{1}\{\mathcal{E}, \mathcal{E}_{(t)}\}.
\end{align*}
Substituting the expression of the expected revenue, applying the inequality $\bm{v}^\star \mathbf{1}\{\mathcal{E}, \mathcal{E}_{(t)}\}\le \bm{v}^{t,\mathrm{ucb}} \mathbf{1}\{\mathcal{E}, \mathcal{E}_{(t)}\}$ and simplifying, the right hand side of the above inequality can be upper bounded by
\[
\sum_{t=1}^T\frac{1}{1+\sum_{j\in S_t}v_j^\star\theta_{\sigma_t(j)}^\star}\sum_{j\in S_t}r_j\theta_{\sigma_t(j)}^\star\left(v_j^{t,\mathrm{ucb}}-v_j^\star\right)\mathbf{1}\{\mathcal{E}, \mathcal{E}_{(t)}\}.
\]
Using $r_j\le 1$ and defining $v(S_t,\sigma_t):= 1+ \sum_{j \in S_t} v^\star_j \theta^\star_{\sigma_t (j)}$, the above can be upper bounded by 
\[
\sum_{i\in[N]}\sum_{t=1}^T\mathbf{1}\{i\in S_t\}\mathbf{1}\{\mathcal{E},\mathcal{E}_{(t)}\}\frac{\theta^\star_{\sigma_t(i)}(v_i^{t,\mathrm{ucb}}-v_i^\star)}{1+v(S_t,\sigma_t)}.
\]
Therefore, we have obtained that 
\[
\sum_{t=1}^T(R(S^\star,\sigma^\star,\bm{v}^\star,\bm{\theta}^\star)-R(S_t,\sigma_t,\bm{v}^\star,\bm{\theta}^\star))\mathbf{1}\{\mathcal{E},\mathcal{E}_{(t)}\} \le 
\sum_{i\in[N]}\sum_{t=1}^T\mathbf{1}\{i\in S_t\}\mathbf{1}\{\mathcal{E},\mathcal{E}_{(t)}\}\frac{\theta^\star_{\sigma_t(i)}(v_i^{t,\mathrm{ucb}}-v_i^\star)}{1+v(S_t,\sigma_t)}.
\]
Taking expectations, applying the Cauchy-Schwarz inequality and using $\theta^\star_j \le 1$, we obtain
\[
\mathbb{E}\left[\sum_{t=1}^T(R(S^\star,\sigma^\star,\bm{v}^\star,\bm{\theta}^\star)-R(S_t,\sigma_t,\bm{v}^\star,\bm{\theta}^\star))\mathbf{1}\{\mathcal{E},\mathcal{E}_{(t)}\}\right] \le \sum_{i\in [N]} \sqrt{A_i} \sqrt{B_i},
\]
where 
\[
A_i = \mathbb{E} \left[\sum_{t=1}^T \frac{1/|S_t| + v_i^\star\theta_{\sigma_t(i)}^\star}{1+v(S_t,\sigma_t)}\mathbf{1}\{i\in S_t\} \right],
\]
and 
\[
B_i=\mathbb{E}\left[\sum_{t=1}^{T}
\frac{\mathbf{1}\{i\in S_t\}\mathbf{1}\{\mathcal{E},\mathcal{E}_{(t)}\}
\theta_{\sigma_t(i)}^\star
\left(v_i^{t,\mathrm{ucb}}-v_i^\star\right)^2}
{\left(1/|S_t|+v_i^\star\theta_{\sigma_t(i)}^\star\right)(1+v(S_t,\sigma_t))} \right].
\]
We first bound the term $\sum_{i \in [N]}\sqrt{A_i}$. Applying the Cauchy-Schwarz inequality yields 
\[
\sum_{i \in [N]}\sqrt{A_i} \le \sqrt{N \sum_{i\in [N]} A_i}.
\]
Notice that
\[
\sum_{i\in [N]} A_i = \mathbb{E}{\left[\sum_{t=1}^T\sum_{i\in[N]}\frac{1/|S_t|+v_i^\star\theta_{\sigma_t(i)}^\star}{1+v(S_t,\sigma_t)}\mathbf{1}\{i\in S_t\}\right]} = \mathbb{E}\left[\sum_{t=1}^T 1\right] = T,
\]
we conclude that
\[
\sum_{i \in [N]}\sqrt{A_i}\le \sqrt{NT}.
\]
\noindent
\textbf{Step III: bounding the estimation error term ($B_i$).}
Notice that
\begin{equation}
\mathbf{1}\{\mathcal{E}, \mathcal{E}_{(t)}\} (v_i^{t,\mathrm{ucb}} - v_i^\star)^2 \le 2 C_6^2 \log(3 c N T /2) \frac{v^\star_i}{\sum_k n_{i,k}^t \theta_k} + 2 C_7^2 (\log(3 c N T /2))^2 \frac{1}{(\sum_k n_{i,k}^t \theta^\star_k)^2}. \label{eq:proof_of_thm1_withoutexploration_estimationerrorsquare}
\end{equation}
We first look at the term
\begin{equation}
\mathbb{E}\left[\sum_{t=1}^T \frac{\mathbf{1}\{i\in S_t\} \mathbf{1}\{\mathcal{E}, \mathcal{E}_{(t)}\}}{\left(\frac{1}{|S_t|}+v_i^\star\theta_{\sigma_t(i)}^\star\right)(1+v(S_t,\sigma_t))}\frac{v_i^\star\theta_{\sigma_t(i)}^\star}{\sum_kn_{i,k}^t\theta_k^\star}\right]. \label{eq:proof_of_thm1_withoutexploration_term1}
\end{equation}
Notice that we have $n^t_i >0$ under $\mathcal{E}\cap \mathcal{E}_{(t)}$, simple derivation yields that the above can be upper bounded by
\[
(\theta^\star_{\min})^{-1}\mathbb{E}\left[\sum_{t=1}^T \frac{\mathbf{1}\{i\in S_t\}}{1+v(S_t,\sigma_t)}\cdot\frac{1+v_i^\star\theta_{\sigma_t(i)}^\star}{n^t_i} \mathbf{1}\{n^t_i > 0\}\right].
\]
Recall that $\mathcal{F}_{t-1} = \sigma(S_1, \sigma_1, i_1, \dots, S_{t-1}, \sigma_{t-1}, i_{t-1}, S_t, \sigma_t)$, then we have $\mathbf{1}\{i \in S_t\}, \mathbf{1}\{n_i^t >0\} \in \mathcal{F}_{t-1}$, and
\[
\mathbb{E}[\mathbf{1}\{i_t \in \{i, 0\}\} \mid \mathcal{F}_{t-1}] = \frac{1+v^\star_i \theta^\star_{\sigma_t (i)}}{1+v(S_t,\sigma_t)} \mathbf{1}\{i \in S_t\} + \frac{1}{1+v(S_t,\sigma_t)} \mathbf{1}\{i \notin S_t\}.
\]
Therefore, for $1\le t\le T$,
\begin{align*}
\mathbb{E}\left[\frac{\mathbf{1}\{i_t\in\{i,0\},i\in S_t\}}{n^t_i} \mathbf{1}\{n^t_i > 0\}\right]
& = \mathbb{E}\left[\mathbb{E}\left[\frac{\mathbf{1}\{i_t\in\{i,0\},i\in S_t\}}{n^t_i} \mathbf{1}\{n^t_i > 0\} \mid \mathcal{F}_{t-1}\right]\right] \\
& = \mathbb{E}\left[\frac{\mathbf{1}\{i \in S_t\}}{n^t_i} \mathbf{1}\{n^t_i > 0\} \cdot \mathbb{E}\left[\mathbf{1}\{i_t\in\{i,0\}\}\mid\mathcal{F}_{t-1}\right]\right] \\
& = \mathbb{E}\left[\frac{\mathbf{1}\{i \in S_t\}}{n^t_i} \cdot \frac{1+v_i^\star\theta_{\sigma_t(i)}^\star}{1+v(S_t,\sigma_t)} \mathbf{1}\{n^t_i > 0\}\right],
\end{align*}
hence we have
\begin{equation}
\mathbb{E}\left[\sum_{t=1}^T \frac{\mathbf{1}\{i\in S_t\}}{1+v(S_t,\sigma_t)}\cdot\frac{1+v_i^\star\theta_{\sigma_t(i)}^\star}{n^t_i} \mathbf{1}\{n^t_i > 0\} \right] = \mathbb{E}\left[\sum_{t=1}^T\frac{\mathbf{1}\{i_t\in\{i,0\},i\in S_t\}}{n^t_i} \mathbf{1}\{n^t_i > 0\}\right]. \label{eq:proof_of_thm1_withoutexploration_eqofexpectation}
\end{equation}
And the term $\sum_{t=1}^T \mathbf{1}\{i_t\in\{i,0\},i\in S_t\} \mathbf{1}\{n^t_i > 0\} / n^t_i$ can be upper bounded by $1 + \log T$. Therefore, the term~\eqref{eq:proof_of_thm1_withoutexploration_term1} can be bounded by $(\theta^\star_{\min})^{-1}(1 + \log T)$.
Then we bound the term
\[
\mathbb{E}\left[\sum_{t=1}^T\frac{\mathbf{1}\{i\in S_t\} \mathbf{1}\{\mathcal{E},\mathcal{E}_{(t)}\}}{\left(\frac{1}{|S_t|}+v_i^\star\theta_{\sigma_t(i)}^\star\right)(1+v(S_t,\sigma_t))}\frac{1}{\left(\sum_kn_{i,k}^t\theta_k^\star\right)^2}\right]. 
\]
Simple derivation yields that, the above can be bounded by
\[
(\theta_{\min}^\star)^{-2} \mathbb{E} \left[\sum_{t=1}^T \frac{\mathbf{1}\{i\in S_t\}}{1+v(S_t,\sigma_t)}\cdot\frac{1}{n^t_i}\cdot\frac{1}{\left(\frac{1}{|S_t|}+v_i^\star\theta_{\sigma_t(i)}^\star\right)n^t_i} \mathbf{1}\{\mathcal{E},\mathcal{E}_{(t)}\}\right].
\]
Notice that $n^t_i \ge K$ under $\mathcal{E} \cap \mathcal{E}_{(t)}$, the above can be bounded by
\[
(\theta^\star_{\min})^{-2}\, \mathbb{E}\left[\sum_{t=1}^T \frac{\mathbf{1}\{i \in S_t\}}{1+v(S_t,\sigma_t)} \cdot \frac{1}{n^t_i}\, \mathbf{1}\{\mathcal{E}, \mathcal{E}_{(t)}\}\right].
\]
From Eq.~\eqref{eq:proof_of_thm1_withoutexploration_eqofexpectation} and the discussion following it, the above can be bounded by $\left(\theta_{\min}^\star\right)^{-2}(1+\log T)$.
Combining these results and using Eq.~\eqref{eq:proof_of_thm1_withoutexploration_estimationerrorsquare} we obtain
\begin{align*}
  \sqrt{B_i} & \le \sqrt{2} C_6 \sqrt{\log(3 c N T /2)} (\theta^\star_{\min})^{-1/2} \sqrt{1+\log T} + \sqrt{2} C_7 \log(3 c N T /2) (\theta^\star_{\min})^{-1} \sqrt{1+\log T}\\
  & = \left(C_6 (\theta^\star_{\mathrm{min}})^{-1/2} \sqrt{\log (3 c NT /2)} + C_7 (\theta^\star_{\mathrm{min}})^{-1} \log (3 c NT /2) \right) \sqrt{2(1+\log T)}.
\end{align*}
\noindent
\textbf{Step IV: final regret bound synthesis.}
To conclude the proof, let us bundle the parameter-dependent constant factor as:
\[
  \tilde{C} := \left(C_6 (\theta^\star_{\mathrm{min}})^{-1/2} \sqrt{\log (3 c NT /2)} + C_7 (\theta^\star_{\mathrm{min}})^{-1} \log (3 c NT /2) \right) \sqrt{2(1+\log T)}.
\]
Multiplying the upper bounds for $A_i$ and $B_i$ yields:
\[
\sum_{i\in [N]} \sqrt{A_i} \sqrt{B_i} \le \tilde{C} \sqrt{NT}.
\]
Combining all parts, we obtain the final total expected cumulative regret bound
\[
  \text{Regret}(T, \pi)
  \le 1 + N \cdot \max\{2 \log(3 N T^2) (1+V_{\max})^2, 2 K (1+V_{\max})\} + \tilde{C} \sqrt{NT}.
\]

\subsection{Proof of Lemma~\ref{lem:offered_times_win_times_for_theta_known}}

Recall that $\tau_{i}^t = \sum_{s=1}^{t-1} \mathbf{1}\{i \in S_s\}$ is the total number of times product $i$ was offered at any position up to the start of round $t$, and $n_{i}^t = \sum_{s=1}^{t-1} \mathbf{1}\{i_s \in \{i, 0\}, i \in S_s\}$ is the number of times product $i$ or the outside option $0$ was chosen when $i$ was offered. 

When product $i$ is offered at round $s$ in assortment $S_s$ and position $\sigma_s(i) = k$, the conditional probability that either $0$ or $i$ wins is:
\[
\mathbb{P}[i_s \in \{i, 0\} \mid \mathcal{F}_{s-1}] = \frac{1 + v_i^\star \theta_k^\star}{1 + V(S_s, \sigma_s)} \ge \frac{1 + v_i^\star \theta_{\min}^\star}{1 + V_{\max}},
\]
where $V(S_s, \sigma_s) = \sum_{j \in S_s} v_j^\star \theta_{\sigma_s(j)}^\star$ and $V_{\max} = \max_{S, \sigma} V(S, \sigma) \le K$.

Because the sample size $\tau_i^t$ is a random variable dependent on the adaptive sampling policy, we cannot apply concentration inequalities directly over $t$. Instead, we construct a martingale indexed by $m$, the number of times product $i$ is offered. 

For each $m \ge 1$, let $s_m$ be the stopping time representing the round index of the $m$-th time product $i$ is offered (if $i$ is offered fewer than $m$ times, we set $s_m = \infty$). We define the filtration $\tilde{\mathcal{F}}_m = \mathcal{F}_{s_m}$, which records the history up to the $m$-th offer of product $i$. For $m$ such that $s_m < \infty$, let $I_m = \mathbf{1}\{i_{s_m} \in \{i, 0\}\}$ be the indicator of a successful observation, and let $P_m = \mathbb{P}[i_{s_m} \in \{i, 0\} \mid \mathcal{F}_{s_m-1}]$ be its conditional probability.

We define the centered sequence:
\[
X_m^{(i)} = \begin{cases} 
I_m - P_m, & \text{if } s_m < \infty \\ 
0, & \text{if } s_m = \infty 
\end{cases}
\]
By construction, $\{X_m^{(i)}\}_{m=1}^T$ is a martingale difference sequence adapted to the filtration $\{\tilde{\mathcal{F}}_m\}_{m=1}^T$, satisfying $\mathbb{E}[X_m^{(i)} \mid \tilde{\mathcal{F}}_{m-1}] = 0$. Furthermore, each increment strictly lies in the bounded interval $[-P_m, 1-P_m]$ of length 1. 

Let $M_m^{(i)} = \sum_{\ell=1}^m X_\ell^{(i)}$. Applying the Azuma-Hoeffding inequality (Lemma~\ref{lem:azuma}) for variables with range 1, we have for any fixed $m$ and any $x > 0$:
\[
\mathbb{P}\left( M_m^{(i)} \le -x \right) \le \exp\left(-\frac{2x^2}{m}\right).
\]
We wish to establish a bound that holds uniformly over all products $i \in [N]$ and all possible offer counts $m \in \{1, \dots, T\}$. Setting $x = \sqrt{\frac{m}{2} \log\left(\frac{NT}{\delta}\right)}$ and applying a union bound over $i \in [N]$ and $m \in [T]$, we obtain:
\[
\mathbb{P}\left( \exists i \in[N], \exists m \in [T] : M_m^{(i)} \le - \sqrt{\frac{m}{2} \log\left(\frac{NT}{\delta}\right)} \right) \le N \cdot T \cdot \exp\left(-\log\left(\frac{NT}{\delta}\right)\right) = \delta.
\]
Therefore, with probability at least $1 - \delta$, the bound holds simultaneously for all $i$ and $m$. 

Now we translate this uniform bound back to the original time horizon $t$. Notice that for any round $t$, the number of offers is precisely $\tau_i^t \le T$. Substituting $m = \tau_i^t$ and rearranging $M_{\tau_i^t}^{(i)} = n_i^t - \sum_{\ell=1}^{\tau_i^t} P_\ell$, we have:
\[
n_i^t \ge \sum_{\ell=1}^{\tau_i^t} P_\ell - \sqrt{\frac{\tau_i^t}{2} \log\left(\frac{NT}{\delta}\right)}.
\]
Since $P_\ell \ge \frac{1 + v_i^\star \theta_{\min}^\star}{1 + V_{\max}}$ for all valid $\ell$, the sum is bounded below by $\tau_i^t \frac{1 + v_i^\star \theta_{\min}^\star}{1 + V_{\max}}$, giving:
\[
n_i^t \ge \frac{1 + v_i^\star \theta_{\min}^\star}{1 + V_{\max}}\tau_i^t - \sqrt{\frac{\tau_i^t}{2} \log\left(\frac{NT}{\delta}\right)}.
\]
To complete the proof, we find the condition under which the right-hand side is further bounded below by half of its principal term, i.e., $\frac{1 + v_i^\star \theta_{\min}^\star}{2(1 + V_{\max})} \tau_i^t$. This holds whenever:
\[
\frac{1 + v_i^\star \theta_{\min}^\star}{2(1 + V_{\max})} \tau_i^t \ge \sqrt{\frac{\tau_i^t}{2} \log\left(\frac{NT}{\delta}\right)}.
\]
Squaring both sides and isolating $\tau_i^t$ yields the requirement:
\[
\tau_i^t \ge 2 \log\left(\frac{NT}{\delta}\right) \frac{(1 + V_{\max})^2}{(1 + v_i^\star \theta_{\min}^\star)^2}.
\]
Because $v_i^\star, \theta_{\min}^\star \ge 0$ implies $(1 + v_i^\star \theta_{\min}^\star)^2 \ge 1$, the hypothesis condition $\tau_i^t \ge 2 \log\left(\frac{NT}{\delta}\right) (1 + V_{\max})^2$ strictly implies the above inequality. This yields:
\[
n_i^t \ge \frac{1 + v_i^\star \theta_{\min}^\star}{2(1 + V_{\max})} \tau_i^t,
\]
which concludes the proof.

\section{Proofs Omitted in Section \ref{sec:general_analysis}}

\subsection{Proof of Lemma \ref{lemmaconcentrationofp}}
Fix any $(i,k) \in [N]\times[K]$ and any round $t$. The sample size $n_{i,k}^t$ is a predictable random count under adaptive sampling, so direct application of Bernstein's inequality (which assumes a deterministic sample size and i.i.d.\ summands) is not valid here. We instead use the precise martingale-difference structure of Appendix~\ref{sec:probabilistic_setup}. 

Specifically, recall the centered increments $X_s^{(i,k)}$ defined over the refined filtration $\mathcal{G}_s^{(i)}$, where $\mathcal{G}_{s-1}^{(i)}$ augments $\mathcal{F}_{s-1}$ with the indicator of an active pairwise comparison at round $s$ (making $A_s^{(i,k)}$ fully $\mathcal{G}_{s-1}^{(i)}$-measurable) but withholds the final purchase decision $i_s$. This yields the exact equivalence $w_{i,k}^t - n_{i,k}^t p_{i,k} = \sum_{s<t} X_s^{(i,k)}$.

Crucially, the predictable quadratic variation precisely matches the empirical variance path:
\[
W_t^{(i,k)} := \sum_{s=1}^{t-1} \mathbb{E}\!\left[(X_s^{(i,k)})^2 \mid \mathcal{G}_{s-1}^{(i)}\right] = \sum_{s=1}^{t-1} A_s^{(i,k)} p_{i,k}(1-p_{i,k}) = n_{i,k}^t p_{i,k}(1-p_{i,k}).
\]
We can apply Freedman's inequality (Lemma~\ref{lem:freedman}, with $R=1$, combined with a union bound to obtain the two-sided form) to $S_{t-1} = \sum_{s=1}^{t-1} X_s^{(i,k)} = w_{i,k}^t - n_{i,k}^t p_{i,k}$.
Rather than taking a loose union bound over all linearly many possible active sample sizes, we employ a geometric peeling argument over the interval $[1, T]$ for the possible values of $n_{i,k}^t > 0$.

We divide the interval $(0, T]$ into grids $(s_{m+1}, s_m]$ where $s_m = T 2^{-m}$ for $m=0, 1, \dots, \lceil \log_2 T \rceil$. For each $m$, solving the Freedman bound with deterministic variance upper-bound $s_m p_{i,k}(1-p_{i,k})$ guarantees that, with probability at least $1 - \frac{\delta}{\lceil \log_2 T \rceil + 1}$, simultaneously for all $t\in [T]$
\[
n_{i,k}^t \le s_m \implies |w_{i,k}^t - n_{i,k}^t p_{i,k}| \le \sqrt{2 s_m p_{i,k}(1-p_{i,k}) \log\left(\frac{2(\lceil \log_2 T \rceil + 1)}{\delta}\right)} + \frac{2}{3} \log\left(\frac{2(\lceil \log_2 T \rceil + 1)}{\delta}\right).
\]
Taking a union bound over all $m=0, \dots, \lceil \log_2 T \rceil$ ensures this holds simultaneously for all grid levels. Whenever $n_{i,k}^t > 0$, the realized count must fall into some $(s_{m+1}, s_m]$, giving $s_m = 2 s_{m+1} < 2 n_{i,k}^t$. Substituting this upper bound back into the inequality, we obtain that with probability at least $1-\delta$, simultaneously for all $t\in [T]$, $n^t_{i,k}>0$ implies
\[
|w_{i,k}^t - n_{i,k}^t p_{i,k}| \le 2\sqrt{n_{i,k}^t p_{i,k}(1-p_{i,k}) \log\left(\frac{2(\lceil\log_2T\rceil+1)}{\delta}\right)} + \frac{2}{3}\log\left(\frac{2(\lceil\log_2T\rceil+1)}{\delta}\right).
\]
Abbreviating $L := \log(2(\lceil\log_2T\rceil+1)/\delta)$ and dividing both sides by $n_{i,k}^t$ implies that
\begin{equation}
  |\hat{p}_{i,k}^t - p_{i,k}| \le 2\sqrt{\frac{p_{i,k} (1-p_{i,k}) L}{n^t_{i,k}}} + \frac{2 L}{3n^t_{i,k}}.   \label{eq:proof_of_lemma3_inequality1}
\end{equation}

Assume Eq.~\eqref{eq:proof_of_lemma3_inequality1} holds. Since the function $x \mapsto x(1-x)$ is $1$-Lipschitz on $[0,1]$, we have
\[
  |\hat{p}_{i,k}^t (1- \hat{p}_{i,k}^t) - p_{i,k}(1- p_{i,k})| \le |\hat{p}_{i,k}^t - p_{i,k}|.
\]
Therefore,
\[
  p_{i,k}(1- p_{i,k}) \le  \hat{p}_{i,k}^t (1- \hat{p}_{i,k}^t) + 2\sqrt{\frac{p_{i,k} (1- p_{i,k}) L}{n_{i,k}^t}} + \frac{2 L}{3 n_{i,k}^t}.
\]
Rearranging terms, we obtain a quadratic inequality
\[
  p_{i,k}(1- p_{i,k}) - 2\sqrt{\frac{p_{i,k} (1- p_{i,k}) L}{n_{i,k}^t}} - \hat{p}_{i,k}^t (1- \hat{p}_{i,k}^t) - \frac{2 L}{3 n_{i,k}^t} \le 0.
\]
Viewing the left hand side as a quadratic polynomial in $\sqrt{p_{i,k} (1- p_{i,k})}$, solving for the positive root and using $\sqrt{a+b} \le \sqrt{a} + \sqrt{b}$ to simplify, yields
\begin{equation}
  \sqrt{p_{i,k} (1- p_{i,k})} \le \sqrt{\hat{p}_{i,k}^t (1- \hat{p}_{i,k}^t)} + \left(1 + \sqrt{5/3}\right) \sqrt{\frac{L}{n_{i,k}^t}}. \label{eq:proof_of_lemma3_inequality2}
\end{equation}
Substituting Eq.~\eqref{eq:proof_of_lemma3_inequality2} back into Eq.~\eqref{eq:proof_of_lemma3_inequality1}, we obtain
\begin{align*}
  |\hat{p}_{i,k}^t - p_{i,k}| &\le 2\sqrt{\frac{\hat{p}_{i,k}^t (1- \hat{p}_{i,k}^t) L}{n_{i,k}^t}} + \left(2\left(1 + \sqrt{5/3}\right) + \frac{2}{3}\right) \frac{L}{n_{i,k}^t} \\
  &\le 2\sqrt{\frac{\hat{p}_{i,k}^t (1- \hat{p}_{i,k}^t) L}{n_{i,k}^t}} + \frac{6 L}{n_{i,k}^t}.
\end{align*}
So we obtain that
\[
  p_{i,k} \le \hat{p}_{i,k}^t + 2\sqrt{\frac{\hat{p}_{i,k}^t (1- \hat{p}_{i,k}^t) L}{n_{i,k}^t}} + \frac{6 L}{n_{i,k}^t}
\]
and
\[
  \hat{p}_{i,k}^t (1- \hat{p}_{i,k}^t) \le p_{i,k}(1- p_{i,k}) + 2\sqrt{\frac{\hat{p}_{i,k}^t (1- \hat{p}_{i,k}^t) L}{n_{i,k}^t}} + \frac{6 L}{n_{i,k}^t}.
\]
Recall that $p_{i,k}\le 1/2$ and the definition of $p_{i,k}^{t,\mathrm{ucb}}$:
\[
  p_{i,k}^{t,\mathrm{ucb}} := \min\left\{\hat{p}_{i,k}^t + 2\sqrt{\frac{\hat{p}_{i,k}^t (1- \hat{p}_{i,k}^t) L}{n_{i,k}^t}} + \frac{6 L}{n_{i,k}^t}, 1/2\right\}
\]
we obtain that $p_{i,k}^{t,\mathrm{ucb}}\ge p_{i,k}$.
Similar to deriving Eq.~\eqref{eq:proof_of_lemma3_inequality2}, we obtain
\begin{equation}
  \sqrt{\hat{p}_{i,k}^t (1- \hat{p}_{i,k}^t)} \le \sqrt{p_{i,k}(1-p_{i,k})} + (1+\sqrt{7}) \sqrt{\frac{L}{n_{i,k}^t}}. \label{eq:proof_of_lemma3_inequality3}
\end{equation}
Combining the definition of $p_{i,k}^{t,\mathrm{ucb}}$ with Eq.~\eqref{eq:proof_of_lemma3_inequality1}, \eqref{eq:proof_of_lemma3_inequality3}, yields
\begin{align*}
  p_{i,k}^{t,\mathrm{ucb}} &\le p_{i,k} + 4 \sqrt{\frac{p_{i,k}(1-p_{i,k}) L}{n_{i,k}^t}} + \left(\frac{2}{3} + 6 + 2(1+\sqrt{7})\right) \frac{L}{n_{i,k}^t} \\
  &\le p_{i,k} + 4 \sqrt{\frac{p_{i,k}(1-p_{i,k}) L}{n_{i,k}^t}} + \frac{14 L}{n_{i,k}^t}.
\end{align*}

\subsection{Proof of Lemma \ref{lemmaconcentrationofv}}

Applying Lemma~\ref{lemmaconcentrationofp} with confidence level $\delta/(KNT)$ and taking a union bound over all $(i,k,t)$, we abbreviate $\tilde{L} = \log(2KNT(\lceil\log_2 T\rceil+1)/\delta)$ and obtain that with probability at least $1-\delta$,
\[
  p_{i,k} \le p_{i,k}^{t,\mathrm{ucb}} \le p_{i,k} + 4 \sqrt{\frac{p_{i,k} (1-p_{i,k}) \tilde{L}}{n_{i,k}^t}} + \frac{14 \tilde{L}}{n_{i,k}^t}, \quad \text{if}\ n_{i,k}^t > 0,
\]
simultaneously for all $t \in [T]$ and $(i,k) \in [N]\times[K]$. We assume that this high-probability event holds and $n^t_{i,k}>0$ in the remainder of the proof. Define the function $f \colon x \mapsto x/(1-x)$ on $[0,1/2]$, then we have $v_{i,k}^{t,\mathrm{ucb}} = f(p_{i,k}^{t,\mathrm{ucb}})$ and $v_{i,k}^\star = f(p_{i,k})$. Since $f$ is non-decreasing, and $p_{i,k}^{t,\mathrm{ucb}} \ge p_{i,k}$, we have
\[
  v_{i,k}^{t,\mathrm{ucb}} \ge v_{i,k}^\star.
\]
Furthermore, denote
\[
  \Delta_{i,k}^t := 4 \sqrt{\frac{p_{i,k} (1-p_{i,k}) \tilde{L}}{n_{i,k}^t}} + \frac{14 \tilde{L}}{n_{i,k}^t} = 4 \sqrt{\frac{v_{i,k}^\star \tilde{L}}{(1+ v_{i,k}^\star)^2 n_{i,k}^t}} + \frac{14 \tilde{L}}{n_{i,k}^t}.
\]
In the rest of the proof, we assume that $n_{i,k}^t \ge 80 \tilde{L} (1 + v_{i,k}^\star)$. Then we have
\[
  (1+v_{i,k}^\star) \Delta_{i,k}^t = 4 \sqrt{\frac{v_{i,k}^\star \tilde{L}}{n_{i,k}^t}} + \frac{14 \tilde{L} (1+ v_{i,k}^\star)}{n_{i,k}^t} \le \frac{4}{\sqrt{160}} + \frac{14}{80} < 1/2,
\]
and
\[
  p_{i,k} + \Delta_{i,k}^t = \frac{v_{i,k}^\star}{1+v_{i,k}^\star} + \Delta_{i,k}^t \le \frac{v_{i,k}^\star+1/2}{1+ v_{i,k}^\star} < 1.
\]
Using the monotonicity of $f$, we have
\[
  v_{i,k}^{t,\mathrm{ucb}} - v_{i,k}^\star =
  f(p_{i,k}^{t,\mathrm{ucb}}) - f(p_{i,k})
  \le
  f(p_{i,k} + \Delta_{i,k}^t) - f(p_{i,k}).
\]
A direct calculation gives
\[
  f(p_{i,k} + \Delta_{i,k}^t) - f(p_{i,k})
  =
  \frac{\Delta_{i,k}^t}{(1-p_{i,k})(1-p_{i,k}-\Delta_{i,k}^t)}
  =
  \frac{(1+v_{i,k}^\star)^2 \Delta_{i,k}^t}{1 - (1+v_{i,k}^\star)\Delta_{i,k}^t}.
\]
Using $(1+v_{i,k}^\star)\Delta_{i,k}^t \le 1/2$ and substituting the definition of $\Delta_{i,k}^t$, we obtain
\[
  v_{i,k}^{t,\mathrm{ucb}} - v_{i,k}^\star
  \le
  2 (1+v_{i,k}^\star)^2 \Delta_{i,k}^t
  =
  8 (1 + v_{i,k}^\star) \sqrt{\frac{v_{i,k}^\star \tilde{L}}{n_{i,k}^t}} + \frac{28 \tilde{L} (1 + v_{i,k}^\star)^2}{n_{i,k}^t}.
\]

\subsection{Proof of Theorem \ref{regretboundofgeneralmodel}}

We first establish a concentration property relating the number of times a product is offered and the number of effective observations.

\begin{lemma} \label{lem:offered_times_win_times_general_model}
Let $T \ge 1$ and $\delta \in (0,1)$. Then, with probability at least $1-\delta$,
\[
\tau_{i,k}^t < 2 \log\left(\frac{KNT}{\delta}\right) (1 + V_{\max})^2 \quad \text{or} \quad n_{i,k}^t \ge \frac{(1 + v_{i,k}^\star) \tau_{i,k}^t}{2 (1+ V_{\max})},
\]
simultaneously for all $t \in [T]$ and $(i,k) \in [N] \times [K]$, where $\tau_i^t =\sum_{k=1}^K \tau_{i,k}^t, \tau_{i,k}^t := \sum_{s=1}^{t-1} \mathbf{1}\{i \in S_s, \sigma_s (i)= k\}$ and $n_i^t:=\sum_{k=1}^K n_{i,k}^t$.
\end{lemma}
To streamline the analysis, we divide the remainder of the proof into four logical steps.
\noindent \newline
\textbf{Step I: good events and regret decomposition.}
We define two high-probability events controlling the accuracy of the UCB estimates and the relationship between the offered and observed counts.

Let $x= \log (3 K N T (\lceil \log_2 T \rceil + 1))$. Let $\mathcal{E}_1$ be the event that for all $t\in [T]$ and $(i,k)\in [N]\times [K]$: $v_{i,k}^\star \le v_{i,k}^{t,\mathrm{ucb}}$ and, $n_{i,k}^t \ge 80 x (1 + v_{i,k}^\star)$ implies that
\[
v_{i,k}^{t,\mathrm{ucb}} \le v_{i,k}^\star + 8 (1+v_{i,k}^\star) \sqrt{v_{i,k}^\star x / n_{i,k}^t} + 28 x (1+v_{i,k}^\star)^2 / n_{i,k}^t.
\]
Then by Lemma~\ref{lemmaconcentrationofv} and a union bound, we have $\P[\mathcal{E}_1] \ge 1 - 2/(3T)$.

Let $\mathcal{E}_2$ be the event that for all $t \in [T]$ and $(i,k) \in [N]\times [K]$: either $\tau_{i,k}^t < 2 \log\left(3KNT^2\right) (1 + V_{\max})^2$ or $n_{i,k}^t \ge (1 + v_{i,k}^\star) \tau_{i,k}^t / (2 (1+ V_{\max}))$ is satisfied. Then by Lemma~\ref{lem:offered_times_win_times_general_model}, we have $\mathbb{P}[\mathcal{E}_2] \ge 1-1/(3T)$.

Let $\mathcal{E}:=\mathcal{E}_1\cap \mathcal{E}_2$, then $\mathbb{P}[\mathcal{E}] \ge 1 - 1/T$. Since the per-round regret is at most $1$, the regret incurred when $\mathcal{E}$ fails is at most 
\[
\sum_{t=1}^T (R(S^\star, \sigma^\star, \bm{v}^\star) - R(S_t, \sigma_t, \bm{v}^\star)) \mathbf{1}\{\mathcal{E}^c\} \le 1. 
\]
We next separate the rounds according to whether each offered product has been observed sufficiently many times. Let $M= \max  \{2 \log(3 K N T^2) (1+V_{\mathrm{max}})^2,  160 x (1+K)^2\}$ and define the following event
\[
\mathcal{E}_{(t)} := \{\tau_{i,k}^t \ge M, \forall (i,k) \in (S_t, \sigma_t)\}
\]
for each round $t$. 
The event $\mathcal{E}_{(t)}$ fails only if some pair $(i,k) \in (S_t, \sigma_t)$ satisfies $\tau_{i,k}^t < M$; call such a pair under-offered at round $t$. Fix any pair $(i,k)$, it can be in $(S_t, \sigma_t)$ while under-offered only during rounds $t$ for which $\tau_{i,k}^t <M$, which happens at most $M$ times in total (each such round increases $\tau_{i,k}^t$ by $1$). Hence, summing over the $KN$ pairs,
\[
\sum_{t=1}^T \mathbf{1}\{\mathcal{E}_{(t)}^c\} \le \sum_{(i,k)} M \le KN M.
\]
Therefore,
\[
\sum_{t=1}^T (R(S^\star, \sigma^\star, \bm{v}^\star) - R(S_t, \sigma_t, \bm{v}^\star)) \mathbf{1}\{\mathcal{E}, \mathcal{E}_{(t)}^c\} \le KN\cdot \max  \{2 \log(3 K N T^2) (1+V_{\mathrm{max}})^2,  160 x (1+K)^2\}.
\]
It remains to bound the following term:
\[
\sum_{t=1}^T (R(S^\star, \sigma^\star, \bm{v}^\star) - R(S_t, \sigma_t, \bm{v}^\star)) \mathbf{1}\{\mathcal{E}, \mathcal{E}_{(t)}\}.
\]
\noindent
\textbf{Step II: decoupling the error via optimism and Cauchy-Schwarz.}
Under $\mathcal{E} \cap \mathcal{E}_{(t)}$, the entrywise inequality $v^\star_{i,k}\le v^{t,\mathrm{ucb}}_{i,k}$ holds for every $(i,k)\in[N]\times[K]$. Apply Lemma~\ref{lem:structural_dominance} with $\alpha^\star_{i,k}=v^\star_{i,k}$ and $\overline{\alpha}_{i,k}=v^{t,\mathrm{ucb}}_{i,k}$ and using $(S_t,\sigma_t)$ is the maximizer of $R(\cdot,\cdot,\bm{v}^{t,\mathrm{ucb}})$ on $\mathcal{F}$, we obtain
\[
R(S^\star,\sigma^\star,\bm{v}^\star)\,\mathbf{1}\{\mathcal{E},\mathcal{E}_{(t)}\}
\;\le\; R(S_t,\sigma_t,\bm{v}^{t,\mathrm{ucb}})\,\mathbf{1}\{\mathcal{E},\mathcal{E}_{(t)}\}.
\]
These inequalities yield
\[
\sum_{t=1}^T(R(S^\star,\sigma^\star,\bm{v}^\star)-R(S_t,\sigma_t,\bm{v}^\star))\mathbf{1}\{\mathcal{E},\mathcal{E}_{(t)}\} 
\le \sum_{t=1}^T\left(R(S_t,\sigma_t,\bm{v}^\mathrm{t,ucb})-R(S_t,\sigma_t,\bm{v}^\star)\right)\mathbf{1}\{\mathcal{E},\mathcal{E}_{(t)}\}.
\]
Substituting the expression of the expected revenue, applying the inequality $\bm{v}^\star \mathbf{1}\{\mathcal{E},\mathcal{E}_{(t)}\}\le \bm{v}^{t,\mathrm{ucb}} \mathbf{1}\{\mathcal{E},\mathcal{E}_{(t)}\}$ and simplifying, the right hand side of the above inequality can be upper bounded by
\[
\sum_{t=1}^T \sum_{i\in S_t} \frac{r_i \left(v^{t,\mathrm{ucb}}_{i,\sigma_t (i)}-v^\star_{i,\sigma_t (i)}\right)}{1+\sum_{j\in S_t}v^\star_{j,\sigma_t (j)}} \mathbf{1}\{\mathcal{E},\mathcal{E}_{(t)}\}.
\]
Using $r_j\le 1$ and defining $V(S_t,\sigma_t):= 1+ \sum_{j \in S_t} v^\star_{j,\sigma_t (j)}$, the above can be upper bounded by 
\[
\sum_{t=1}^T \sum_{(i,k)} \mathbf{1}\{i \in S_t, \sigma_t (i)=k\} \frac{v_{i,k}^{t,\mathrm{ucb}} - v_{i,k}^\star}{1 + V(S_t, \sigma_t)} \mathbf{1}\{\mathcal{E},\mathcal{E}_{(t)}\}.
\]
Taking expectations and applying the Cauchy-Schwarz inequality, we obtain that
\[
\E \left[ \sum_{t=1}^T \sum_{(i,k)} \mathbf{1}\{i \in S_t, \sigma_t(i)=k\} \frac{v_{i,k}^{t,\mathrm{ucb}} - v_{i,k}^\star}{1 + V(S_t, \sigma_t)} \mathbf{1}\{\mathcal{E}, \mathcal{E}_{(t)}\} \right] \le \sum_{(i,k)} \sqrt{A_{i,k}} \sqrt{B_{i,k}},
\]
where 
\[
A_{i,k} = \E \left[ \sum_{t=1}^T \frac{(1/|S_t| + v_{i,k}^\star) \mathbf{1}\{i \in S_t, \sigma_t(i) = k\}}{1 + V(S_t, \sigma_t)} \right]
\]
and 
\[
B_{i,k} = \E \left[ \sum_{t=1}^T \frac{(v_{i,k}^{t,\mathrm{ucb}} - v_{i,k}^\star)^2 \mathbf{1}\{i \in S_t, \sigma_t(i) = k\} \mathbf{1}\{\mathcal{E}, \mathcal{E}_{(t)}\}}{(1/|S_t| + v_{i,k}^\star)(1 + V(S_t, \sigma_t))} \right].
\]
For the term $\sum_{(i,k)}\sqrt{A_{i,k}}$, applying the Cauchy-Schwarz inequality yields 
\[
\sum_{(i,k)}\sqrt{A_{i,k}} \le \sqrt{KN \sum_{(i,k)} A_{i,k}}.
\]
Notice that
\[
\sum_{(i,k)} A_{i,k} = \E \left[ \sum_{t=1}^T \sum_{(i,k)} \frac{(1/|S_t| + v_{i,k}^\star) \mathbf{1}\{i \in S_t, \sigma_t (i)= k\}}{1 + V(S_t, \sigma_t)} \right] = \E \left[\sum_{t=1}^T 1\right] = T,
\]
we have 
\[
\sum_{(i,k)}\sqrt{A_{i,k}} \le \sqrt{KNT}.
\]
\noindent
\textbf{Step III: bounding the estimation error term ($B_{i,k}$).}
Now we bound the term $B_{i,k}$. Notice that we have $n_{i,k}^t \ge 80 x (1+v_{i,k}^\star)(1+K)$ under $\mathcal{E} \cap \mathcal{E}_{(t)}$, then
\begin{align*}
\mathbf{1}\{\mathcal{E},\mathcal{E}_{(t)}\} (v_{i,k}^{t,\mathrm{ucb}} - v_{i,k}^\star)^2 &\le \left( 8 (1+ v_{i,k}^\star) \sqrt{\frac{v_{i,k}^\star x}{n_{i,k}^t}} + \frac{28 x (1+ v_{i,k}^\star)^2}{n_{i,k}^t} \right)^2 \\
&\le 128 (1+v_{i,k}^\star)^2 v_{i,k}^\star x \cdot \frac{1}{n_{i,k}^t} + 1568 x^2 (1+v_{i,k}^\star)^4 \cdot \frac{1}{(n_{i,k}^t)^2} \\
&\le 128 (1+v_{i,k}^\star)^2 v_{i,k}^\star x \cdot \frac{1}{n_{i,k}^t} + 20 x (1+v_{i,k}^\star)^3/(1+K) \cdot \frac{1}{n_{i,k}^t}.
\end{align*}
Substituting this into the definition of $B_{i,k}$ yields
\[
B_{i,k} \le 
\mathbb{E}\left[ \sum_{t=1}^T \left( \frac{128 (1+v_{i,k}^\star)^2 v_{i,k}^\star x}{(1/|S_t| + v_{i,k}^\star) n_{i,k}^t} + \frac{20 x (1+v_{i,k}^\star)^3}{(1/|S_t| + v_{i,k}^\star) (1+K) n_{i,k}^t} \right) \cdot \frac{\mathbf{1}\{i \in S_t, \sigma_t (i)=k\}}{1+ V(S_t, \sigma_t)} \mathbf{1}\{\mathcal{E}, \mathcal{E}_{(t)}\} \right].
\]
Using $v_{i,k}^\star < 1/|S_t| + v_{i,k}^\star$ and $(1/|S_t| + v_{i,k}^\star)(1+K) >1$, the right hand side of the above inequality can be bounded by
\[
\mathbb{E}\left[ \sum_{t=1}^T \left( \frac{128 (1+v_{i,k}^\star)^2 x}{n_{i,k}^t} + \frac{20 x (1+v_{i,k}^\star)^3}{n_{i,k}^t} \right) \cdot \frac{\mathbf{1}\{i \in S_t, \sigma_t (i)=k\}}{1+ V(S_t, \sigma_t)} \mathbf{1}\{\mathcal{E}, \mathcal{E}_{(t)}\} \right].
\]
Notice that we have $n^t_i >0$ under $\mathcal{E}\cap \mathcal{E}_{(t)}$, and use $1+v_{i,k}^\star\le 2$, the above can be bounded by
\[
336 x \cdot \E \left[\sum_{t=1}^T \frac{1}{n_{i,k}^t} \cdot \mathbf{1}\{i \in S_t, \sigma_t (i)=k\} \cdot \frac{1+v_{i,k}^\star}{1+ V(S_t, \sigma_t)} \mathbf{1}\{n_{i,k}^t >0\} \right].
\]
Recall that $\mathcal{F}_{t-1} := \sigma(S_1, \sigma_1, i_1, \dots, S_{t-1}, \sigma_{t-1}, i_{t-1}, S_t, \sigma_t)$ then $\mathbf{1}\{i \in S_t, \sigma_t (i) = k' \}, \mathbf{1}\{n_{i,k}^t >0\} \in \mathcal{F}_{t-1}$ and
\[
\mathbb{E}[\mathbf{1}\{i_t \in \{i, 0\}\} \mid \mathcal{F}_{t-1}] = \sum_{k'} \frac{1+v_{i,k'}^\star}{1+ V(S_t, \sigma_t)} \mathbf{1}\{i \in S_t, \sigma_t (i)= k'\} + \frac{1}{1+ V(S_t, \sigma_t)} \mathbf{1}\{i \notin S_t\}.
\]
Therefore, for each round $t$,
\begin{align*}
& \mathbb{E}\left[\frac{1}{n_{i,k}^t} \cdot \mathbf{1}\{i \in S_t, \sigma_t (i)=k, i_t \in \{i,0\}\} \mathbf{1}\{n_{i,k}^t >0\}\right] \\
&= \mathbb{E}\left[\mathbb{E}\!\left[\frac{1}{n_{i,k}^t} \cdot \mathbf{1}\{i \in S_t, \sigma_t (i)=k, i_t \in \{i,0\}\} \mathbf{1}\{n_{i,k}^t >0\} \,\Big|\, \mathcal{F}_{t-1}\right]\right] \\
&= \mathbb{E}\left[\frac{1}{n_{i,k}^t} \cdot \mathbf{1}\{i \in S_t, \sigma_t (i)= k\} \mathbf{1}\{n_{i,k}^t >0\} \cdot \mathbb{E}[\mathbf{1}\{i_t \in \{i,0\}\} \mid \mathcal{F}_{t-1}]\right] \\
&= \mathbb{E}\left[\frac{1}{n_{i,k}^t} \cdot \mathbf{1}\{i \in S_t, \sigma_t (i)=k\} \mathbf{1}\{n_{i,k}^t >0\} \cdot \frac{1+ v_{i,k}^\star}{1+ V(S_t, \sigma_t)}\right].
\end{align*}
Hence
\begin{align*}
& 336 x \cdot \E \left[\sum_{t=1}^T \frac{1}{n_{i,k}^t} \cdot \mathbf{1}\{i \in S_t, \sigma_t (i)=k\} \cdot \frac{1+v_{i,k}^\star}{1+ V(S_t, \sigma_t)} \mathbf{1}\{n_{i,k}^t >0\} \right]\\
& = 336 x \cdot \mathbb{E}\left[\sum_{t=1}^T \frac{1}{n_{i,k}^t} \cdot \mathbf{1}\{ i \in S_t, \sigma_t (i)=k, i_t\in \{i,0\}\} \mathbf{1}\{n_{i,k}^t >0\}\right].
\end{align*}
Notice that $\sum_{t=1}^T \frac{1}{n_{i,k}^t} \cdot \textbf{1}\{ i \in S_t, \sigma_t (i)=k, i_t\in \{i,0\}\} \textbf{1}\{n_{i,k}^t >0\}$ can be upper bounded by $1+\log T$, we obtain that 
\[
B_{i,k} \le 336 x \cdot (1 + \log T).
\]
\noindent
\textbf{Step IV: final regret bound synthesis.}
Combining the upper bounds for $\sum_{(i,k)} \sqrt{A_{i,k}}$ and $\sqrt{B_{i,k}}$, we have:
\[
\sum_{(i,k)} \sqrt{A_{i,k}} \sqrt{B_{i,k}} \le 4 \sqrt{21 x \cdot (1 + \log T)} \cdot \sqrt{K N T}.
\]
Combining all parts, we obtain that the total regret is at most 
\begin{align*}
\text{Regret}(T, \pi) &\le 1 + KN\cdot \max  \{2 \log(3 K N T^2) (1+V_{\mathrm{max}})^2,  160 x (1+K)^2\} \\
& \quad + 4 \sqrt{21 (1 + \log T) \cdot \log(3 K N T (\lceil \log_2 T \rceil + 1) )} \cdot \sqrt{K N T} \\
&= \tilde{O}(\sqrt{K N T}).
\end{align*}

\subsection{Proof of Lemma~\ref{lem:offered_times_win_times_general_model}}

The proof follows an identical structure to the proof of Lemma~\ref{lem:offered_times_win_times_for_theta_known}, but applies to the general position effects model where observations are tracked per product-position pair $(i,k) \in [N] \times [K]$. 

Recall that $\tau_{i,k}^t = \sum_{s=1}^{t-1} \mathbf{1}\{i \in S_s, \sigma_s(i) = k\}$ is the total number of times product $i$ was offered at position $k$ up to round $t$, and $n_{i,k}^t = \sum_{s=1}^{t-1} \mathbf{1}\{i_s \in \{i, 0\}, i \in S_s, \sigma_s(i) = k\}$ is the number of times $i$ or the outside option $0$ was chosen during those offers. 

When the pair $(i,k)$ is offered at round $s$, the conditional probability that either $0$ or $i$ wins is:
\[
\mathbb{P}[i_s \in \{i, 0\} \mid \mathcal{F}_{s-1}] = \frac{1 + v_{i,k}^\star}{1 + V(S_s, \sigma_s)} \ge \frac{1 + v_{i,k}^\star}{1 + V_{\max}}.
\]

As in Lemma~\ref{lem:offered_times_win_times_for_theta_known}, we construct a martingale difference sequence by taking stopping times $s_m$ corresponding to the $m$-th time the pair $(i,k)$ is offered. The sequence of centered increments again falls into predictable conditional intervals of length exactly $1$. 

Applying the Azuma--Hoeffding inequality (Lemma~\ref{lem:azuma}) and taking a union bound over all $N$ products, $K$ positions, and $T$ possible values of $m$, we obtain with probability at least $1 - \delta$ that simultaneously for all $(i,k) \in [N] \times[K]$ and $t \in [T]$:
\[
n_{i,k}^t \ge \sum_{\ell=1}^{\tau_{i,k}^t} P_\ell - \sqrt{\frac{\tau_{i,k}^t}{2} \log\left(\frac{KNT}{\delta}\right)} \ge \frac{1 + v_{i,k}^\star}{1 + V_{\max}}\tau_{i,k}^t - \sqrt{\frac{\tau_{i,k}^t}{2} \log\left(\frac{KNT}{\delta}\right)}.
\]

We then find the condition under which the right-hand side is bounded below by half of its principal term, $\frac{1 + v_{i,k}^\star}{2(1 + V_{\max})} \tau_{i,k}^t$. Following the exact same algebraic rearrangement as in Lemma~\ref{lem:offered_times_win_times_for_theta_known}, this holds whenever:
\[
\tau_{i,k}^t \ge 2 \log\left(\frac{KNT}{\delta}\right) \frac{(1 + V_{\max})^2}{(1 + v_{i,k}^\star)^2}.
\]
Because $(1 + v_{i,k}^\star)^2 \ge 1$, the hypothesis $\tau_{i,k}^t \ge 2 \log\left(\frac{KNT}{\delta}\right) (1 + V_{\max})^2$ strictly implies the condition above. Thus, under this hypothesis, we have:
\[
n_{i,k}^t \ge \frac{1 + v_{i,k}^\star}{2(1 + V_{\max})} \tau_{i,k}^t,
\]
which completes the proof.

\subsection{Proof of Theorem \ref{thm:generalmodellb}}
To establish the minimax lower bound, we extend the methods in \cite{chen2018note}. We begin by constructing a family of hard instances. Consider the following sets of assortment-positioning configurations:
\begin{align*}
\mathcal{S}_K & := \{(S, \sigma) \mid S \subseteq [N], |S|=K, \sigma: S \to [K] \text{ is bijective} \}, \\
\mathcal{S}_{K-1} & := \{(S, \sigma) \mid S \subseteq [N], |S|=K-1, \sigma: S \to [K] \text{ is injective} \},
\end{align*}
The cardinalities of these sets are $|\mathcal{S}_K| = \binom{N}{K} \cdot (K!)$ and $|\mathcal{S}_{K-1}| = \binom{N}{K-1} \cdot K \cdot (K-1)! = \binom{N}{K-1} \cdot (K!)$.
Throughout the proof, we assume unit revenues for all products, i.e., $r_1=\cdots=r_N=1$. For each target configuration $(S,\sigma)\in \mathcal{S}_K$, we define a specific positioning effects parameterization $V_{(S,\sigma)}$, where the utilities are given by:
\[
v_{i,k} = \begin{cases} 
\frac{1 + \epsilon}{K}, \quad & \text{if}\ i \in S\ \text{and}\ k = \sigma(i), \\ 
\frac{1}{K}, & \text{otherwise},
\end{cases}
\]
where $\epsilon \in (0, 1/2]$ is a perturbation parameter to be optimized later.

To streamline our main argument, we first introduce two key lemmas (their proofs are deferred to subsequent subsections). The first lemma quantifies the instantaneous regret incurred by playing any suboptimal assortment relative to the optimal configuration.
\begin{lemma} \label{lem:lb_regret_gap}
Fix an arbitrary $(S_0, \sigma_0) \in \mathcal{S}_K$ and let $V_{(S_0, \sigma_0)}$ be the associated parameter, then for any $(\tilde{S}_t, \tilde{\sigma}_t)\in \mathcal{S}_K$, it holds that 
\[
\max_{S, \sigma} R(S, \sigma, V_{(S_0, \sigma_0)}) - R(\tilde{S}_t, \tilde{\sigma}_t, V_{(S_0, \sigma_0)}) \ge \frac{4}{25} \cdot \delta \epsilon,
\]
where $\delta = 1 - |(\tilde{S}_t, \tilde{\sigma}_t) \cap (S_0, \sigma_0)| / K$ represents the fraction of mismatched product-position pairs.
\end{lemma}

Our second lemma controls the information-theoretic distance (KL-divergence) between the observation distributions of a baseline instance and a perturbed instance. Let $\mathcal{S}_{K-1}^{(i, k)} = \mathcal{S}_{K-1} \cap \{(S, \sigma) \mid i \notin S, k \notin \sigma(S)\}$ denote the configurations of size $K-1$ that exclude product $i$ and position $k$.
\begin{lemma} \label{lem:lb_kl_bound}
Suppose $\epsilon \in (0,1/2]$. For any $(S', \sigma') \in \mathcal{S}_{K-1}^{(i,k)}$, let $(S, \sigma) = (S', \sigma') \cup \{(i,k)\}$. Then, the KL-divergence between the resulting distributions is bounded by:
\[
D_{\text{KL}}(P_{(S', \sigma')} \parallel P_{(S, \sigma)}) \le \mathbb{E}_{(S', \sigma')} [N_{(i,k)}] \cdot \frac{27 \epsilon^2}{K},
\]
where $N_{(i,k)}$ is the total number of times the pair $(i,k)$ is offered.
\end{lemma}

With the hard instances defined, we proceed to bound the expected cumulative regret by dividing the remainder of the proof into four logical steps. 
Because any subset of size $|S_t|< K$ can be trivially extended to a full configuration in $\mathcal{S}_K$ to yield a strictly higher expected revenue, we can safely assume the policy always plays some $(\tilde{S}_t,\tilde{\sigma}_t)\in \mathcal{S}_K$.
\noindent \newline
\textbf{Step I: lower bounding regret via expected pulls.}
We first average the expected regret uniformly over all target configurations $(S,\sigma)\in \mathcal{S}_K$. Applying Lemma~\ref{lem:lb_regret_gap}, we obtain:
\begin{align*}
& \max_{(S, \sigma) \in \mathcal{S}_K} \mathbb{E}_{(S, \sigma)} \left[ \sum_{t=1}^T (R(S, \sigma, V_{(S, \sigma)}) - R(S_t, \sigma_t, V_{(S, \sigma)})) \right] \\
& \ge \max_{(S, \sigma) \in \mathcal{S}_K} \mathbb{E}_{(S, \sigma)} \left[ \sum_{t=1}^T (R(S, \sigma, V_{(S, \sigma)}) - R(\tilde{S}_t, \tilde{\sigma}_t, V_{(S, \sigma)})) \right] \\
& \ge \frac{1}{|\mathcal{S}_K|} \sum_{(S, \sigma) \in \mathcal{S}_K} \mathbb{E}_{(S, \sigma)} \left[ \sum_{t=1}^T (R(S, \sigma, V_{(S, \sigma)}) - R(\tilde{S}_t, \tilde{\sigma}_t, V_{(S, \sigma)})) \right] \\
& \ge \frac{1}{|\mathcal{S}_K|} \sum_{(S, \sigma) \in \mathcal{S}_K} \mathbb{E}_{(S, \sigma)} \left[ \sum_{(i, k) \notin (S, \sigma)} N_{(i, k)} \cdot \frac{4\epsilon}{25 K} \right] \\
&= \frac{4}{25}\cdot \epsilon \cdot \left( T - \frac{1}{|\mathcal{S}_K|} \sum_{(S, \sigma) \in \mathcal{S}_K} \frac{1}{K} \sum_{(i, k) \in (S, \sigma)} \mathbb{E}_{(S, \sigma)} [N_{(i, k)}] \right),
\end{align*}
where we slightly abuse notation by treating $(S,\sigma)$ as the set of pairs $\{(i,k) \mid i \in S, k = \sigma(i)\}$, and $N_{(i, k)} := \sum_{t=1}^T \mathbf{1}\{i \in \tilde{S}_t, k = \tilde{\sigma}_t(i)\}$ tracks the total number of times pair $(i,k)$ is offered. The last equality leverages the fact that $\sum_{(i,k)} N_{(i,k)} = KT$.
\noindent \newline
\textbf{Step II: change of measure via Pinsker's Inequality.}
The lower bound proof is then reduced to finding the largest $\epsilon$ such that the summation term in the last equation above is upper bounded by, say, $cT$ for some constant $c < 1$. To facilitate this, we rearrange the summation over $\mathcal{S}_K$ into a summation over the smaller configurations $\mathcal{S}_{K-1}^{(i,k)}$:
\begin{align*}
& \frac{1}{|\mathcal{S}_K|} \sum_{(S, \sigma) \in \mathcal{S}_K} \frac{1}{K} \sum_{(i, k) \in (S, \sigma)} \mathbb{E}_{(S, \sigma)} [N_{(i, k)}] \\
&= \frac{1}{K} \sum_{(i, k)} \frac{1}{|\mathcal{S}_K|} \sum_{(S, \sigma) \in \mathcal{S}_K : (i, k) \in (S, \sigma)} \mathbb{E}_{(S, \sigma)} [N_{(i, k)}] \\
&= \frac{1}{K} \sum_{(i, k)} \frac{1}{|\mathcal{S}_K|} \sum_{(S', \sigma') \in \mathcal{S}_{K-1}^{(i, k)}} \mathbb{E}_{(S', \sigma') \cup \{(i, k)\}} [N_{(i, k)}].
\end{align*}

Let $P = P_{(S', \sigma')}$ and $Q = P_{(S', \sigma') \cup \{(i, k)\}}$ denote the probability measures under the base and augmented configurations, respectively. Because $0 \le N_{(i, k)} \le T$ almost surely, we can relate the expected pulls under $Q$ to those under $P$ via Pinsker's inequality:
\begin{align*}
|\mathbb{E}_P [N_{(i, k)}] - \mathbb{E}_Q [N_{(i, k)}]| &\le \sum_{j=0}^T j \cdot |P[N_{(i, k)} = j] - Q[N_{(i, k)} = j]| \\
&\le T \cdot \sum_{j=0}^T |P[N_{(i, k)} = j] - Q[N_{(i, k)} = j]| \\
&\le T \cdot \delta(P, Q) \le T \cdot \sqrt{\frac{1}{2} D_{\text{KL}}(P \parallel Q)},
\end{align*}
where $\delta(P, Q)$ and $D_{\text{KL}}(P \parallel Q)$ are the total variation distance and the KL-divergence between $P$ and $Q$ respectively. Applying this shift of measure yields:
\begin{align}
& \frac{1}{|\mathcal{S}_K|} \sum_{(S, \sigma) \in \mathcal{S}_K} \frac{1}{K} \sum_{(i, k) \in (S, \sigma)} \mathbb{E}_{(S, \sigma)} [N_{(i, k)}] \nonumber \\
&\le \frac{1}{K} \sum_{(i, k)} \frac{1}{|\mathcal{S}_K|} \sum_{(S', \sigma') \in \mathcal{S}_{K-1}^{(i, k)}} \left( \mathbb{E}_{(S', \sigma')} [N_{(i, k)}] + T \cdot \sqrt{\frac{1}{2} D_{\text{KL}}(P_{(S', \sigma')} \parallel P_{(S', \sigma') \cup \{(i, k)\}})} \right). \label{eq:averageofferedtimesub}
\end{align}
\noindent
\textbf{Step III : bounding the base pulls and information divergence.}
We bound the two terms on the right-hand side of Eq.~\eqref{eq:averageofferedtimesub} separately. 
\newline \textit{Bounding the base expected pulls:} 
Assuming a non-trivial catalog size $K \le N/4$, we have:
\begin{align*}
& \frac{1}{K} \sum_{(i, k)} \frac{1}{|\mathcal{S}_K|} \sum_{(S', \sigma') \in \mathcal{S}_{K-1}^{(i, k)}} \mathbb{E}_{(S', \sigma')} [N_{(i, k)}] \\
&= \frac{1}{|\mathcal{S}_K|} \sum_{(S', \sigma') \in \mathcal{S}_{K-1}} \frac{1}{K} \sum_{(i, k) : i \notin S', k \notin \sigma'(S')} \mathbb{E}_{(S', \sigma')} [N_{(i, k)}] \\
&\le \frac{1}{|\mathcal{S}_K|} \sum_{(S', \sigma') \in \mathcal{S}_{K-1}} \frac{1}{K} \sum_{(i, k)} \mathbb{E}_{(S', \sigma')} [N_{(i, k)}] \\
&\le \frac{|\mathcal{S}_{K-1}|}{|\mathcal{S}_K|} \frac{1}{K} \cdot KT \\
&\le \frac{KT}{N-K+1} \le \frac{T}{3}.
\end{align*}
\newline \textit{Bounding the information divergence term:} By reordering the summations, the second term can be written as:
\begin{align*}
& \frac{1}{K} \sum_{(i, k)} \frac{1}{|\mathcal{S}_K|} \sum_{(S', \sigma') \in \mathcal{S}_{K-1}^{(i, k)}} T \cdot \sqrt{\frac{1}{2} D_{\text{KL}}(P_{(S', \sigma')} \parallel P_{(S', \sigma') \cup \{(i, k)\}})} \nonumber \\
&= \frac{T}{|\mathcal{S}_K|} \sum_{(S', \sigma') \in \mathcal{S}_{K-1}} \frac{1}{K} \sum_{(i, k) : i \notin S', k \notin \sigma'(S')} \sqrt{\frac{1}{2} D_{\text{KL}}(P_{(S', \sigma')} \parallel P_{(S', \sigma') \cup \{(i, k)\}})}.
\end{align*}
We upper bound this expression by replacing the first sum with maximum over all base configurations:
\begin{align*}
& \frac{T |\mathcal{S}_{K-1}|}{K |\mathcal{S}_K|} \cdot \max_{(S', \sigma') \in \mathcal{S}_{K-1}} \sum_{(i,k) : i \notin S', k \notin \sigma'(S')} \sqrt{\frac{1}{2} D_{\text{KL}} (P_{(S', \sigma')} \parallel P_{(S', \sigma') \cup \{(i,k)\}})} \\
&= \frac{T}{N-K+1} \cdot \max_{(S', \sigma') \in \mathcal{S}_{K-1}} \sum_{(i,k) : i \notin S', k \notin \sigma'(S')} \sqrt{\frac{1}{2} D_{\text{KL}} (P_{(S', \sigma')} \parallel P_{(S', \sigma') \cup \{(i,k)\}})}.
\end{align*}
For each $(S', \sigma') \in \mathcal{S}_{K-1}$, since $|\sigma'(S')| = K-1$, there are exactly $N-K+1$ valid pairs $(i,k)$ satisfying $i \notin S', k \notin \sigma'(S')$. Applying Jensen's inequality allows us to push the summation inside the square root, yielding:
\begin{align*}
& \frac{1}{N-K+1} \sum_{(i,k) : i \notin S', k \notin \sigma'(S')} \sqrt{\frac{1}{2} D_{\text{KL}} (P_{(S', \sigma')} \parallel P_{(S', \sigma') \cup \{(i,k)\}})} \\
&\le \sqrt{\frac{1}{2(N-K+1)} \sum_{(i,k) : i \notin S', k \notin \sigma'(S')} D_{\text{KL}} (P_{(S', \sigma')} \parallel P_{(S', \sigma') \cup \{(i,k)\}})} \\
&\le \sqrt{\frac{1}{2(N-K+1)} \sum_{(i,k) : i \notin S', k \notin \sigma'(S')} \mathbb{E}_{(S', \sigma')} [N_{(i,k)}] \cdot \frac{27 \epsilon^2}{K}} \\
&\le \sqrt{\frac{27 \epsilon^2}{2 K (N-K+1)} \cdot T} \le \sqrt{27 \epsilon^2 \cdot \frac{T}{KN}},
\end{align*}
where the second inequality applies Lemma~\ref{lem:lb_kl_bound}, and the third inequality follows by $\sum_{(i,k) : i \notin S', k \notin \sigma'(S')} \mathbb{E}_{(S', \sigma')} [N_{(i,k)}] \le T$, which holds since there remains only one position $k \notin \sigma'(S')$. Hence, the divergence term is bounded by:
\begin{align*}
& \frac{T}{|\mathcal{S}_K|} \sum_{(S', \sigma') \in \mathcal{S}_{K-1}} \frac{1}{K} \sum_{(i,k) : i \notin S', k \notin \sigma'(S')} \sqrt{\frac{1}{2} D_{\text{KL}} (P_{(S', \sigma')} \parallel P_{(S', \sigma') \cup \{(i,k)\}})} \\
&\le T \cdot \sqrt{27 \epsilon^2 \cdot \frac{T}{KN}}.
\end{align*}
\noindent
\textbf{Step IV: final regret bound synthesis.}
We now substitute the bounds derived in Step III back into the regret lower bound established in Step I. This gives:
\begin{align*}
& \max_{(S, \sigma) \in \mathcal{S}_K} \mathbb{E}_{(S, \sigma)} \left[ \sum_{t=1}^T (R(S, \sigma, V_{(S, \sigma)}) - R(S_t, \sigma_t, V_{(S, \sigma)})) \right] \\
&\ge \frac{4}{25} \cdot \epsilon \cdot \left( T - \frac{1}{|\mathcal{S}_K|} \sum_{(S, \sigma) \in \mathcal{S}_K} \frac{1}{K} \sum_{(i,k) \in (S, \sigma)} \mathbb{E}_{(S, \sigma)} [N_{(i,k)}] \right) \\
&\ge \frac{4}{25} \cdot \epsilon \cdot \left( T - \frac{T}{3} - T \cdot \sqrt{27 \epsilon^2 \cdot \frac{T}{KN}} \right).
\end{align*}
Finally, we set the perturbation parameter $\epsilon = \sqrt{\frac{KN}{9 \cdot 27 \cdot T}}$ (this needs $T\ge \frac{4}{243} KN$ to ensure $\epsilon \in (0,1/2]$), then $\sqrt{27 \epsilon^2 \cdot \frac{T}{KN}} = 1/3$. Substituting $\epsilon$ into the lower bound yields:
\[
\frac{4}{25} \cdot \epsilon \cdot \frac{T}{3} = \frac{4}{675 \sqrt{3}} \cdot \sqrt{KNT}.
\]
This completes the proof of Theorem~\ref{thm:generalmodellb}.

\subsection{Proof of Lemma \ref{lem:lb_regret_gap}}

The expected revenue is maximized by playing the target configuration $(S_0, \sigma_0)$ itself. In this case, all $K$ displayed products are in their optimal positions. The sum of their utilities is exactly $K \times \frac{1+\epsilon}{K} = 1 + \epsilon$. Therefore, the optimal expected revenue is:
\[
\max_{S, \sigma} R(S, \sigma, V_{(S_0, \sigma_0)}) = R(S_0, \sigma_0, V_{(S_0, \sigma_0)}) = \sum_{i \in S_0} \frac{v_{i, \sigma_0(i)}}{1 + \sum_{j \in S_0} v_{j, \sigma_0(j)}} = \frac{1 + \epsilon}{2 + \epsilon}.
\]
Now, consider any played configuration $(\tilde{S}_t, \tilde{\sigma}_t) \in \mathcal{S}_K$. We formally define the set of correctly matched product-position pairs as the intersection of the played and optimal configurations:
\[
(\tilde{S}_t, \tilde{\sigma}_t) \cap (S_0, \sigma_0) := \{(i, k) \mid i \in \tilde{S}_t \cap S_0, \tilde{\sigma}_t(i) = \sigma_0(i) = k \}.
\]
Let $\delta = 1 - |(\tilde{S}_t, \tilde{\sigma}_t) \cap (S_0, \sigma_0)| / K$ denote the fraction of mismatched pairs. This implies there are exactly $K(1 - \delta)$ correctly matched pairs and $K\delta$ mismatched pairs. The total attraction parameter of $(\tilde{S}_t,\tilde{\sigma}_t)$ is thus:
\[
\sum_{i \in \tilde{S}_t} v_{i, \tilde{\sigma}_t(i)} = K(1 - \delta) \cdot \frac{1+\epsilon}{K} + K\delta \cdot \frac{1}{K} = (1 - \delta)(1 + \epsilon) + \delta = 1 + (1 - \delta)\epsilon.
\]
Consequently, the expected revenue of the played configuration is:
\[
R(\tilde{S}_t, \tilde{\sigma}_t, V_{(S_0, \sigma_0)}) = \frac{1 + (1 - \delta) \epsilon}{1 + 1 + (1 - \delta) \epsilon} = \frac{1 + (1 - \delta) \epsilon}{2 + (1 - \delta) \epsilon}.
\]
By substituting the expressions derived above, we obtain:
\begin{align*}
\max_{S, \sigma} R(S, \sigma, V_{(S_0, \sigma_0)}) - R(\tilde{S}_t, \tilde{\sigma}_t, V_{(S_0, \sigma_0)}) 
&= \frac{1 + \epsilon}{2 + \epsilon} - \frac{1 + (1 - \delta) \epsilon}{2 + (1 - \delta) \epsilon} \\
&= \frac{(1 + \epsilon)(2 + (1 - \delta) \epsilon) - (1 + (1 - \delta) \epsilon)(2 + \epsilon)}{(2 + \epsilon)(2 + (1 - \delta) \epsilon)} \\
&= \frac{\delta \epsilon}{(2 + \epsilon)(2 + (1 - \delta) \epsilon)}.
\end{align*}
Because the mismatch fraction $\delta \in[0,1]$ and the perturbation parameter $\epsilon \in (0, 1/2]$, we can upper bound the terms in the denominator as $2 + \epsilon \le 5/2$ and $2 + (1 - \delta) \epsilon \le 5/2$. The denominator is therefore at most $25/4$. This yields the desired lower bound:
\[
\max_{S, \sigma} R(S, \sigma, V_{(S_0, \sigma_0)}) - R(\tilde{S}_t, \tilde{\sigma}_t, V_{(S_0, \sigma_0)}) \ge \frac{4}{25} \cdot \delta \epsilon,
\]
which completes the proof.

\subsection{Proof of Lemma~\ref{lem:lb_kl_bound}}

By the chain rule for KL divergence, the total divergence is the expected sum of conditional KL divergences in each round. In any round $t$, if the offered assortment and positioning $(S_t, \sigma_t)$ does not contain the perturbed pair $(i,k)$, the choice probabilities under both configurations are identical. Consequently, the conditional KL divergence is zero:
\[
D_{\text{KL}}(P_{(S', \sigma')} (\cdot \mid (S_t, \sigma_t)) \parallel P_{(S, \sigma)} (\cdot \mid (S_t, \sigma_t))) = 0, \quad \text{for } (i,k) \notin (S_t, \sigma_t).
\]
Therefore, we only need to consider rounds where the pair $(i,k)$ is offered, i.e., $(i,k) \in (S_t, \sigma_t)$. This event occurs exactly $\mathbb{E}_{(S', \sigma')} [N_{(i,k)}]$ times in expectation under the base measure. It suffices to upper bound the maximum conditional KL divergence for any such assortment.

Consider a specific assortment $(S_t, \sigma_t)$ such that $(i,k) \in (S_t, \sigma_t)$. Let $K' := |S_t| \le K$ denote the number of displayed products, and let $J := |(S_t, \sigma_t) \cap (S', \sigma')| \le K-1$ be the number of displayed product-position pairs that match the base optimal configuration.
Let $p_j$ and $q_j$ denote the probability that the customer selects option $j \in S_t \cup \{0\}$ under the base measure $P_{(S', \sigma')}$ and the augmented measure $P_{(S, \sigma)}$, respectively. 
Then we have
\begin{align*}
p_j &= P_{(S', \sigma')} [c_t=j \mid (S_t, \sigma_t)] = \frac{v'_{j, \sigma_t(j)}}{1 + \sum_{i \in S_t} v'_{i, \sigma_t(i)}} = \frac{v'_{j, \sigma_t(j)}}{a + J \epsilon / K}, \\
q_j &= P_{(S, \sigma)} [c_t=j \mid (S_t, \sigma_t)] = \frac{v_{j, \sigma_t(j)}}{1 + \sum_{i \in S_t} v_{i, \sigma_t(i)}} = \frac{v_{j, \sigma_t(j)}}{a + (J+1) \epsilon / K},
\end{align*}
where $a = 1 + K' / K \in (1,2]$.

We analyze the absolute difference $|p_j - q_j|$ for three distinct cases: $j=0, j \in S_t \setminus \{i\}$ and $j=i$.
For $j=0$, we have:
\[
|p_0 - q_0| = \left| \frac{1}{a + J \epsilon / K} - \frac{1}{a + (J+1) \epsilon / K} \right| \le \frac{\epsilon}{K}.
\]
For $j \in S_t \setminus \{i\}$, its intrinsic attraction parameter is at most $(1 + \epsilon)/K$. Thus:
\[
|p_j - q_j| \le \frac{1 + \epsilon}{K} \left| \frac{1}{a + J \epsilon / K} - \frac{1}{a + (J+1) \epsilon / K} \right| \le \frac{2 \epsilon}{K^2}.
\]
For $j=i$, the intrinsic attraction parameter changes, we have that
\begin{align*}
|p_i - q_i| &= \left| \frac{1/K}{a + J \epsilon / K} - \frac{(1 + \epsilon) / K}{a + (J+1) \epsilon / K} \right| \\
&\le \frac{\epsilon}{K} \cdot \frac{1}{a + (J+1) \epsilon / K} + \left| \frac{1/K}{a + J \epsilon / K} - \frac{1/K}{a + (J+1) \epsilon / K} \right| \\
&\le \frac{\epsilon}{K} + \frac{\epsilon}{K^2} \le \frac{2 \epsilon}{K}.
\end{align*}

We upper bound the conditional KL divergence by the chi-squared divergence (which follows from $\log x \le x - 1$): $D_{\text{KL}}(P \parallel Q) \le \chi^2(P \parallel Q) = \sum_j \frac{(p_j - q_j)^2}{q_j}$.

First, we establish lower bounds on the probabilities $q_j$. Note that $a \le 2$, $J \le K-1$, and $\epsilon \le 1/2$, we have $a + (J+1) \epsilon / K \le 2 + K \cdot (1/K) = 3$. Therefore, 
\[
q_0 = \frac{1}{a + (J+1) \epsilon / K} \ge \frac{1}{3}.
\]
For any offered product $j \in S_t$, its intrinsic attraction parameter is at least $1/K$, giving:
\[
q_j = \frac{v_{j, \sigma_t(j)}}{a + (J+1) \epsilon / K} \ge \frac{1/K}{3} = \frac{1}{3K}.
\]
Now, we sum the contributions across all possible choices.
\begin{align*}
D_{\text{KL}}(P_{(S', \sigma')} (\cdot \mid (S_t, \sigma_t)) \parallel P_{(S, \sigma)} (\cdot \mid (S_t, \sigma_t))) 
&\le \sum_{j \in S_t \cup \{0\}} \frac{(p_j - q_j)^2}{q_j} \\
&\le \frac{1}{q_0} |p_0 - q_0|^2 + \sum_{j \in S_t \setminus \{i\}} \frac{1}{q_j} |p_j - q_j|^2 + \frac{1}{q_i} |p_i - q_i|^2 \\
&\le 3 \cdot \left(\frac{\epsilon}{K}\right)^2 + (K-1) \cdot 3K \cdot \left(\frac{2 \epsilon}{K^2}\right)^2 + 3K \cdot \left(\frac{2 \epsilon}{K}\right)^2 \\
&\le \frac{3\epsilon^2}{K^2} + \frac{12 \epsilon^2}{K^2} + \frac{12 \epsilon^2}{K} \\
&\le \frac{15 \epsilon^2}{K^2} + \frac{12 \epsilon^2}{K} \le \frac{27 \epsilon^2}{K}.
\end{align*}
Here, the final inequality uses the fact that $K \ge 1$, meaning $1/K^2 \le 1/K$. This completes the proof of the lemma.

\end{document}